\documentclass[11 pt]{article}
\usepackage{style}
\newcommand{\Gap}[1]{\text{NG}\left(#1\right)}

\newcommand{\pj}{\text{Proj}}

\title{\LARGE\bfseries Two-Timescale Q-Learning with Function Approximation \\ in Zero-Sum Stochastic Games }

\author{Zaiwei Chen\textsuperscript{$1,*$}, Kaiqing Zhang\textsuperscript{$2$}, Eric Mazumdar\textsuperscript{$1,\dagger$}, Asuman Ozdaglar\textsuperscript{$3$}, Adam Wierman\textsuperscript{$1,\ddagger$}\\
{\small
\textsuperscript{$1$}\textit{CMS, Caltech,} \href{mailto:zchen458@caltech.edu}{\textit{\textsuperscript{$*$}zchen458@caltech.edu}}, \href{mailto:mazumdar@caltech.edu}{\textit{\textsuperscript{$\dagger$}mazumdar@caltech.edu}}, \href{mailto:adamw@caltech.edu}{\textit{\textsuperscript{$\ddagger$}adamw@caltech.edu}}}\\
{\small
\textsuperscript{$2$}\textit{ECE, University of Maryland, College Park,} \href{mailto:kaiqing@umd.edu}{\textit{kaiqing@umd.edu}}}\\
{\small
	\textsuperscript{$3$}\textit{EECS, MIT,} \href{mailto:asuman@mit.edu}{\textit{asuman@mit.edu}}
}
}
\date{\vspace{-0.4 in}}
\begin{document}
\maketitle

\begin{abstract}
We consider two-player zero-sum stochastic games and propose a two-timescale $Q$-learning algorithm with function approximation that is payoff-based, convergent, rational, and symmetric between the two players. In two-timescale $Q$-learning, the fast-timescale iterates are updated in spirit to the stochastic gradient descent and the slow-timescale iterates (which we use to compute the policies) are updated by taking a convex combination between its previous iterate and the latest fast-timescale iterate. Introducing the slow timescale as well as its update equation marks as our main algorithmic novelty. In the special case of linear function approximation, we establish, to the best of our knowledge, the first last-iterate finite-sample bound for payoff-based independent learning dynamics of these types. The result implies a polynomial sample complexity to find a Nash equilibrium in such stochastic games. 

To establish the results, we model our proposed algorithm as a two-timescale stochastic approximation and derive the finite-sample bound through a Lyapunov-based approach. The key novelty lies in constructing a valid Lyapunov function to capture the evolution of the slow-timescale iterates. Specifically, through a change of variable, we show that the update equation of the slow-timescale iterates resembles the classical smoothed best-response dynamics, where the regularized Nash gap serves as a valid Lyapunov function. This insight enables us to construct a valid Lyapunov function via a generalized variant of the Moreau envelope of the regularized Nash gap. The construction of our Lyapunov function might be of broad independent interest in studying the behavior of stochastic approximation algorithms. 
\end{abstract}

\section{Introduction}\label{sec:intro}
Recent years have seen remarkable successes  of reinforcement learning (RL) across  many application domains, such as the game of Go  \cite{silver2016mastering}, autonomous driving \cite{shalev2016safe}, and robotics \cite{kober2013reinforcement}. Among existing learning dynamics, $Q$-learning \cite{watkins1992q} and its variants, such as the celebrated deep $Q$-network (DQN) \cite{mnih2015human}, are among the most popular. Although $Q$-learning was originally proposed for single-agent RL, since many real-world applications by nature involve the decision-making and learning of multiple agents in a shared environment, it is natural to explore the potential of $Q$-learning in multi-agent systems \cite{bu2008comprehensive,zhang2021multi}. However, in contrast to the single-agent setting (where $Q$-learning has a solid theoretical foundation \cite{tsitsiklis1994asynchronous,even2003learning,chen2021lyapunov,li2021breaking}), due to the interactions among agents in multi-agent systems, it is still an active research topic to theoretically investigate whether the agents will jointly achieve any performance guarantees if each of them independently implements $Q$-learning (or its variants) \cite{leslie2005individual,sayin2021decentralized}. 

One of the key challenges in the deployment of $Q$-learning to real-world applications is the curse of dimensionality \cite{bertsekas1996neuro,powell2007approximate}.  Specifically, the state space in practical applications can be massively or even infinitely large, making it computationally inefficient, if not intractable, to estimate and store the $Q$-functions in learning. A common way to deal with this curse has been to employ \emph{function approximation} for the $Q$-function. The high-level idea is to approximate the solution of the RL problem from a chosen function class, in which each function can be specified by a parameter (or weight) with a much smaller dimension than the size of the state-action space of the underlying  RL problem. A representative example of $Q$-learning with function approximation is the DQN, in which neural networks are used as a means for approximating the $Q$-functions. Unfortunately, when function approximation is employed, the ``deadly triad''\footnote{The deadly triad refers to the combination of off-policy sampling, function approximation, and bootstrapping \cite{sutton2015introduction}.} \cite{sutton2015introduction} would occur  in $Q$-learning, resulting in the algorithm to diverge even under linear function approximation, as demonstrated by several counterexamples in the literature \cite{baird1995residual,chen2022target}. To overcome the deadly triad, variants of $Q$-learning have been proposed, such as Greedy-GQ \cite{maei2010toward}, $Q$-learning with a target network \cite{chen2022target,zhang2021breaking}, fitted $Q$-iteration \cite{munos2008finite}, and Zap $Q$-learning \cite{chen2020zap}, etc. However, the results are mostly limited to the single-agent RL setting. Therefore, it is natural to ask:
\begin{center}
    \textit{Is it possible to make $Q$-learning with function approximation provably work for multi-agent systems?}
\end{center}

In this work, we provide an answer to the above question in the benchmark setting of two-player zero-sum stochastic games \cite{shapley1953stochastic}, where both players aim to optimize their strategies in a competitive environment with evolving states. Specifically, we propose a two-timescale $Q$-learning dynamics that can  incorporate function approximation. In particular, our learning dynamics enjoy several desired properties in multi-agent reinforcement learning (MARL), i.e., the learning dynamics are \emph{independent} (require no coordination between the players \cite{ozdaglar2021independent}), \emph{rational} (each player will converge to the best-response to the opponent if the opponent plays an asymptotically stationary policy \cite{bowling2001rational}), and \emph{payoff-based} (each player can only observe their realized payoff as well as the system state during learning \cite{foster2006regret,marden2012revisiting}). On the theoretical side, we establish the last-iterate finite-sample convergence  guarantees in the case when the function approximator has a linear structure. We detail our contributions in the next subsection.

\subsection{Main Contributions} \label{subsec:contribution}

\paragraph{The Learning Dynamics.} From the perspective of player $i\in \{1,2\}$, our two-timescale $Q$-learning has two loops and maintains $4$ sets of parameters (or weights) $w^i_{t,k}$, $\theta^i_{t,k}$, $\Bar{w}_t^i$, and $\Bar{\theta}_t^i$ that we use to approximate the local $q$-functions, where $k$ is the inner-loop iteration index and $t$ is the outer-loop iteration index. Similarly to the DQN in the single-agent setting, the two sets of parameters $\Bar{w}_t^i$ and $\Bar{\theta}_t^i$ are called the \emph{target networks}, which are fixed in the inner loop and are synchronized to the last iterates $w_{t,K}^i$ and $\theta_{t,K}^i$ from the inner loop, respectively, where $K$ is the number of iterations of the inner loop. As will be illustrated in Section \ref{subsec:viewpoint}, such an update rule for $\Bar{w}_t^i$ and $\Bar{\theta}_t^i$ can be interpreted as an approximate version of the \emph{minimax value iteration}, a well-known equilibrium-computation algorithm for zero-sum stochastic games \cite{shapley1953stochastic}. 
    
The parameters $w^i_{t,k}$ and $\theta^i_{t,k}$ are updated in the inner loop in a two-timescale fashion, and is designed to find the values of a sequence of induced zero-sum matrix games, which are needed for implementing the minimax value iteration in the outer loop. Specifically, $w_{t,k}^i$ is the fast-timescale parameter and is updated according to the projected stochastic gradient descent (SGD) for minimizing a variant of the Bellman error, which reduces to
the temporal-difference (TD)-learning when using linear function approximation \cite{sutton1988learning,tsitsiklis1997analysis}. The slow-timescale iterate $\theta_{t,k}^i$  is used to compute the policies for playing the game, which we propose to update by performing the convex combination $\theta_{t,k+1}^i=(1-\beta_k)\theta_{t,k}^i+\beta_k w_{t,k}^i$ with stepsize  $\beta_k>0$. Importantly, we show in Section \ref{sec:analysis} that the seemingly simple update equation for $\theta_{t,k}^i$ can be interpreted as a generalization of the smoothed best-response dynamics (classical learning dynamics for solving zero-sum matrix games) implemented in the \emph{parameter space} of the function approximation class. 

Notably, introducing the two-timescale structure and the update equation of the slow-timescale iterate $\theta_{t,k}^i$ is essential for connecting our algorithm to the classical smoothed best-response dynamics, which also motivates our Lyapunov-based analysis framework, and thus presents our main algorithmic novelty.

\paragraph{The Last-Iterate Finite-Sample Analysis.} When using linear function approximation, under certain assumptions, we establish the last-iterate finite-sample guarantee of our proposed learning dynamics measured by the Nash gap, which implies a polynomial sample complexity to find a Nash equilibrium of the underlying stochastic game. To the best of our knowledge, this is the first result that shows last-iterate convergence bounds for payoff-based independent learning dynamics under function approximation for stochastic games. See Section \ref{subsec:related_work} for a detailed comparison with the existing literature.

\paragraph{A New Lyapunov Function for Analysis.} To establish the results, we model our proposed algorithm as a two-timescale stochastic approximation, and the overall framework for our analysis is a Lyapunov-based approach, which has been widely used  to study stochastic approximation algorithms in optimization and machine learning \cite{bansal2019potential,bottou2018optimization,lan2020first,srikant2019finite,chen2021lyapunov,chen2023finite}. The main challenge, as in all of the Lyapunov-based approaches, is the construction of a valid Lyapunov function to capture the stability of the iterates. As illustrated at the beginning of this subsection, one of our main algorithmic novelties is to propose to update the slow-timescale iterate $\theta_{t,k}^i$ (which we use to compute the policies) according to the convex combination of the previous $\theta_{t,k}^i$ and the updated $w_{t,k}^i$, as  $\theta_{t,k+1}^i=(1-\beta_k)\theta_{t,k}^i+\beta_k w_{t,k}^i$. Although such an update equation is simple to implement, it cannot be interpreted as any standard RL algorithms such as policy iteration or policy gradient \cite{sutton2018reinforcement}, which makes studying $\{\theta_{t,k}^i\}$ particularly challenging as there is no off-the-shelf Lyapunov function. 

To overcome this challenge, firstly, we show that the update equation for the local $q$-function associated with the parameter $\theta_{t,k}^i$ (after a change of variable) mimics the update equation for the policies in the smoothed best-response dynamics. This observation motivates us to use the entropy-regularized Nash gap (ERNG) (which was known to be a valid Lyapunov function for studying the smoothed best-response dynamics for matrix games \cite{hofbauer2005learning}) as a starting point for our Lyapunov function construction. However, several issues prevent us from directly using the ERNG to study $\theta_{t,k}^i$. For example, since the parameter  of our function approximator  in general lives in a real-vector space (beyond the simplexes), the entropy function (which is the regularizer in ERNG) is not even directly well-defined in the parameter space. 
In addition, the change of variable that enabled us to connect the update equation for $\theta_{t,k}^i$ to the smoothed best-response dynamics requires the invertibility of the reward matrix at each state, which may not hold in general. To resolve these issues, \textit{we construct our Lyapunov function as the infimal convolution between the ERNG and the square of a properly defined seminorm, which can also be interpreted as a generalization of the Moreau envelope}. The construction process is detailed in Section \ref{sec:analysis}, and the roadmap behind it might be of independent interest.

\subsection{Related Work}\label{subsec:related_work}

\paragraph{Single-Agent $Q$-Learning (with Function Approximation).} $Q$-learning, first proposed in \cite{watkins1992q}, has become one of the most classical and empirically popular algorithms in RL. As a result, the theoretical understanding of $Q$-learning has attracted a great deal of attention. While early literature has focused on the \emph{asymptotic}  convergence guarantees \cite{tsitsiklis1994asynchronous,jaakkola1994convergence,borkar2000ode}, the more recent ones have shared the common interest in establishing \emph{finite-sample} guarantees \cite{even2003learning,beck2012error,beck2013improved,li2021q,li2020sample,chen2021lyapunov,jin2018q}. To overcome the curse of dimensionality, $Q$-learning is usually incorporated with function approximation in practical applications, a typical example of which is the DQN \cite{mnih2013playing}. However, since $Q$-learning is inherently a bootstrapped off-policy learning algorithm (as our goal is to estimate the state-action value function of an optimal policy while using samples collected from a potentially sub-optimal one), under function approximation, the so-called phenomenon of \emph{deadly triad} \cite{sutton2015introduction} will occur, and the algorithm can be unstable and easily diverge \cite{baird1995residual,chen2022target}. Overcoming the deadly triad in $Q$-learning (empirically and also theoretically) has been an active research topic in the RL community, see e.g.,  \cite{melo2008analysis,chen2019finitesample,gao2021finite,lee2019unified,zou2019finite,xu2020finite,cai2023neural,carvalho2020new,zhang2021breaking,agarwal2022online,zanette2022stabilizing,maei2010toward,szepesvari2005finite, yang2020reinforcement,jin2020provably,weisz2021exponential,zhou2021nearly,munos2008finite,du2019provably,chen2022target,meyn2023stability}. Among existing approaches, the introduction of a target network in $Q$-learning \cite{mnih2013playing,zhang2021breaking,chen2022target} has proven to be an effective idea to improve the stability of $Q$-learning, which is also used in the present work.

\paragraph{Finite-Sample Analysis.} We here focus on the most related works on finite-sample analysis through the lens of stochastic approximation   \cite{robbins1951stochastic,benaim2006dynamics,borkar2009stochastic,benveniste2012adaptive,kushner2012stochastic}, which is an iterative framework for solving root-finding problems, and has been widely applied in continuous optimization (to solve for the zeros of the gradient operator) \cite{lan2020first} and RL (to solve the Bellman equation) \cite{bertsekas1996neuro}. Specifically, for linear stochastic approximation, finite-sample analysis has been conducted in \cite{bhandari2018finite,srikant2019finite,lakshminarayanan2018linear,mou2020linear,konda2004convergence}, and for nonlinear stochastic approximation, finite-sample analysis has been conducted in \cite{even2003learning,li2020sample,chen2020finite,qu2020finite,chen2021lyapunov,yan2022efficacy,zou2019finite}, motivated by various RL algorithms, such as $Q$-learning  \cite{watkins1992q} and SARSA \cite{rummery1994line}. In addition, stochastic gradient descent/ascent, as a special case of stochastic approximation, has been extensively studied in the literature; see \cite{bottou2018optimization,beck2017first,lan2020first} and the references therein for the finite-sample analysis of stochastic gradient algorithms.

\paragraph{Best-Response-Based Dynamics in Stochastic Games.} Learning in games has been one of the core topics in game theory, and 
the \emph{best-response-based}  dynamics has served as one of the most classical ones \cite{fudenberg1998theory}. A prominent example of such dynamics is the (smoothed) fictitious play \cite{brown1951iterative,robinson1951iterative,ref:Fudenberg93}, wherein each player responds with the (smoothed) best action it can choose given its belief of the opponents' strategies. Such best-response-type dynamics provide a way to justify \emph{equilibria} as the long-run outcome of bounded rational players optimizing myopically in games  \cite{fudenberg1998theory}. Despite their long history, the literature on understanding the behavior of best-response-type dynamics  in \emph{stochastic games}  \cite{shapley1953stochastic,fink1964equilibrium} -- a fundamental framework of MARL \cite{littman1994markov} and games in dynamic environments -- has not become rich until recently  \cite{sayin2021decentralized,sayin2022fictitious,sayin2022fictitious2,baudin2022fictitious,baudinsmooth,ozdaglar2021independent,chenzy2021sample,cai2023uncoupled}. The results, however, either concern the tabular setting without function approximation, or have established only \emph{asymptotic} convergence guarantees.
Note that, in \cite{leslie2005individual}, the authors proposed an individual $Q$-learning algorithm for zero-sum matrix games and established the asymptotic convergence, where they connected the ordinary differential equation (ODE) associated with their $Q$-function to the smoothed best-response dynamics via the chain rule of calculus. This is similar in spirit to our change of variable in the analysis. 

The existing work most related to ours is \cite{chen2023finite}, where the authors proposed a learning algorithm that combines value iteration with smoothed best-response dynamics, and conducted last-iterate finite-sample analysis. We now highlight the main differences between this work and \cite{chen2023finite}. Algorithmically, since \cite{chen2023finite} considers the tabular setting while we consider the function approximation setting, the learning algorithms are largely different as we cannot directly operate in the policy space,  but have to operate in the parameter space of the function approximation class. Technically, while both our work and \cite{chen2023finite} have used Lyapunov-based arguments to establish finite-sample guarantees, the Lyapunov functions, especially the one we use to study the slow-timescale iterates $\{\theta_{t,k}^i\}$ for computing the policies, are largely different from those used in \cite{chen2023finite}. In particular, since \cite{chen2023finite} directly uses the smoothed best-response dynamics for their policy update equation, their Lyapunov function for studying the policies is the ERNG, which is not the case in our setting. The construction of a valid Lyapunov function for $\theta_{t,k}^i$ presents one of our main technical novelties, as illustrated in Section \ref{subsec:contribution}.  

\paragraph{Independent Learning in Games.} Independent learning is a desired property in learning in games, where each player makes decisions using local information, without communication or coordination with other players. It has been extensively studied in normal-form/matrix game settings in the forms of (smoothed) fictitious play \cite{brown1951iterative,robinson1951iterative,ref:Fudenberg93}, and in general, no-regret learning \cite{cesa2006prediction}. In stochastic games, provable (independent) learning has received increasing attention recently \cite{bai2020near,daskalakis2020independent,zhang2020model,sayin2021decentralized,ozdaglar2021independent,wei2021last,sayin2022fictitious,sayin2022fictitious2,baudin2022fictitious,baudinsmooth,jin2021v,songcan,mao2022improving,zhangpolicy2022,ding2022independent,daskalakis2022complexity,erez2022regret,song2023can,cai2023uncoupled,chen2023finite},  as well as the hardness results of independent learning \cite{liu2022learning,foster2023hardness}. However, these works all focused on the \emph{tabular} setting without resorting to function approximation. Several recent works that have incorporated function approximation in learning in stochastic games \cite{xie2020learning,huangtowards,zhao2021provably,jin2022power,chen2022unified,ni2022representation,foster2023complexity,cisneros2023finite} are mostly focused on the online exploration setting with regret guarantees, a more challenging setting, without last-iterate convergence guarantees. More importantly, most of these algorithms are either asymmetric or require coordination among the agents, making them not independent learning algorithms. Very recently, \cite{wang2023breaking,cui2023breaking} have both studied independent function approximation for finite-horizon episodic stochastic games, in the online exploration setting. However, they used the regret bound as the performance metric, which does not imply last-iterate finite-sample guarantees and is thus not comparable to our results. 

\paragraph{Outline.} The rest of this paper is organized as follows. In Section \ref{sec:setup}, we present the problem formulation, which includes the model of a zero-sum stochastic game and the notation we will use throughout the paper. In Section \ref{sec:algorithm}, we introduce our learning dynamics, i.e., the two-timescale $Q$-learning with function approximation, its connection to the celebrated DQN algorithm, and how it can be viewed as an approximate version of the minimax value iteration for zero-sum stochastic games. In Section \ref{sec:results}, we present our main theoretical result -- the finite-sample guarantee of our proposed learning dynamics, followed by a detailed roadmap of the construction of a valid Lyapunov function to study the slow-timescale iterates in Section \ref{sec:analysis}. Finally, we conclude the paper in Section \ref{sec:conclusion}. The complete proof of all theoretical results can be found in the appendix.

\section{Problem Formulation}\label{sec:setup}

An infinite-horizon two-player zero-sum stochastic game is defined by a tuple $(\mathcal{S},\mathcal{A}^1,\mathcal{A}^2,p,R_1,R_2,\gamma)$, where $\mathcal{S}$ is a finite state space, $\mathcal{A}^i$, $i\in \{1,2\}$, is the finite action space of player $i$, $p:\mathcal{S}\times \mathcal{A}^1\times \mathcal{A}^2\mapsto \Delta(\mathcal{S})$ is the state transition kernel (where $\Delta(\mathcal{S})$ stands for the probability simplex supported on $\mathcal{S}$), $R_i:\mathcal{S}\times \mathcal{A}^i\times \mathcal{A}^{-i}\mapsto \mathbb{R}$, $i\in \{1,2\}$, is the reward function of player $i$, and $\gamma\in [0,1)$ is the discount factor. Here, we use $-i$ to denote the index of the opponent of player $i$. In zero-sum stochastic games, the reward functions satisfy $\sum_{i=1,2}R_i(s,a^i,a^{-i})=0$ for all $(s,a^1,a^2)$. Throughout, we assume that $\max_{s,a^1,a^2}|R_1(s,a^1,a^2)|\leq 1$,  which is without loss of generality because we are working with a finite stochastic game\footnote{Our results can be generalized to a stochastic game with an infinite state space but a finite action space. We focus on the finite setting for ease of presentation.}. 

At time step $k\geq 0$, denote the current state of the environment by $\mathcal{S}_k$. For $i\in \{1,2\}$, player $i$ plays the game with an action $A_k^i\sim \pi^i(\cdot\mid S_k)$, where $\pi^i$ is called a policy and can be viewed as a mapping from the state space $\mathcal{S}$ to the set of probability distributions supported on the action space $\mathcal{A}^i$. The environment then transitions to a new state $S_{k+1}$ based on the transition kernel $p(\cdot\mid S_k,A_k^1,A_k^2)$, and returns player $i$ a one-stage reward $R_i(S_k,A_k^i,A_k^{-i})$. The process is then repeated.

Given a joint policy $\pi=(\pi^1,\pi^2)$ and $i\in \{1,2\}$, the local state-action function $q_\pi^i:\mathcal{S}\times \mathcal{A}^i\mapsto\mathbb{R}$, the global value function $v_\pi^i:\mathcal{S}\mapsto\mathbb{R}$, and the expected value function $U^i(\pi^i,\pi^{-i})\in\mathbb{R}$ are defined as 
\begin{align*}
	q_\pi^i(s,a^i)=\mathbb{E}_\pi\left[\sum_{k=0}^{\infty}\gamma^kR_i(S_k,A_k^i,A_k^{-i})\;\middle|\; S_0=s,A_0^i=a^i\right],\quad \forall\;(s,a^i)\in\mathcal{S}\times \mathcal{A}^i,
\end{align*}
$v_\pi^i(s)=\mathbb{E}_{A^i\sim \pi^i(\cdot|s)}[q_\pi^i(s,A^i)]$ for all $s\in\mathcal{S}$, and $U^i(\pi^i,\pi^{-i})=\mathbb{E}_{S\sim \rho_o(\cdot)}[v_\pi^i(S)]$, respectively, where we use $\mathbb{E}_\pi[\,\cdot\,]$ to indicate that the actions are chosen according to the joint policy $\pi=(\pi^1,\pi^2)$, and $\rho_o(\cdot)$ is an arbitrary initial distribution over the states. 

Unlike in the single-agent setting, there may not exist a universal optimal policy for each player. Therefore, we use the Nash gap $\text{NG}(\cdot,\cdot)$ to measure the performance of a joint policy $\pi=(\pi^1,\pi^2)$, which is defined as
\begin{align*}
	\Gap{\pi^1,\pi^2}=\sum_{i=1,2}\big\{\max_{\hat{\pi}^i}U^i(\hat{\pi}^i,\pi^{-i})-U^i(\pi^i,\pi^{-i})\big\}.
\end{align*}
Note that $\Gap{\pi^1,\pi^2}=0$ if and only if $\max_{\hat{\pi}^i}U^i(\hat{\pi}^i,\pi^{-i})=U^i(\pi^i,\pi^{-i})$ for $i\in \{1,2\}$, which implies that the joint policy $\pi=(\pi^1,\pi^2)$ is a Nash equilibrium \cite{nash1950equilibrium}. 

\paragraph{Notation.} Given $i\in \{1,2\}$, for any $x\in \mathbb{R}^{|\mathcal{S}||\mathcal{A}^i||\mathcal{A}^{-i}|}$, we use $x(s)$ to denote an $|\mathcal{A}^i|\times |\mathcal{A}^{-i}|$ matrix with its $(a^i,a^{-i})$'s entry being $x(s,a^i,a^{-i})$. Similarly, for any $y\in \mathbb{R}^{|\mathcal{S}||\mathcal{A}^i|}$, we use $y(s)$ to denote an $|\mathcal{A}^i|$ -- dimensional vector with its $a^i$-th entry being $y(s,a^i)$. For example, for any $s\in\mathcal{S}$, $\pi^i(s)$ is a vector living in the probability simplex $\Delta(\mathcal{A}^i)$, and $q_\pi^i(s)$ is a real-valued vector in $\mathbb{R}^{|\mathcal{A}^i|}$.

\section{The Learning Dynamics}\label{sec:algorithm}

In this section, we present our two-timescale $Q$-learning with function approximation. Before that, we formally introduce the function approximation architecture.

\paragraph{Function Approximation.} Given $i\in \{1,2\}$, let $q_{w^i}^i\in\mathbb{R}^{|\mathcal{S}||\mathcal{A}^i|}$ be the approximation of the local $q$-function of player $i$, where $w^i\in\mathbb{R}^{d_i}$ is the parameter (or weight) and $d_i$ is a positive integer that is usually smaller than $|\mathcal{S}||\mathcal{A}^i|$. We assume that $q_{w^i}^i$ is differentiable with respect to $w^i$. As an example of function approximation, consider a neural network with the dimension of the input vector  being $d_i$ and the dimension of the output vector being $|\mathcal{S}||\mathcal{A}^i|$. With the parametrization of the local $q$-function introduced above, we consider using the following softmax policy class:
\begin{align*}
    \Pi^i=\left\{\pi_{\theta^i}^i(a^i\mid s)=\frac{\exp(q_{\theta^i}^i(s,a^i)/\tau)}{\sum_{\Tilde{a}^i\in\mathcal{A}^i}\exp(q_{\theta^i}^i(s,\tilde{a}^i)/\tau)},\,\forall\,(s,a^i)\,\middle|\,\theta^i\in\mathbb{R}^{d_i}\right\},
\end{align*}
where $\tau>0$ is called the temperature. For simplicity of notation, given $\tau>0$, let 
$\sigma_\tau: \mathbb{R}^{|\mathcal{A}^i|}\mapsto\mathbb{R}^{|\mathcal{A}^i|}$ be the softmax operator with temperature $\tau$ defined as $[\sigma_\tau(x^i)](a^i)=\exp(x^i(a^i)/\tau)/\sum_{\Tilde{a}^i\in\mathcal{A}^i}\exp(x^i(\Tilde{a}^i)/\tau)$ for any $a^i\in\mathcal{A}^i$ and $x^i\in\mathbb{R}^{|\mathcal{A}^i|}$.

\paragraph{Linear Function Approximation.} A special case of function approximation is when $q_{w^i}^i$ is linear in $w^i$, which is called linear function approximation \cite{sutton2018reinforcement,tsitsiklis1997analysis}. Specifically, for $i\in \{1,2\}$, player $i$ is associated with a feature matrix $\Phi^i\in \mathbb{R}^{|\mathcal{S}|| \mathcal{A}^i|\times \mathbb{R}^{d_i}}$, and the local $q$-function is parametrized as $q_{w^i}^i=\Phi^iw^i$. Note that the tabular setting corresponds to $\Phi^i=I^i$, where $I^i$ is the identity matrix with dimension $|\mathcal{S}||\mathcal{A}^i|$. We assume without loss of generality that $\Phi^i$ has linearly independent columns and is normalized so that $\max_{s,a^i}\|\phi^i(s,a^i)\|_2\leq 1$ \cite{bertsekas1996neuro}, where $\phi^i(s,a^i)\in\mathbb{R}^{d_i}$ denotes the $(s,a^i)$-th row of $\Phi^i$. Moreover, for any $s\in\mathcal{S}$, let $\Phi^i_s\in\mathbb{R}^{|\mathcal{A}^i|\times d_i}$ be the ``sub-feature matrix'' of $\Phi^i$ at state $s$, that is, the $a^i$-th row of $\Phi_s^i$ is given by $\phi^i(s,a^i)$ for all $a^i\in\mathcal{A}^i$. 
In this case, the softmax policy class of player $i\in \{1,2\}$ can be compactly written as $\Pi^i=\{ \pi_{\theta^i}^i(s)=\sigma_\tau(\Phi_s^i\theta^i),\,  \,\forall\,s\in\mathcal{S}\mid \theta^i\in\mathbb{R}^{d_i}\}$, where we recall our notation that $\pi_{\theta^i}^i(s)\in \Delta(\mathcal{A}^i)$. 

\subsection{Two-Timescale $Q$-Learning with Function Approximation}
For $i\in \{1,2\}$, our two-timescale $Q$-learning with function approximation is presented in Algorithm \ref{algorithm:FA}. Note that Algorithm \ref{algorithm:FA} has two loops and maintains $4$ sets of iterates $\{w_{t,k}^i\}$, $\{\theta_{t,k}^i\}$, $\{\Bar{w}_t^i\}$, and $\{\Bar{\theta}_t\}$. In Algorithm \ref{algorithm:FA} Line $6$, we use $\ceil{\cdot}$ to represent the truncation operator with truncation radius $r=1/(1-\gamma)$. That is, for any $x\in\mathbb{R}$, we have $\ceil{x}=x$ when $x\in[-r,-r]$, $\ceil{x}=r$ when $x>r$, and $\ceil{x}=-r$ when $x<-r$. For a vector-valued $x$, the operator $\ceil{\cdot}$ operates component-wisely. In Algorithm \ref{algorithm:FA} Line 7, we use $\pj_M^i:\mathbb{R}^{d_i}\mapsto\mathbb{R}^{d_i}$ to represent the projection operator onto the $\ell_2$-norm ball with radius $M$ centered at the origin, that is, $\pj_M^i(x)=\arg\min_{y\in\mathbb{R}^{d_i}, \|y\|_2\leq M}\|x-y\|_2$, where the projection radius $M\geq 0$ is a tunable parameter. Both the truncation and the projection are introduced for the stability of our learning dynamics, which are common in the existing literature even in the single-agent setting \cite{chen2022target,munos2008finite,mnih2015human,zou2019finite}. 

\begin{algorithm}[ht]
	\caption{Two-Timescale $Q$-Learning with Function Approximation}
	\label{algorithm:FA} 
	\begin{algorithmic}[1]
		\STATE \textbf{Input:} Integers $K$ and $T$, initializations $w_{0,0}^i=\theta_{0,0}^i=\Bar{w}_0^i=\Bar{\theta}_0^i=0$, and $S_0\in\mathcal{S}$
  \FOR{$t=0,1,\cdots,T-1$}
		\FOR{$k=0,1,\cdots,K-1$}
		\STATE $\theta^i_{t,k+1}=\theta_{t,k}^i+\beta_k (w_{t,k}^i-\theta_{t,k}^i)$ and compute $\pi_{t,k+1}^i(S_k)=\sigma_\tau(q_{\theta_{t,k+1}^i}^i(S_k))$
		\STATE Play
		$A_k^i\sim \pi_{t,k+1}^i(\cdot\mid S_k)$ (against $A_k^{-i}$)  and observe $S_{k+1}\sim p(\cdot\mid S_k,A_k^i,A_k^{-i})$
  \STATE $\Delta_{t,k}^i=R_i(S_k,A_k^i,A_k^{-i})+\gamma \sigma_\tau(q_{\Bar{\theta}_t^i}^i(S_{k+1}))^\top \ceil{q_{\Bar{w}_t^i}^i(S_{k+1})} -q_{w_{t,k}^i}^i(S_k,A_k^i)$
		\STATE $w_{t,k+1}^i=\pj_M^i \big(w_{t,k}^i+\alpha_k\nabla\big[ q_{w_{t,k}^i}^i(S_k,A_k^i)\big]\Delta_{t,k}^i\big)$ 
\ENDFOR
        \STATE $\Bar{w}_{t+1}^i=w_{t,K}^i$ and $\Bar{\theta}_{t+1}^i=\theta_{t,K}^i$
\STATE $w_{t+1,0}^i=w_{t,K}^i$, $\theta_{t+1,0}^i=\theta_{t,K}^i$, and $S_0=S_K$
		\ENDFOR
	\end{algorithmic}
\end{algorithm} 

Before diving into the details of Algorithm \ref{algorithm:FA}, we see that the learning dynamics are payoff-based and independent because agent $i$ only needs to observe the current state of the environment and the realized payoff to implement the learning dynamics (cf. Algorithm \ref{algorithm:FA} Line $5$). In addition, the learning dynamics are symmetric between the two players and
are driven by a single trajectory of Markovian samples. Therefore, Algorithm \ref{algorithm:FA} can be implemented in an online manner without the need to access a generative model or constantly reset the system.

\paragraph{Algorithm Details.}
For a fixed outer-loop index $t$, consider the update equations in the inner loop of Algorithm \ref{algorithm:FA}. Note that $\Bar{w}_t^i$ and $\Bar{\theta}_t$ are fixed in the inner loop. The iterate $\theta_{t+1,k}^i$ is obtained by simply taking a convex combination (with parameter $\beta_k$) between its previous iterate $\theta_{t,k}^i$ and $w_{t,k}^i$ (cf. Algorithm \ref{algorithm:FA} Line $4$). The parameter $\beta_k$ is also called the stepsize, or learning rate. The iterate $w_{t,k}^i$ is updated according to a projected SGD (cf. Algorithm \ref{algorithm:FA} Line $7$) for minimizing a variant of the Bellman error, which will be elaborated in detail in Section \ref{subsec:viewpoint}. In the special case of linear function approximation, since $\nabla q_{w_{t,k}^i}^i(S_k,A_k^i)=\phi^i(S_k,A_k^i)$, Algorithm \ref{algorithm:FA} Line $7$ is in spirit to the TD-learning in RL \cite{sutton1988learning}, where $\Delta_{t,k}^i$ is known as the temperal difference and $\alpha_k$ is the stepsize. We require $\beta_k\ll \alpha_k$ so that $\theta_{t,k}^i$ evolves at a much slower rate compared to that of $w_{t,k}^i$, which is why we call Algorithm \ref{algorithm:FA} two-timescale $Q$-learning.

In the outer loop of Algorithm \ref{algorithm:FA}, we update $\Bar{w}_t^i$ and $\Bar{\theta}_t^i$ by synchronizing them to the values of $w_{t,K}^i$ and $\theta_{t,K}^i$ from the last inner loop. Following the terminology used in the existing literature \cite{mnih2015human}, we call $\Bar{w}_t^i$ and $\Bar{\theta}_t^i$ the target network parameters. Note that in single-agent $Q$-learning, there is only one set of target network parameters, which corresponds to $\{\Bar{w}_t^i\}$ in Algorithm \ref{algorithm:FA}. Here, we have another set of target network parameters $\{\Bar{\theta}_t\}$, which we use to compute the policies and are designed for the setting of zero-sum stochastic games. In Algorithm \ref{algorithm:FA} Line $10$, we set the initializations of the iterates for the next inner loop to be their corresponding last iterates of the previous inner loop,  ensuring that Algorithm \ref{algorithm:FA} follows a single trajectory of Markovian samples.

\paragraph{Connection to the DQN.} Suppose that we use the deep neural net as a means for function approximation. Then the main difference between Algorithm \ref{algorithm:FA} and the DQN \cite{mnih2015human} is the two-timescale structure. In fact, if we do not choose
$\beta_k\ll\alpha_k$ but instead do the opposite by setting $\beta_k\equiv 1$. Since every time we simply assign the fast-timescale iterate $w_{t,k}^i$ to the slow-timescale iterate $\theta_{t,k}^i$ (cf. Algorithm \ref{algorithm:FA} Line $4$), there are essentially only two sets of iterates $\{w_{t,k}^i\}$ and $\Bar{w}_t^i$. In this case, Algorithm \ref{algorithm:FA} reduces to the DQN in \cite{mnih2015human}, with modular changes such as replacing the softmax by the hardmax and the use of experience replay for efficient sampling. 

\begin{remark}
    In light of the discussion in the previous paragraph, the main algorithmic novelty of Algorithm \ref{algorithm:FA} is essentially that we introduced another set of parameters $\theta_{t,k}^i$ (along with its target network $\Bar{\theta}_t^i$) to compute the policies instead of directly taking softmax based on the local $q$-function associated with $w_{t,k}^i$. Although the update equation for $\theta_{t,k}^i$ is a simple convex combination, we illustrate in Section \ref{sec:analysis} that such an update equation resembles the smoothed best-response dynamics in zero-sum games, which enables us to show the last-iterate convergence of Algorithm \ref{algorithm:FA} in the special case of linear function approximation.
\end{remark}

\subsection{The Viewpoint of Algorithm \ref{algorithm:FA} as the Minimax Value Iteration}\label{subsec:viewpoint}

In single-agent $Q$-learning with target networks, it was shown in the literature that the algorithm mimics the value iteration \cite{munos2008finite,chen2022target}. In this section, we also elaborate on how to view Algorithm \ref{algorithm:FA} as an approximate version of the minimax value iteration, which is a well-known algorithm for solving zero-sum stochastic games. This also serves as a starting point for our analysis as will be presented in Section \ref{sec:analysis}. We begin by introducing the minimax value iteration to provide context.  

\paragraph{The Minimax Value Iteration.} 
Given $i\in \{1,2\}$, let $\mathcal{B}^i:\mathbb{R}^{|\mathcal{S}|}\mapsto\mathbb{R}^{|\mathcal{S}|}$ be the minimax Bellman operator defined as $\mathcal{B}^i(v^i)(s)=\max_{\mu^i\in\Delta(\mathcal{A}^i)}\min_{\mu^{-i}\in\Delta(\mathcal{A}^{-i})}(\mu^i)^\top\mathcal{T}^i(v^i)(s)\mu^{-i}$ for all $s\in\mathcal{S}$,
where the operator $\mathcal{T}^i:\mathbb{R}^{|\mathcal{S}|}\mapsto \mathbb{R}^{|\mathcal{S}||\mathcal{A}^i||\mathcal{A}^{-i}|}$ is defined as
\begin{align}\label{def:need_for_Bellman}
	\mathcal{T}^i(v^i)(s,a^i,a^{-i})=R_i(s,a^i,a^{-i})+\gamma \mathbb{E}\left[v^i(S_1)\mid S_0=s,A_0^i=a^i,A_0^{-i}=a^{-i}\right]
\end{align}
for all $v^i\in\mathbb{R}^{|\mathcal{S}|}$ and $(s,a^i,a^{-i})$. Here, we recall our notation that $\mathcal{T}^i(v^i)(s)$ is an $|\mathcal{A}^i|\times |\mathcal{A}^{-i}|$ matrix with its $(a^i,a^{-i})$-th entry being $\mathcal{T}^i(v^i)(s,a^i,a^{-i})$.
In the minimax value iteration, Player $i$ initializes $v_0^i\in\mathbb{R}^{|\mathcal{S}|}$ arbitrarily and iteratively updates $v_t^i$ according to 
\begin{align}\label{eq:VI}
    v_{t+1}^i(s)=\mathcal{B}^i(v_t^i)(s),\quad \forall\,s\in\mathcal{S}.
\end{align}
It is well-known that the operator $\mathcal{B}^i(\cdot)$ is a contraction mapping \cite{littman1994markov}, hence having a unique fixed-point $v_*^i\in\mathbb{R}^{|\mathcal{S}|}$ \cite{banach1922operations}. As a result, the minimax value iteration converges geometrically fast to $v_*^i$. In addition, given $i\in \{1,2\}$, a joint policy $\pi=(\pi^1,\pi^2)$ satisfying  $\pi^i(s)\mathcal{T}^i(v_*^i)(s)\pi^{-i}(s)=\max_{\mu^i\in\Delta(\mathcal{A}^i)}\min_{\mu^{-i}\in\Delta(\mathcal{A}^{-i})}(\mu^i)^\top\mathcal{T}^i(v_*^i)(s)\mu^{-i}$ 
is a Nash equilibrium of the zero-sum stochastic game. Although the minimax value iteration is provably efficient, it is not a payoff-based independent learning algorithm, as each agent needs to know the underlying model of the stochastic game (i.e., the transition kernel and the reward function) to implement the algorithm. Next, we illustrate why Algorithm \ref{algorithm:FA} can be viewed as an approximate version of the minimax value iteration, while only using the realized payoffs for implementation.

We begin by reformulating Algorithm \ref{algorithm:FA}. Observe that the target network parameters $\Bar{\theta}_t^i$ and $\Bar{w}_t^i$ are only used in Algorithm \ref{algorithm:FA} Line $7$ and are synchronized to $\theta_{t,K}^i$ and $w_{t,K}^i$ in the outer loop. Therefore, letting $v_{t+1}^i(s)= \sigma_\tau(q_{\theta_{t,K}^i}^i(s))^\top \ceil{q_{\theta_{t,K}^i}^i(s)}$ for all $s\in\mathcal{S}$ and $t$, we can rewrite Algorithm \ref{algorithm:FA} in an equivalent form in Algorithm \ref{algorithm:reformulation}. Note that Algorithm \ref{algorithm:reformulation} is not meant for implementation but only used for the analysis. 

\begin{algorithm}[ht]
	\caption{Equivalent Formulation of Algorithm \ref{algorithm:FA}}
	\label{algorithm:reformulation} 
	\begin{algorithmic}[1]
		\STATE \textbf{Input:} Integers $K$ and $T$, initializations $w_{0,0}^i=\theta_{0,0}^i=0$, $v_0^i=0$, and $S_0\in\mathcal{S}$
  \FOR{$t=0,1,\cdots,T-1$}
		\FOR{$k=0,1,\cdots,K-1$}
		\STATE $\theta^i_{t,k+1}=\theta_{t,k}^i+\beta_k (w_{t,k}^i-\theta_{t,k}^i)$ and $\pi_{t,k+1}^i(s)=\sigma_\tau(q_{\theta_{t,k+1}^i}^i(s))$ for all $s\in\mathcal{S}$
		\STATE Play
		$A_k^i\sim \pi_{t,k+1}^i(\cdot\mid S_k)$ (against $A_k^{-i}$)  and observe $S_{k+1}\sim p(\cdot\mid S_k,A_k^i,A_k^{-i})$
		\STATE $w_{t,k+1}^i=\pj_M^i \big(w_{t,k}^i+\alpha_k\nabla\big[ q_{w_{t,k}^i}^i(S_k,A_k^i)\big](R_i(S_k,A_k^i,A_k^{-i})+\gamma v_t^i(S_{k+1}) -q_{w_{t,k}^i}^i(S_k,A_k^i))\big)$ 
        \ENDFOR
        \STATE $v_{t+1}^i(s)= \pi_{t,K}^i(s)^\top \ceil{q_{w_{t,K}^i}^i(s)}$ for all $s\in\mathcal{S}$
        \STATE $w_{t+1,0}^i=w_{t,K}^i$, $\theta_{t+1,0}^i=\theta_{t,K}^i$, and $S_0=S_K$
		\ENDFOR
	\end{algorithmic}
\end{algorithm}

To interpret Algorithm \ref{algorithm:reformulation} as the minimax value iteration, in view of Eq. (\ref{eq:VI}) and Algorithm \ref{algorithm:reformulation} Line $8$, we would like to have
\begin{align}
    q_{w_{t,K}^i}^i(s)\approx\;& \mathcal{T}^i(v_t^i)(s)\pi_{t,K}^{-i}(s),\quad \forall\,s\in\mathcal{S},\label{eq:q_elaboration}\\
    \pi^i_{t,K}(s)^\top \mathcal{T}^i(v_t^i)(s)\pi_{t,K}^{-i}(s)\approx \;&\max_{\mu^i\in\Delta(\mathcal{A}^i)}\min_{\mu^{-i}\in\Delta(\mathcal{A}^{-i})}(\mu^i)^\top \mathcal{T}^i(v_t^i)(s)\mu^{-i},\quad \forall\,s\in\mathcal{S}.\label{eq:pi_elaboration}
\end{align}
Note that
Eq. (\ref{eq:pi_elaboration}) means that the joint policy (at state $s$) $(\pi_{t,K}^i(s),\pi_{t,K}^{-i}(s))$  is an approximate Nash equilibrium of the matrix game with the payoff matrix being $\mathcal{T}^i(v_t^i)(s)$.

\paragraph{The Fast-Timescale Iterates.} We first elaborate on how Eq. (\ref{eq:q_elaboration}) can be achieved with Algorithm \ref{algorithm:reformulation} Line $6$, which updates the fast-timescale parameter $w_{t,k}^i$. Given $i\in \{1,2\}$ and $t\geq 0$, for any $v_t^i\in\mathbb{R}^{|\mathcal{S}|}$ and joint policy $(\pi_t^1,\pi_t^2)$ that are \textit{not} time-varying in $k$, consider minimizing the objective function
\begin{align*}
    J(w^i)=\;&\sum_{s\in\mathcal{S}}\mu_{\pi_t}(s)\sum_{a^i\in\mathcal{A}^i}\pi_t^i(a^i|s)\left(\sum_{a^{-i}\in\mathcal{A}^{-i}}\mathcal{T}^i(v_t^i)(s,a^i,a^{-i})\pi_t^{-i}(a^{-i}|s)-q_{w^i}^i(s,a^i)\right)^2
\end{align*}
using the projected SGD, where $\mu_{\pi_t}\in\mathbb{R}^{|\mathcal{S}|}$ represents the stationary distribution of the Markov chain $\{S_k\}$ induced by the joint policy $(\pi_t^1,\pi_t^2)$, assuming the uniform ergodicity of $\{S_k\}$. Specifically, with a sample $(S_k,A_k^1,A_k^2,S_{k+1})$ such that $S_k\sim \mu_{\pi_t}(\cdot)$, $A_k^i\sim \pi_t^i(\cdot|s)$ for $i\in \{1,2\}$, and $S_{k+1}\sim p(\cdot\mid S_k,A_k^1,A_k^2)$, in view of the explicit expression of the operator $\mathcal{T}^i(\cdot)$ (cf. Eq. (\ref{def:need_for_Bellman})), the projected SGD for minimizing $J(\cdot)$ is given as
\begin{align}\label{eq:interpret}
    w_{t,k+1}^i=\pj_M^i \big(w_{t,k}^i+\alpha_k\nabla\big[ q_{w_{t,k}^i}^i(S_k,A_k^i)\big](R_i(S_k,A_k^i,A_k^{-i})+\gamma v_t^i(S_{k+1}) -q_{w_{t,k}^i}^i(S_k,A_k^i))\big),
\end{align}
where $M$ is the projection radius. 
When using linear function approximation, the objective function $J(\cdot)$ is quadratic (hence strongly convex) in $w^i$. Therefore, with appropriately chosen stepsizes and a large enough $K$, the output $w_{t,K}^i$ of the projected SGD in Eq. (\ref{eq:interpret}) is guaranteed to converge to the global optimal solution of $J(\cdot)$. In addition, when the function approximation class has enough representation power, we would expect the $q$-function associated with $w_{t,K}^i$ to satisfy $q_{w_{t,K}^i}^i(s)\approx \mathcal{T}^i(v_t^i)(s)\pi_t^{-i}(s)$ for all $s\in\mathcal{S}$.

Although we motivated the projected SGD update (\ref{eq:interpret}) when the joint policy $(\pi_t^1,\pi_t^2)$ is fixed in updating $w_{t,k}^i$, the reasoning here can be easily extended to the inner loop of Algorithm \ref{algorithm:reformulation}, where $\pi_{t,k}^i$ is time-varying in $k$. Specifically, in the inner loop of Algorithm \ref{algorithm:reformulation}, observe that the parameters $\{\theta_{t,k}^i\}$  of the policies $\{\pi_{t,k}^i\}$ are updated in a slower timescale compared to the parameters $\{w_{t,k}^i\}$ of the local $q$-functions $\{q_{w_{t,k}^i}^i\}$. Therefore, from the perspective of $w_{t,k}^i$, the policy $\pi_{t,k}^i$ is as if it were stationary in $k$. As a result, we would expect $q_{w_{t,K}^i}^i(s)\approx \mathcal{T}^i(v_t^i)(s)\pi_{t,K}^{-i}(s)$, as desired in Eq. (\ref{eq:q_elaboration}). To make the claim rigorous, we need to carefully analyze the two-timescale stochastic approximation in the inner loop of Algorithm \ref{algorithm:reformulation}, which is done in Appendix \ref{ap:inner_loop}.

\paragraph{The Slow-Timescale Iterates.} 
Now that we have elaborated that the fast-timescale iterate in the inner loop of Algorithm \ref{algorithm:reformulation} provides us with a $q_{w_{t,K}^i}^i$ such that $q_{w_{t,K}^i}^i(s)\approx\mathcal{T}^i(v_t^i)(s)\pi_{t,K}^{-i}(s)$, it remains to demonstrate that the slow-timescale iterates $\{\theta_{t,k}^i\}$ provide us with a joint policy that satisfies Eq. (\ref{eq:pi_elaboration}). The problem here is essentially to solve a matrix game (for each state $s$) with the payoff matrix $\mathcal{T}^i(v_t^i)(s)$. To achieve that, one of the most classical and promising learning dynamics is the smoothed best-response dynamics \cite{ref:Fudenberg93,harris1998rate,chen2023finite}, which maintain a policy $\pi_{t,k}^i$, and update it iteratively towards the smoothed best response to the opponent's latest policy according to
\begin{align}\label{eq:best_response}
	\pi_{t,k+1}^i(s)=\pi_{t,k}^i(s)+\beta_k(\sigma_\tau (\mathcal{T}^i(v_t^i)(s)\pi_{t,k}^{-i}(s))-\pi_{t,k}^i(s)),\quad \forall\;s\in\mathcal{S},
\end{align}
where $\beta_k$ is the stepsize and $\sigma_\tau(\cdot)$ is the softmax operator. Suppose that both players follow the smoothed best-response dynamics, they are guaranteed to find the Nash distribution \cite{leslie2005individual,hofbauer2005learning}, which approximates a Nash equilibrium as the temperature $\tau$ goes to zero. There are two challenges in implementing Eq. (\ref{eq:best_response}). First, it requires the quantity $\mathcal{T}^i(v_t^i)(s)\pi_{t,k}^{-i}(s)$ (which involves the opponent's policy and the underlying model), and second, it operates directly in the policy space instead of in the parameter space, which violates the purpose of using function approximation. The first challenge has been overcome by introducing the fast-timescale iterate $\{w_{t,k}^i\}$ to estimate the quantity $\mathcal{T}^i(v_t^i)(s)\pi_{t,k}^{-i}(s)$ via the projected SGD. To overcome the second challenge, we need to extend the idea of the smoothed best-response dynamics to the parameter space of the function approximation class.

In view of Eq. (\ref{eq:best_response}), the key insight here is that, in each iteration, the policy is updated by taking a step towards the smoothed best-response to the opponent's latest policy. Inspired by this idea,
in Algorithm \ref{algorithm:reformulation} Line $4$, we propose to update the policy parameter $\theta_{t,k+1}^i$ by taking a convex combination between the previous policy parameter $\theta_{t,k}^i$ and the fast-timescale iterate $w_{t,k}^i$. Recall that the $q$-function associated with $w_{t,k}^i$ is constructed as an estimator of the marginalized payoff to the opponent $\mathcal{T}^i(v_t^i)(s)\pi_{t,k}^{-i}(s)$.
Therefore, this is in the same spirit as the smoothed best-response dynamics in the policy space as, in each iteration, the parameter of the policy is updated by taking a step towards the parameter of the smoothed best response, which is estimated through the local $q$-function associated with iterate $w_{t,k}^i$. 

\begin{remark}
    Although Algorithm \ref{algorithm:reformulation} Line $4$ is simple, intuitive, and mimics the smoothed best-response  dynamics in the parameter space, rigorously analyzing it is particularly challenging. This is because, when using softmax policy, even under linear function approximation, a linear update in the parameter space does not imply a linear update in the policy space as in the smoothed best-response dynamics in Eq. (\ref{eq:best_response}). Therefore, in a Lyapunov-based approach, there is no existing Lyapunov function to study the update equation for the slow-timescale iterates, the construction of which presents our main technical novelty. This will be illustrated in detail in Section \ref{sec:analysis}.
\end{remark}

\section{Last-Iterate Finite-Sample Analysis}\label{sec:results}

In the rest of this paper, we consider using linear function approximation. Denote $\Pi=\{(\pi^1,\pi^2)\mid \pi^i\in \Pi^i, \,\forall\,i\in \{1,2\}\}$ as the set of joint policies that are representable using our softmax policy class. Before presenting our main theoretical results, we first state our assumptions.

\begin{assumption}\label{as:ergodicity}
\textit{For any joint policy $(\pi^1,\pi^2)\in\Pi$, the Markov chain $\{S_k\}$ induced by $(\pi^1,\pi^2)$ is irreducible and aperiodic. In addition,
\begin{enumerate}[(1)]
    \item there exist $C>0$ and $\rho\in (0,1)$ such that $\sup_{\pi\in \Pi}\max_{s\in\mathcal{S}}\left\|P_\pi(s,\cdot)-\mu_\pi(\cdot)\right\|_{\text{TV}}\leq C\rho^k$ for all $k\geq 0$, where $P_\pi\in\mathbb{R}^{|\mathcal{S}|\times |\mathcal{S}|}$ and $\mu_\pi\in\Delta(\mathcal{S})$ are the transition matrix and the stationary distribution of the Markov chain $\{S_k\}$ induced by $\pi$, respectively; 
    \item it holds that $\lambda:=\min_{i\in \{1,2\}}\inf_{\pi\in \Pi}\lambda_{\min}((\Phi^i)^\top D_\pi^i \Phi^i)>0$, where $D_\pi^i\in\mathbb{R}^{|\mathcal{S}||\mathcal{A}^i|\times |\mathcal{S}||\mathcal{A}^i|}$ is a diagonal matrix with diagonal components $\{\mu_\pi(s)\pi^i(a^i|s)\}_{(s,a^i)\in\mathcal{S}\times \mathcal{A}^i}$ and $\lambda_{\min}(\cdot)$ returns the smallest eigenvalue of a symmetric matrix.
\end{enumerate}}
\end{assumption}

Assumption \ref{as:ergodicity} is standard in the existing literature studying RL algorithms under time-varying policies \cite{zou2019finite,khodadadian2021finite,chenziyi2022sample,chenzy2021sample,xu2021sample,wu2020finite,qiu2019finite}, which ensures that all policies encountered from the algorithm trajectory have enough exploration. Since we use softmax policies and both the fast and the slow-timescale iterates are uniformly bounded (due to the projection operator $\pj_M^i(\cdot)$), there exists $\ell>0$ such that $\pi_{t,k}^i(a^i\mid s)\geq \ell$ for all $(s,a^i)$, $i\in \{1,2\}$, and $t,k$. Therefore, with some additional work, it can be shown that Assumption  \ref{as:ergodicity} is satisfied if the following (weaker) assumption is imposed. 

\begin{assumption}\label{as:weaker}
	\textit{There exists a joint policy $(\pi_b^1,\pi_b^2)$ such that the induced Markov chain $\{S_k\}$ is irreducible and aperiodic.}
\end{assumption}

The implication from Assumption \ref{as:weaker} to Assumption \ref{as:ergodicity} was formally established in \cite[Lemma 4.1]{chen2023finite}. However, in that case, the convergence rate of the learning dynamics involves problem-dependent constants that are implicit functions of the underlying transition probabilities of the stochastic game. In this work, for ease of exposition, we directly impose Assumption \ref{as:ergodicity}. Given a precision level $\delta>0$, we define 
\begin{align}\label{eq:def:mixing_time}
    z_\delta=\min\left\{k\geq 0 \,:\,\sup_{\pi\in \Pi}\max_{s\in\mathcal{S}}\left\|P_\pi^k(s,\cdot)-\mu_\pi(\cdot)\right\|_{\text{TV}}
    \leq \delta\right\},
\end{align}
which can be viewed as the uniform mixing time of the Markov chain $\{S_k\}$ induced by any joint policy from the policy class $\Pi$. Under Assumption \ref{as:ergodicity}, it is easy to see that 
$z_\delta=\mathcal{O}(\log(1/\delta))$.

To state the next assumption, we need to define a variant of the inherent Bellman error. For any joint policy $(\pi^1,\pi^2)\in\Pi$ and $i\in \{1,2\}$, let $\Bar{\mathcal{H}}_\pi^i:\mathbb{R}^{|\mathcal{S}||\mathcal{A}^i|}\mapsto\mathbb{R}^{|\mathcal{S}||\mathcal{A}^i|}$ be an operator defined as
\begin{align*}
    [\Bar{\mathcal{H}}_\pi^i(q^i)](s,a^i)=\sum_{a^{-i}\in\mathcal{A}^{-i}}\left(R_i(s,a^i,a^{-i})+\gamma \mathbb{E}\left[\pi^i(S_1)^\top q^i(S_1)\,\middle|\,S_0=s,A_0^i=a^i\right]\right)\pi^{-i}(a^{-i}\mid s)
\end{align*}
for all $(s,a^i)\in\mathcal{S}\times \mathcal{A}^i$ and $q^i\in\mathbb{R}^{|\mathcal{S}||\mathcal{A}^i|}$, where we recall our notation that $\pi^i(s)\in\Delta(\mathcal{A}^i)$ and $q^i(s)\in\mathbb{R}^{|\mathcal{A}^i|}$ for all $s\in\mathcal{S}$. Define 
\begin{align}\label{def:inherent Bellman error}
    \mathcal{E}=\sum_{i=1,2}\sup_{\pi\in\Pi,\Tilde{w}^i\in\mathbb{R}^{d_i}}\inf_{w^i\in\mathbb{R}^{d_i}}\left\|\Phi^i w^i-\Bar{\mathcal{H}}_\pi^i(\ceil{\Phi^i\Tilde{w}^i})\right\|_\infty.
\end{align}

\begin{assumption}\label{as:Bellman_Completeness}
\textit{It holds that $\mathcal{E}=0$.}    
\end{assumption}

In the special case of tabular setting, since $\Phi^i=I_{|\mathcal{S}||\mathcal{A}^i|}$ is the identity matrix, Assumption \ref{as:Bellman_Completeness} is automatically satisfied. More generally, Assumption \ref{as:Bellman_Completeness} means that the function approximation class is rich enough, and can be viewed as an extension of the Bellman completeness assumption from the single-agent setting \cite{munos2008finite,zanette2020learning,agarwal2022online} to the setting of solving zero-sum stochastic games with independent learning. 

Next, we state the condition for choosing the stepsizes.
\begin{condition}\label{con:stepsize}
We choose $\tau\leq 1/(1-\gamma)$, $\alpha_0\leq 1/\lambda$, and $c_{\alpha,\beta}:=\beta_k/\alpha_k\leq \min(\frac{1}{64L_p A_{\max}},\lambda)$, where $A_{\max}=\max\{|\mathcal{A}^1|,|\mathcal{A}^2|\}$ and $L_p:=\frac{\log(\rho/C)}{\log(\rho)}+\frac{1}{1-\rho}$.
\end{condition}

Now, we present the last-iterate finite-sample analysis of Algorithm \ref{algorithm:FA} when using constant stepsizes $\alpha_k\equiv \alpha$ and $\beta_k\equiv \beta=c_{\alpha,\beta}\alpha$, where $c_{\alpha,\beta}\in (0,1)$ is the stepsize ratio. The results for using diminishing stepsizes are straightforward extensions. 

\begin{theorem}\label{thm:main}
Suppose that we use linear function approximation, Assumptions \ref{as:ergodicity} and \ref{as:Bellman_Completeness} are satisfied, the stepsizes verify Condition \ref{con:stepsize}, and the projection radius $M\geq \lambda^{-1/2}(1-\gamma)^{-1}$, Then, when both players follow the learning dynamics presented in Algorithm \ref{algorithm:FA}, we have for any $T\geq 0$ and $K\geq z_\beta$ that
\begin{align}\label{eq:finite_sample}
    \mathbb{E}[\text{NG}(\pi_{T,K}^1,\pi_{T,K}^2)]\lesssim \;&\underbrace{\frac{A_{\max}^2T\gamma^{T-1}}{\tau^2(1-\gamma)^4}}_{\mathcal{E}_1} +\underbrace{\frac{\tau \log(A_{\max})}{(1-\gamma)^2}}_{\mathcal{E}_2}+\underbrace{\frac{A_{\max}^3K\left(1-\beta/2\right)^{\frac{K-z_\beta-1}{2}}}{\tau^3\lambda(1-\gamma)^7}}_{\mathcal{E}_{3,1}}\nonumber\\
    &+\underbrace{\frac{A_{\max}^2z_\beta\alpha^{\frac{1}{2}}}{\tau^{\frac{5}{2}}\lambda^{\frac{5}{4}} (1-\gamma)^{\frac{13}{2}}}+\frac{A_{\max}^3\beta}{\tau^3 \lambda^2(1-\gamma)^7\alpha} +\frac{A_{\max}^3z_\beta^2\alpha}{\tau^4\lambda^{5/2} (1-\gamma)^7}+\frac{A_{\max}^5\beta^2}{\tau^5 \lambda^4(1-\gamma)^8\alpha^2}}_{\mathcal{E}_{3,2}},
\end{align}
where $z_\beta$ is the mixing time with precision $\beta$ (see Eq. (\ref{eq:def:mixing_time})), and we use $a\lesssim b$ to mean that there exists an absolute constant $c> 0$ such that $a\leq cb$. As a result, given $\epsilon>0$, to achieve $\mathbb{E}[\text{NG}(\pi_{T,K}^1,\pi_{T,K}^2)]\leq \epsilon$, the sample complexity is $\tilde{\mathcal{O}}(\text{poly}(1/\epsilon.1/\lambda,1/(1-\gamma),A_{\max})$.
\end{theorem}

To the best of our knowledge, Theorem \ref{thm:main} presents the first last-iterate finite-sample analysis of payoff-based independent learning that successfully incorperates function approximation. We next provide an interpretation for each term on the RHS of Eq. (\ref{eq:finite_sample}). 

The term $\mathcal{E}_1$ represents the error in the outer loop. As illustrated in Section \ref{subsec:viewpoint}, the outer loop of Algorithm \ref{algorithm:FA} is an approximation of the minimax value iteration. Therefore, since the minimax value iteration converges geometrically fast, the first term also converges geometrically in $T$. The term $\mathcal{E}_2$ represents the error due to using the softmax policy class with a fixed temperature $\tau$. To see why it arises, consider a special case where the stochastic game has a unique Nash equilibrium, which consists of a pair of deterministic policies. In this case, due to the deterministic nature of the Nash equilibrium and the stochastic nature of the softmax policies, we cannot make the Nash gap vanish with a fixed temperature $\tau$. As a result, the error term $\mathcal{E}_2$ is proportional to $\tau$, and can be made arbitrarily small by using a small enough $\tau$. The terms $\mathcal{E}_{3,1}$ and $\mathcal{E}_{3,2}$ on the RHS of Eq. (\ref{eq:finite_sample}) represent the convergence error in the inner loop of Algorithm \ref{algorithm:FA}, which is a two-timescale stochastic approximation algorithm. 
The term $\mathcal{E}_{3,1}$ goes to zero geometrically as $K$ goes to infinity and the terms in $\mathcal{E}_{3,2}$ are constants that are functions of the stepsize for the fast-timescale and the ratio between the two stepsizes. 

\section{Construction of the Lyapunov Function for the Slow-Timescale Iterates}\label{sec:analysis}

Throughout this section, we will work with Algorithm \ref{algorithm:reformulation} and linear function approximation. The overall framework for our proof is a Lyapunov-based approach. As illustrated in Section \ref{subsec:viewpoint}, the outer loop of Algorithm \ref{algorithm:reformulation} is designed to approximate the minimax value iteration, and the inner loop is designed to solve an induced matrix game for each state $s\in\mathcal{S}$ using a two-timescale stochastic approximation algorithm, where the fast-timescale iterates are updated using TD-learning to estimate the marginalized payoff to the opponent, and the slow-timescale iterates are updated in spirit to the smoothed best-response dynamics, albeit in the parameter space of the function approximation class. Inspired by \cite{chen2023finite}, we construct a Lyapunov function for each set of the iterates (i.e., $\{v_t^i\}$, $\{w_{t,k}^i\}$, and $\{\theta_{t,k}^i\}$), derive a Lyapunov drift inequality for each, and finally combine them together to establish the finite-sample bounds. The main challenge here (see the last remark of Section \ref{subsec:viewpoint} for more details) is to construct a valid Lyapunov function to study the slow-timescale iterates $\{\theta_{t,k}^i\}$, which will be the focus of this section. 

Since we are concerning the inner loop of Algorithm 
\ref{algorithm:reformulation}, we will omit the subscript $t$. To begin with, due to the fact that $\Phi^i$ has linearly independent columns, Algorithm \ref{algorithm:reformulation} Line $4$ can be equivalently written as 
\begin{align}
	q^i_{\theta_{k+1}^i}(s)=\;&q^i_{\theta_{k}^i}(s)+\beta_k \left(q^i_{w_{k}^i}(s)-q^i_{\theta_{k}^i}(s)\right)\nonumber\\
	=\;&q^i_{\theta_{k}^i}(s)+\beta_k \left(\mathcal{T}^i(v^i)(s)\sigma_\tau(q^{-i}_{\theta_{k}^{-i}}(s))-q^i_{\theta_{k}^i}(s)\right)+\beta_k \left(q^i_{w_{k}^i}(s)-\mathcal{T}^i(v^i)(s)\sigma_\tau(q^i_{\theta_{k}^{-i}}(s)\right)\label{eq:policy_equivalent1}
\end{align}
for all $s\in\mathcal{S}$. Recall that we designed the update equation for the fast-timescale iterates $\{w_k^i\}$ precisely to estimate $\mathcal{T}^i(v^i)(s)\sigma_\tau(q_{\theta_k^{-i}}^{-i}(s))=\mathcal{T}^i(v^i)(s)\pi_k^{-i}(s)$. Therefore, the last term in the RHS of Eq. (\ref{eq:policy_equivalent1}) corresponds to the estimation error of the fast-timescale iterate $w_k^i$, and can be controlled by studying Algorithm \ref{algorithm:reformulation} Line $6$. As a result, to understand the evolution of the slow-timescale iterates $\{\theta_k^i\}$, it essentially reduces to studying the following update equation:
\begin{equation}\label{eq:policy_equivalent_update}
	x_0^i\in\mathbb{R}^{|\mathcal{A}^i|},\quad x_{k+1}^i=x_k^i+\beta_k(X_i\sigma_\tau(x_k^{-i})-x_k^i),\quad i\in \{1,2\},
\end{equation}
where $X_i\in\mathbb{R}^{|\mathcal{A}^i|\times |\mathcal{A}^{-i}|}$ for $i\in \{1,2\}$. To make the connection between Eqs. (\ref{eq:policy_equivalent1}) and (\ref{eq:policy_equivalent_update}), fix $s\in\mathcal{S}$, identify $x_k^i=q_{\theta_k^i}^i(s)$ and $X_i=\mathcal{T}^i(v^i)(s)$ for $i\in \{1,2\}$. 
For the purpose of studying the slow-timescale iterates, we assume that $X_1+X_2^\top =0$\footnote{In Algorithm \ref{algorithm:reformulation},
	we do not necessarily have $\mathcal{T}^1(v_t^1)(s)+(\mathcal{T}^2(v_t^2)(s))^\top=0$ because $v_t^1$ and $v_t^2$ are independently maintained by players $1$ and $2$, and do not sum up to zero in general. Such a non-zero sum structure of the induced matrix game in the inner loop has to be taken into account (see Appendix \ref{ap:pf}), but is not the focus of this section.}. Our goal is to show that $(x_k^1,x_k^2)$ of Eq. (\ref{eq:policy_equivalent_update}) converges to $(x_*^1,x_*^2)$ that satisfies $x_*^i=X_i\sigma_\tau(x_*^{-i})$ for $i\in\{1,2\}$, in which case the induced policy pair $(\sigma_\tau(x_*^1),\sigma_\tau(x_*^2))$ is an approximate Nash equilibrium (or the Nash distribution \cite{leslie2003convergent}) of the matrix game with payoff matrices $X_1$ and $X_2$. 
In the rest of this section, we present the roadmap of the construction of our Lyapunov function to study $(x_k^1,x_k^2)$ generated by Eq. (\ref{eq:policy_equivalent_update}).

\subsection{The Change of Variable}
As illustrated in Section \ref{subsec:viewpoint}, we view Algorithm \ref{algorithm:reformulation} Line $4$ as an extension of the smoothed best-response dynamics implemented in the parameter space of the function approximation class. To formally make the connection, we state the smoothed best-response dynamics for solving the zero-sum matrix game with payoff matrices $X_1$ and $X_2$ in the following:
\begin{align}\label{eq:SBRD}
	\hat{\pi}_{k+1}^i=\hat{\pi}_k^i+\beta_k(\sigma_\tau(X_i\hat{\pi}_k^{-i})-\hat{\pi}_k^i),\quad i\in \{1,2\}.
\end{align}
Next, we present a change of variable to connect Eq. (\ref{eq:SBRD}) with Eq. (\ref{eq:policy_equivalent_update}), which relies on the following assumption.

\begin{assumption}\label{as:invertibility}
	The payoff matrix $X_1$ is invertible.
\end{assumption}
\begin{remark}
	Since $X_1+X_2^\top =0$, $X_1$ being invertible also implies $X_2$ being invertible. We want to emphasize that Assumption \ref{as:invertibility} is made only for the illustration of the construction of our Lyapunov function. In fact, neither Theorem \ref{thm:main} nor the final form of our Lyapunov function for studying $\{\theta_k^i\}$ requires this assumption.
\end{remark}

Under Assumption \ref{as:invertibility}, by multiplying $X_i^{-1}$ on both sides of Eq. (\ref{eq:policy_equivalent_update}) and denoting $\zeta_k^{-i}=X_i^{-1}x_k^i$ for $i\in \{1,2\}$, we can write Eq. (\ref{eq:policy_equivalent_update}) equivalently as
\begin{align}\label{eq:after_change_of_basis}
	\zeta_{k+1}^i=\zeta_k^i+\beta_k(\sigma_\tau(X_i\zeta_k^{-i})-\zeta_k^i),\quad i\in \{1,2\}.
\end{align}
Observe that Eq. (\ref{eq:after_change_of_basis}) seems to be identical to the smoothed best-response dynamics presented in Eq. (\ref{eq:SBRD}).
However, we want to highlight one major difference. Note that the iterate $\zeta_k^i$ is obtained from a change of variable. Therefore, $\zeta_k^i$ is \textit{not} necessarily a policy (or a probability distribution) as in the smoothed best-response dynamics. In view of Eq. (\ref{eq:after_change_of_basis}), one may ask that, if we initialized $\zeta_0^i$ in the probability simplex $\Delta(\mathcal{A}^i)$, we would be able to show that $\zeta_k^i\in\Delta(\mathcal{A}^i)$ for all $k\geq 0$ using an induction argument. However, this is also not possible because of the following two reasons. \textit{(i)} Recall from Eqs. (\ref{eq:policy_equivalent1}) and (\ref{eq:policy_equivalent_update}) that $X_i$ corresponds to $\mathcal{T}^i(v^i)(s)$ and $x_0^i$ corresponds to $q_{\theta_0^i}^i(s)$. Therefore, in order to ensure that $\zeta_0^{-i}=X_i^{-1}x_0^i=\mathcal{T}^i(v^i)(s)q_{\theta_0^i}^i(s)\in\Delta(\mathcal{A}^i)$, we would require access to the underlying model parameters to compute $\mathcal{T}^i(v^i)(s)$. This is not possible in payoff-based independent learning. \textit{(ii)} Recall that Eq. (\ref{eq:policy_equivalent_update}) is a simplified version of Eq. (\ref{eq:policy_equivalent1}) by ignoring the estimation error of the fast-timescale iterates $\{w_k^i\}$, which is stochastic by nature. Therefore, even if we managed to initialize $\theta_0^i$ such that $\zeta_0^i\in\Delta(\mathcal{A}^i)$, the stochastic error in $w_k^i$ may drive $\zeta_k^i$ out of the probability simplex.

In summary, by performing a change of variable, the resulting update equation (cf. Eq. (\ref{eq:after_change_of_basis})) is identical in form to the smoothed best-response dynamics. However, there are two issues:
\begin{enumerate}[(1)]
	\item \textbf{Issue 1.} the iterates are not necessarily probability distributions;
	\item \textbf{Issue 2.} the change of variable requires the invetibility of $X_1$, which may not hold in general.
\end{enumerate}
In the next two subsections, we will start with the Lyapunov function for studying the smoothed best-response dynamics, i.e., the regularized Nash gap, and work our way back to construct the Lyapunov function for studying $(x_k^1,x_k^2)$ generated by Eq. (\ref{eq:policy_equivalent_update}), while resolving the two issues stated above. 

\subsection{The Moreau Envelope for a Smooth Domain Extension}

In the existing literature, to study the smooth best-response dynamics (\ref{eq:SBRD}), the regularized Nash gap $\text{RNG}_\tau:\Delta(\mathcal{A}^1)\times \Delta(\mathcal{A}^2)\mapsto\mathbb{R}$ defined as
\begin{align}\label{eq:RNG}
	\text{RNG}_\tau(\hat{\pi}^1,\hat{\pi}^2)=\sum_{i=1,2}\max_{\mu^i\in  \Delta(\mathcal{A}^i)}\{(\mu^i-\hat{\pi}^i)^\top X_i\hat{\pi}^{-i}+\tau \nu(\mu^i)-\tau \nu(\hat{\pi}^i)\},\quad \forall\,(\hat{\pi}^1,\hat{\pi}^2),
\end{align}
has been shown to be a valid Lyapunov function \cite{hofbauer2005learning}, where $\nu(\mu^i)=-\sum_{a^i}\mu^i(a^i)\log(\mu^i(a^i))$ is the entropy function.

Naturally, we want to use $\text{RNG}_\tau(\cdot,\cdot)$ to study $(\zeta_k^1,\zeta_k^2)$ generated by Eq. (\ref{eq:after_change_of_basis}). However, since $\zeta_k^i$ is not necessarily a probability distribution, the entropy function of $\zeta_k^i$ is not even well-defined, which prevents us from directly using $\text{RNG}_\tau(\cdot,\cdot)$. Therefore, we need to extend the domain of the regularized Nash gap from the joint probability simplex $\Delta(\mathcal{A}^1)\times \Delta(\mathcal{A}^2)$ to the joint real vector space $\mathbb{R}^{|\mathcal{A}^1|}\times \mathbb{R}^{|\mathcal{A}^2|}$. 

To provide a starting point, we first present a naive way to extend the domain of the regularized Nash gap. For $i\in \{1,2\}$, let $\Bar{V}^i:\mathbb{R}^{|\mathcal{A}^1|}\times \mathbb{R}^{|\mathcal{A}^2|}\mapsto\mathbb{R}$ be defined as
\begin{align}\label{eq:middle_guy}
	\Bar{V}^i(\zeta^1,\zeta^2)=\begin{dcases}
		\max_{u^i\in  \Delta(\mathcal{A}^i)}\{(u^i-\zeta^i)^\top X_i\zeta^{-i}+\tau \nu(u^i)-\tau \nu(\zeta^i)\},&\zeta^i\in \Delta(\mathcal{A}^i),\\
		+\,\infty,&\text{Otherwise},
	\end{dcases}
\end{align}
and define $\Bar{V}:\mathbb{R}^{|\mathcal{A}^1|}\times \mathbb{R}^{|\mathcal{A}^2|}\mapsto\mathbb{R}$ as $\Bar{V}=\Bar{V}^1+\Bar{V}^2$.
Note that the function value of $\Bar{V}$ was directly assigned to infinity if the variable $(\zeta^1,\zeta^2)$ is not in the joint probability simplex $\Delta(\mathcal{A}^1)\times \Delta(\mathcal{A}^2)$.
However, the function $\Bar{V}(\cdot,\cdot)$ fails as a Lyapunov function due to its non-smoothness. What we truly need is to smoothly extend the domain of the regularized Nash gap. To achieve that, we use the Moreau envelope as a means for the construction, which is defined in the following.

\begin{definition}[The Moreau Envelope]\label{def:Moreau}
	Let $h:\mathbb{R}^d\mapsto \mathbb{R}$ be a closed and convex function and $c>0$. The Moreau envelope of $h(\cdot)$ is defined to be the function $M_{h}^{c}(x)=\inf_{u\in\mathbb{R}^d}\{h(u)+\frac{1}{2c}\|x-u\|_2^2\}$.
\end{definition}

\begin{remark}
	The Moreau envelope has many nice properties, one of which is the smoothness \cite{beck2017first}. It was previously used in \cite{guzman2015lower,beck2012smoothing} to study convex but non-smooth optimization problems, and used in \cite{chen2021lyapunov,kalogiannis2022efficiently} to construct smooth Lyapunov functions.
\end{remark}

Next, we use the Moreau envelope to smoothly extend the domain of the regularized Nash gap. We start with the function $\Bar{V}^i(\cdot,\cdot)$ defined in Eq. (\ref{eq:middle_guy}). For any $i\in \{1,2\}$ and $\mu>0$, let $\hat{V}^i:\mathbb{R}^{|\mathcal{A}^1|}\times \mathbb{R}^{|\mathcal{A}^2|}\mapsto\mathbb{R}$ be defined as
\begin{align}
	\hat{V}^i(\zeta^1,\zeta^2)=\;&\inf_{\Bar{u}^i\in\mathbb{R}^{|\mathcal{A}^i|}}\left\{\Bar{V}^i(\Bar{u}^i,\zeta^{-i})+\frac{1}{2\mu}\|\zeta^i-\Bar{u}^i\|_2^2\right\}\nonumber\\
	=\;&\min_{\Bar{u}^i\in\Delta(\mathcal{A}^i)}\left\{\max_{u^i\in  \Delta(\mathcal{A}^i)}\{(u^i-\Bar{u}^i)^\top X_i\zeta^{-i}+\tau \nu(u^i)-\tau \nu(\Bar{u}^i)\}+\frac{1}{2\mu}\|\zeta^i-\Bar{u}^i\|_2^2\right\}\label{eq:middle_guy2},
\end{align}
where Eq. (\ref{eq:middle_guy2}) follows from the definition of $\Bar{V}^i(\cdot,\cdot)$ and the Weierstrass extreme value theorem. Let $\hat{V}:\mathbb{R}^{|\mathcal{A}^1|}\times \mathbb{R}^{|\mathcal{A}^2|}\mapsto\mathbb{R}$ be defined as
\begin{align}\label{eq:middle_guy3}
	\hat{V}(\zeta^1,\zeta^2)=\;&\sum_{i=1,2}\hat{V}^i(\zeta^1,\zeta^2)\nonumber\\
	=\;&\sum_{i=1,2}\min_{\Bar{u}^i\in\Delta(\mathcal{A}^i)}\left\{\max_{u^i\in  \Delta(\mathcal{A}^i)}\{(u^i-\Bar{u}^i)^\top X_i\zeta^{-i}+\tau \nu(u^i)-\tau \nu(\Bar{u}^i)\}+\frac{1}{2\mu}\|\zeta^i-\Bar{u}^i\|_2^2\right\},
\end{align}
which successfully extends the domain of the regularized Nash gap from the joint probability simplex $\Delta(\mathcal{A}^1)\times \Delta(\mathcal{A}^2)$ to the joint real vector space $\mathbb{R}^{|\mathcal{A}^1|}\times \mathbb{R}^{|\mathcal{A}^2|}$. In fact, one can show (following the same line of analysis as in Appendix \ref{ap:pi_analysis}) that $\hat{V}(\cdot,\cdot)$ is a valid Lyapunov function for $(\zeta_k^1,\zeta_k^2)$ from Eq. (\ref{eq:after_change_of_basis}).

\subsection{Overcoming the Invertibility Issue by Generalizing the Moreau Envelope}

Now that we have constructed a Lyapunov function $\hat{V}(\zeta_1,\zeta^2)$ for $(\zeta_k^1,\zeta_k^2)$ generated by Eq. (\ref{eq:after_change_of_basis}), it remains to perform the inverse change of variable to finish constructing the Lyapunov function to study $(x_k^1,x_k^2)$ generated by Eq. (\ref{eq:policy_equivalent_update}). Since $\zeta_k^{-i}=X_i^{-1}x_k^i$ for $i\in \{1,2\}$, let $\Tilde{V}:\mathbb{R}^{|\mathcal{A}^1|}\times \mathbb{R}^{|\mathcal{A}^2|}\mapsto\mathbb{R}$ be defined as
\begin{align}
	\Tilde{V}(x^1,x^2)=\;&\hat{V}(X_2^{-1}x_2,X_1^{-1}x_1)\nonumber\\
	=\;&\sum_{i=1,2}\min_{\Bar{u}^i\in\Delta(\mathcal{A}^i)}\max_{u^i\in  \Delta(\mathcal{A}^i)}\left\{(u^i-\Bar{u}^i)^\top x^i+\tau \nu(u^i)-\tau \nu(\Bar{u}^i)+\frac{1}{2\mu}\|X_{-i}^{-1}x^{-i}-\Bar{u}^i\|_2^2\right\},\label{eq:cv}
\end{align}
which would serve as a valid Lyapunov function to study $(x_k^1,x_k^2)$ from Eq. (\ref{eq:policy_equivalent_update}) if Assumption \ref{as:invertibility} were satisfied, that is, $X_1$ (or equivalently $X_2$) were indeed invertible. However, this may not be true in general. For example, when players $1$ and $2$ have different numbers of actions, the matrix $X_1$ (or $X_2$) is not even a square matrix, hence can not be invertible. Therefore, the last challenge we need to overcome is to remove the invertibility assumption.

Note that $\{X_i^{-1}\}_{i\in \{1,2\}}$ only appear inside the norm square on the RHS of Eq. (\ref{eq:cv}), which arises due to using the Moreau envelope as a means for smooth approximation. Tracing back to Definition \ref{def:Moreau}, the Moreau envelope is nothing but an infimal convolution between the function $h(\cdot)$ and the $\ell_2$-norm square function. We next generalize the definition of the Moreau envelope by replacing $\|\cdot\|_2$ with a properly defined seminorm, the advantage of which is that, after performing the inverse change of variable as in Eq. (\ref{eq:cv}), the inverse of $X_i$ vanishes as if we never needed the invertibility of $X_i$.

Given a matrix $W\in \mathbb{R}^{d_2\times d_1}$, let $\|\cdot\|_W$ be a function in $\mathbb{R}^{d_1}$ defined as $\|x\|_W=\sqrt{x^\top W^\top Wx}$. Note that, unless $W$ has independent columns, $\|\cdot\|_W$ is a seminorm instead of a norm.
\begin{definition}[The Generalized Moreau Envelope]\label{def:new_Moreau}
	Let $h:\mathbb{R}^{d_1}\mapsto \mathbb{R}$ be a closed and convex function, $W\in\mathbb{R}^{d_2\times d_1}$ be a matrix, and $c>0$. The generalized Moreau envelope of $h(\cdot)$ is defined to be the function $M_{h}^{W,c}(x)=\inf_{u\in\mathbb{R}^d}\{h(u)+\frac{1}{2c}\|x-u\|_W^2\}$.
\end{definition}

Now, starting from Eq. (\ref{eq:middle_guy2}), instead of using the standard Moreau envelope (cf. Definition \ref{def:Moreau}) as a means for smoothly extending the domain of the regularized Nash gap, we use the generalized Moreau envelope with $W=X_{-i}$ for defining $\hat{V}^i(\cdot,\cdot)$. In this case, after performing the inverse change of variable, we arrive at the final Lyapunov function for studying $(x_k^1,x_k^2)$ generated by Eq. (\ref{eq:policy_equivalent_update}):
\begin{align}
	V(x^1,x^2)=\;&\sum_{i=1,2}\min_{\Bar{u}^i\in\Delta(\mathcal{A}^i)}\max_{u^i\in  \Delta(\mathcal{A}^i)}\left\{(u^i-\Bar{u}^i)^\top x^i+\tau \nu(u^i)-\tau \nu(\Bar{u}^i)+\frac{1}{2\mu}\|X_{-i}^{-1}x^{-i}-\Bar{u}^i\|_{X_{-i}}^2\right\}\nonumber\\
	=\;&\sum_{i=1,2}\min_{\Bar{u}^i\in\Delta(\mathcal{A}^i)}\max_{u^i\in  \Delta(\mathcal{A}^i)}\left\{(u^i-\Bar{u}^i)^\top x^i+\tau \nu(u^i)-\tau \nu(\Bar{u}^i)+\frac{1}{2\mu}\|x^{-i}-X_{-i}\Bar{u}^i\|_2^2\right\}.\label{eq:middle_last}
\end{align}
Importantly, while we used the invertibility of $X_1$ to present the roadmap in constructing our Lyapunov function, since $X_1^{-1}$ (or $X_2^{-1}$) does not appear in the definition of $V(\cdot,\cdot)$; see Eq. (\ref{eq:middle_last}), it is as if we never needed such invertibility.

It is easy to see that $V(x^1,x^2)$ is non-negative by definition. In addition, we have $V(x^1,x^2)=0$ if and only if $x^i=X_i\sigma_\tau(x^{-i})$ for $i\in \{1,2\}$, as desired. To see this, observe that when $x^i=X_i\sigma_\tau(x^{-i})$ for $i\in \{1,2\}$, $\sigma_\tau(x^i)$ is simultaneously the optimal solution of both ${\max}_{u^i\in  \Delta(\mathcal{A}^i)}\left\{(u^i)^\top x^i+\tau \nu(u^i)\right\}$ and ${\min}_{\bar{u}^i\in\Delta(\mathcal{A}^i)}\{-(\bar{u}^i)^\top x^i-\tau \nu(\bar{u}^i)+\frac{1}{2\mu}\|x^{-i}-X_{-i}\bar{u}^i\|_2^2\}$ for $i\in \{1,2\}$. The following proposition shows that $(x_k^1,x_k^2)$ updated according to Eq. (\ref{eq:policy_equivalent_update}) produces a negative drift with respect to $V(\cdot,\cdot)$.

\begin{proposition}[Proof in Appendix \ref{ap:pf}]\label{prop:drift1}
	Consider $(x_k^1,x_k^2)$ generated by Eq. (\ref{eq:policy_equivalent_update}). By choosing $\mu = \tau/64$, we have for all $k\geq 0$ that
	\begin{align}\label{eq:prop:drift1}
		V(x_{k+1}^1,x_{k+1}^2)\leq\;&\left(1-\frac{\beta_k}{2}\right)V(x_k^1,x_k^2)+\frac{138\max_{i\in \{1,2\}}(\|x_0^i\|_2+2\|X_i\|_{1,2})}{\tau}\beta_k^2,
	\end{align}
 where $\|X_i\|_{1,2}=\max_{\|x\|_1=1}\|X_ix\|_2$ is an induced matrix norm.
\end{proposition}
The first term on the RHS of Eq. (\ref{eq:prop:drift1}) represents the negative drift of $(x_k^1,x_k^2)$ with respect to the Lyapunov function $V(\cdot,\cdot)$, and the second term is a higher order error term, which is due to the fact that Eq. (\ref{eq:policy_equivalent_update}) is a discrete update equation. To this end, we have successfully constructed a Lyapunov function to study $(x_k^1,x_k^2)$ generated by Eq. (\ref{eq:policy_equivalent_update}), and therefore the slow-timescale iterates $\{\theta_{t,k}^i\}$ of Algorithm \ref{algorithm:reformulation}.

\section{Conclusion}\label{sec:conclusion}
In this work, we consider two-player zero-sum stochastic games and develop a two-timescale $Q$-learning that is payoff-based and uses function approximation. In the special case of the linear function approximation, our proposed learning dynamics provably enjoy last-iterate finite-sample guarantees. To establish the result, we use a Lyapunov approach. The key novelty lies in the construction of a valid Lyapunov function for the slow-timescale iterates, which involves a change of variable and the use of a generalized variant of the Moreau envelope of the regularized Nash gap.

We recognize $3$ future directions of this work. Firstly, we would like to remove the projection operator $\pj_M^i(\cdot)$ in Algorithm \ref{algorithm:FA} as the projection radius depends on unknown parameters of the stochastic game, which is not ideal from a practical viewpoint. Secondly, while this is the first result that establishes last-iterate finite-sample bounds for payoff-based independent learning dynamics under function approximation, we believe the convergence rate is not tight, and improving the rate is another future direction. Finally, we would like to see if the algorithmic ideas and technical tools developed in this work, especially the Lyapunov function construction process illustrated in Section \ref{sec:analysis},  can be applied to other MARL settings.

\bibliographystyle{apalike}
\bibliography{references}

\newpage

\begin{center}
    {\LARGE\bfseries Appendices}
\end{center}
\appendix
\section{Proof of Theorem \ref{thm:main}}\label{ap:pf}

We begin by rewriting Algorithm \ref{algorithm:reformulation_ap} in the following when using linear function approximation, where we denote $\tilde{q}_{t,k}^i=\Phi^i\theta_{t,k}^i$ and $q_{t,k}^i=\Phi^iw_{t,k}^i$ for all $t,k$. Other notation that will be used in the proof is summarized in Appendix \ref{ap:notation}. Our proof is divided into $4$ steps. In Appendix \ref{ap:Nash_Gap}, we bound the Nash gap by the Lyapunov functions. In Appendix \ref{ap:v_analysis}, we analyze the outer loop, which is an approximation of the minimax value iteration. In Appendix \ref{ap:inner_loop}, we analyze the inner loop, where the slow-timescale iterates $\{\theta_{t,k}^i\}$ are studied in Appendix \ref{ap:pi_analysis} using the Lyapunov function constructed in Section \ref{sec:analysis}, and the fast-timescale iterates $\{w_{t,k}^i\}$ are studied in Appendix \ref{ap:q_analysis}. Finally, in Appendix \ref{ap:recursion}, we establish the finite-sample bounds. The proof of all supporting lemmas is provided in Appendix \ref{ap:supporting_lemma}.

\begin{algorithm}[ht]
	\caption{Algorithm \ref{algorithm:FA} Under Linear Function Approximation}
	\label{algorithm:reformulation_ap} 
	\begin{algorithmic}[1]
		\STATE \textbf{Input:} Integers $K$ and $T$, initializations $\Tilde{q}_{0,0}^i=q_{0,0}^i=0$, $v_0^i=0$, and $\pi_{0,0}^i(s)=\text{Unif}(\mathcal{A}^i)$ for all $s\in\mathcal{S}$
		\FOR{$t=0,1,\cdots,T-1$}
		\FOR{$k=0,1,\cdots,K-1$}
		\STATE $\Tilde{q}^i_{t,k+1}=\Tilde{q}_{t,k}^i+\beta_k (q_{t,k}^i-\Tilde{q}_{t,k}^i)$ and compute $\pi_{t,k+1}^i(S_k)=\sigma_\tau(\Tilde{q}_{t,k+1}^i(S_k))$
		\STATE Play
		$A_k^i\sim \pi_{k+1}^i(\cdot\mid S_k)$ and observe $S_{k+1}\sim p(\cdot\mid S_k,A_k^i,A_k^{-i})$
		\STATE $w_{t,k+1}^i=\pj_M^i\left[w_{t,k}^i+\alpha_k\phi^i(S_k,A_k^i)\left(R_i(S_k,A_k^i,A_k^{-i})+\gamma v_t^i(S_{k+1})-\phi^i(S_k,A_k^i)^\top w_{t,k}^i\right)\right]$ 
		\ENDFOR
		\STATE $v_{t+1}^i(s)= \pi_{t,K}^i(s)^\top \ceil{q_{t,K}^i(s)}$ for all $s\in\mathcal{S}$
        \STATE $S_0=S_K$, $q_{t+1,0}^i=q_{t,K}^i$, and $\Tilde{q}_{t+1,0}^i=\Tilde{q}_{t,K}^i$
		\ENDFOR
	\end{algorithmic}
\end{algorithm}

\subsection{Notation}\label{ap:notation}
We begin by summarizing the notation we are going to use here.
\begin{enumerate}[(1)]
    \item Given non-negative integers $k_1\leq k_2$, we denote $\alpha_{k_1,k_2}=\sum_{k=k_1}^{k_2}\alpha_k$ and $\beta_{k_1,k_2}=\sum_{k=k_1}^{k_2}\beta_k$.
	\item Given a pair of matrices $\{X^i\in\mathbb{R}^{|\mathcal{A}^i|\times |\mathcal{A}^{-i}|}\}_{i\in \{1,2\}}$, we define a function  $V_X:\mathbb{R}^{|\mathcal{A}^1|}\times \mathbb{R}^{|\mathcal{A}^2|}\mapsto\mathbb{{R}}$ as 
	\begin{align*}
		V_X(x^1,x^2)=\sum_{i=1,2}\max_{u^i\in  \Delta(\mathcal{A}^i)}\min_{\bar{u}^i\in\Delta(\mathcal{A}^i)}\left\{(u^i-\Bar{u}^i)^\top x^i+\tau \nu(u^i)-\tau \nu(\bar{u}^i)+\frac{1}{2\mu}\|x^{-i}-X_{-i}\bar{u}^i\|_2^2\right\}
	\end{align*}
 for all $x^1\in \mathbb{R}^{|\mathcal{A}^1|}$ and $x^2\in \mathbb{R}^{|\mathcal{A}^2|}$,
	where $\nu(\cdot)$ is the entropy function defined as $\nu(x^i)=-\sum_{a^i\in\mathcal{A}^i}x^i(a^i) \log(x^i(a^i))$ for $i\in \{1,2\}$. The function $V_X(\cdot,\cdot)$ serves as our Lyapunov function to analyze the evolution of the policy parameter $\{\theta_{t,k}^i\}$. A sequence of properties related to $V_X(\cdot,\cdot)$ is established in Appendix \ref{ap:lemma_policy}.
	\item Given a pair of $v$-functions $(v^1,v^2)$ and a state $s\in\mathcal{S}$, we define $V_{v,s}:\mathbb{R}^{|\mathcal{A}^1|}\times \mathbb{R}^{|\mathcal{A}^2|}\mapsto\mathbb{{R}}$ as $V_{v,s}(x^1,x^2)=V_X(x^1,x^2)$ with $X_i=\mathcal{T}^i(v^i)(s)$ for $i\in\{1,2\}$.
	\item For any joint policy $(\pi^1,\pi^2)$ and state $s\in\mathcal{S}$, given $i\in \{1,2\}$, let  $v^i_{*,\pi^{-i}}(s)=\max_{\hat{\pi}^i}v^i_{\hat{\pi}^i,\pi^{-i}}(s)$, $v^i_{\pi^i,*}=\min_{\hat{\pi}^{-i}}v^i_{\pi^i,\hat{\pi}^{-i}}(s)$, $v^{-i}_{\pi^{-i},*}(s)=\min_{\hat{\pi}^i}v^{-i}_{\pi^{-i},\hat{\pi}^i}(s)$, and $v^{-i}_{*,\pi^i}(s)=\max_{\hat{\pi}^{-i}}v^{-i}_{\hat{\pi}^{-i},\hat{\pi}^i}(s)$. Note that we have $v^i_{*,\pi^{-i}}+v^{-i}_{\pi^{-i},*}=0$ and $v^i_{\pi^i,*}+v^{-i}_{*,\pi^i}=0$ for $i\in \{1,2\}$.
	\item For any $t,k$ and $i\in \{1,2\}$, we define $q_{t,k}^i,\Tilde{q}_{t,k}^i$, and $\bar{q}_{t,k}^i\in\mathbb{R}^{|\mathcal{S}||\mathcal{A}^i|}$ as
 \begin{align*}
     q_{t,k}^i=\Phi^iw_{t,k}^i,\quad \Tilde{q}_{t,k}^i=\Phi^i\theta_{t,k}^i,\quad \Bar{q}_{t,k}^i(s)=\mathcal{T}^i(v^i)(s)\pi_{t,k}^{-i}(s),\quad \forall\,s\in\mathcal{S}. 
 \end{align*} 
    Moreover, we let $\bar{w}_{t,k}^i\in\mathbb{R}^{d_i}$ be such that $\bar{q}_{t,k}^i=\Phi^i\bar{w}_{t,k}^i$. Note that $\bar{w}_{t,k}^i$ exists and is unique under Assumption \ref{as:Bellman_Completeness}.
 \item For any $t,k\geq 0$, define $\mathcal{L}_v(t)=\sum_{i=1,2}\|v_t^i-v_*^i\|_\infty$, $\mathcal{L}_{\text{sum}}(t)=\|v_t^1+v_t^2\|_\infty$, $\mathcal{L}_\theta(t,k)=\max_{s\in\mathcal{S}}V_{v_t,s}(\Tilde{q}_{t,k}^1(s),\Tilde{q}_{t,k}^2(s))$, and $\mathcal{L}_w(t,k)=\sum_{i=1,2}\|w_{t,k}^i-\Bar{w}_{t,k}^i\|_2^2$.
\end{enumerate}

\subsection{Bounding the Nash Gap}\label{ap:Nash_Gap}

\begin{lemma}\label{prop:NashGap_to_vt}
	It holds for any joint policy $\pi=(\pi^1,\pi^2)$ that
	\begin{align*}
		\Gap{\pi^1,\pi^2}\leq 2\sum_{i=1,2}\left\|v^i_{\pi^i,*}-v^i_*\right\|_\infty.
	\end{align*}
\end{lemma}
\begin{proof}[Proof of Lemma \ref{prop:NashGap_to_vt}]
	By definition of the expected value function, given 
$i\in \{1,2\}$, we have 
	\begin{align*}
		\max_{\hat{\pi}^i}U^i(\hat{\pi}^i,\pi^{-i})=\;&\max_{\hat{\pi}^i} \mathbb{E}_{S\sim \rho_o(\cdot)}\left[v_{\hat{\pi}^i,\pi^{-i}}^i(S)\right]\\
  \leq \;& \mathbb{E}_{S\sim \rho_o(\cdot)}\left[\max_{\hat{\pi}^i}v_{\hat{\pi}^i,\pi^{-i}}^i(S)\right]\tag{Jensen's inequality}\\
		=\;&  \mathbb{E}_{S\sim \rho_o(\cdot)}\left[v_{*,\pi^{-i}}^i(S)\right],
	\end{align*}
	where the last line follows from $v_{*,\pi^{-i}}^i(s)\geq v_{\hat{\pi}^i,\pi^{-i}}^i(s)$ for any policy $\hat{\pi}^i$ and $s\in\mathcal{S}$.
	Therefore, we have 
	\begin{align}\label{eq1:prop:Gap_to_vt}
		\Gap{\pi^1,\pi^2}=\;&\sum_{i=1,2}\left[\max_{\hat{\pi}^i}U^i(\hat{\pi}^i,\pi^{-i})-U^i(\pi^i,\pi^{-i})\right]\nonumber\\
		\leq \;& \sum_{i=1,2}\mathbb{E}_{S\sim p_o}\left[v^i_{*,\pi^{-i}}(S)-v^i_{\pi^i,\pi^{-i}}(S)\right]\nonumber\\
		\leq\;&  \sum_{i=1,2}\left\|v^i_{*,\pi^{-i}}-v^i_{\pi^i,\pi^{-i}}\right\|_\infty.
	\end{align}
	To proceed, observe that for any $i\in \{1,2\}$ and $s\in\mathcal{S}$, we have
	\begin{align*}
		\left|v^i_{*,\pi^{-i}}(s)-v^i_{\pi^i,\pi^{-i}}(s)\right|
		=\;&v^i_{*,\pi^{-i}}(s)-v^i_{\pi^i,\pi^{-i}}(s)\tag{$v^i_{*,\pi^{-i}}(s)\geq v^i_{\pi^i,\pi^{-i}}(s)$ for all $s$}\\
		\leq  \;&v^i_{*,\pi^{-i}}(s)-v^i_{\pi^i,*}(s)\tag{$v^i_{\pi^i,\pi^{-i}}(s)\geq v^i_{\pi^i,*}(s)$ for all $s$}\\
		=  \;&-v^{-i}_{\pi^{-i},*}(s)-v^i_{\pi^i,*}(s)\tag{$v^i_{*,\pi^{-i}}+v^{-i}_{\pi^{-i},*}=0$}\\
		=\;&v_*^{-i}(s)-v^{-i}_{\pi^{-i},*}(s)+v_*^i(s)-v^i_{\pi^i,*}(s)\tag{$v_*^i+v_*^{-i}=0$}\\
		\leq \;&\sum_{i=1,2}\left\|v^i_{\pi^i,*}-v^i_*\right\|_\infty.
	\end{align*}
	Since the RHS of the previous inequality does not depend on $s$, we have for $i\in \{1,2\}$ that
	\begin{align*}
		\left\|v^i_{*,\pi^{-i}}-v^i_{\pi^i,\pi^{-i}}\right\|_\infty\leq \sum_{i=1,2}\left\|v^i_{\pi^i,*}-v^i_*\right\|_\infty.
	\end{align*}
	The result follows from using the previous inequality in Eq. (\ref{eq1:prop:Gap_to_vt}). 
\end{proof}

\begin{lemma}\label{le:vpi_to_vt}
	For any $\tilde{q}^1\in\mathbb{R}^{|\mathcal{S}||\mathcal{A}^1|}$ and $\tilde{q}^2\in\mathbb{R}^{|\mathcal{S}||\mathcal{A}^2|}$, let $\pi^i(s)=\sigma_\tau(\tilde{q}^i(s))$ for all $s\in\mathcal{S}$ and $i\in \{1,2\}$. Then, we have for any $v^1,v^2\in\mathbb{R}^{|\mathcal{S}|}$ that
	\begin{align*}
		\sum_{i=1,2}\|v^i_{\pi^i,*}-v^i_*\|_\infty\leq\;& \frac{2}{1-\gamma}\left(\sum_{i=1,2}\|v^i-v^i_{*}\|_\infty+2\|v^1+v^2\|_\infty\right.\\
		&+\left.\frac{9 A_{\max}^2}{\tau^2(1-\gamma)^2}\max_{s\in\mathcal{S}}V_{v,s}(\tilde{q}^1(s),\tilde{q}^2(s))+6\tau \log(A_{\max})+\mu\right).
	\end{align*}
\end{lemma}
\begin{proof}[Proof of Lemma \ref{le:vpi_to_vt}]	
	Since $v^{-i}_*-v^{-i}_{\pi^{-i},*}=v^i_{*,\pi^{-i}}-v^i_*$ for $i\in \{1,2\}$, it is enough to bound $\|v^i_{*,\pi^{-i}}-v^i_*\|_\infty$. Given $i\in \{1,2\}$, for any $s\in\mathcal{S}$, we have
	\begin{align*}
		0\leq\;&v^i_{*,\pi^{-i}}(s)-v^i_*(s)\\
		=\;&\max_{\mu^i\in \Delta(\mathcal{A}^i)}(\mu^i)^\top  \mathcal{T}^i(v^i_{*,\pi^{-i}})(s)\pi^{-i}(s)-\max_{\mu^i\in \Delta(\mathcal{A}^i)}\min_{\mu^{-i}\in\Delta(\mathcal{A}^{-i})}(\mu^i)^\top \mathcal{T}^i(v_*^i)(s)\mu^{-i}\tag{Bellman optimality equations}\\
		\leq \;&\left|\max_{\mu^i\in \Delta(\mathcal{A}^i)}(\mu^i)^\top  \mathcal{T}^i(v^i_{*,\pi^{-i}})(s)\pi^{-i}(s)-\max_{\mu^i\in \Delta(\mathcal{A}^i)}(\mu^i)^\top  \mathcal{T}^i(v^i_{*})(s)\pi^{-i}(s)\right|\\
		&+\left|\max_{\mu^i\in \Delta(\mathcal{A}^i)}(\mu^i)^\top  \mathcal{T}^i(v^i_{*})(s)\pi^{-i}(s)-\max_{\mu^i\in \Delta(\mathcal{A}^i)}(\mu^i)^\top  \mathcal{T}^i(v^i)(s)\pi^{-i}(s)\right|\\
		&+\max_{\mu^i\in \Delta(\mathcal{A}^i)}(\mu^i)^\top  \mathcal{T}^i(v^i)(s)\pi^{-i}(s)-\max_{\mu^i\in \Delta(\mathcal{A}^i)}\min_{\mu^{-i}\in\Delta(\mathcal{A}^{-i})}(\mu^i)^\top  \mathcal{T}^i(v^i)(s)\mu^{-i}\\
		&+\left|\max_{\mu^i\in \Delta(\mathcal{A}^i)}\min_{\mu^{-i}\in\Delta(\mathcal{A}^{-i})}(\mu^i)^\top  \mathcal{T}^i(v^i)(s)\mu^{-i}-\max_{\mu^i\in \Delta(\mathcal{A}^i)}\min_{\mu^{-i}\in\Delta(\mathcal{A}^{-i})}(\mu^i)^\top \mathcal{T}^i(v_*^i)(s)\mu^{-i}\right|\tag{Triangle inequality}\\
		\leq \;&\max_{\mu^i\in \Delta(\mathcal{A}^i)}\left|(\mu^i)^\top  (\mathcal{T}^i(v^i_{*,\pi^{-i}})(s)-\mathcal{T}^i(v^i_{*})(s))\pi^{-i}(s)\right|\\
		&+\max_{\mu^i\in \Delta(\mathcal{A}^i)}\left|(\mu^i)^\top  (\mathcal{T}^i(v^i_{*})(s)-\mathcal{T}^i(v^i)(s))\pi^{-i}(s)\right|\\
		&+\max_{\mu^i\in \Delta(\mathcal{A}^i)}(\mu^i)^\top  \mathcal{T}^i(v^i)(s)\pi^{-i}(s)-\max_{\mu^i\in \Delta(\mathcal{A}^i)}\min_{\mu^{-i}\in\Delta(\mathcal{A}^{-i})}(\mu^i)^\top  \mathcal{T}^i(v^i)(s)\mu^{-i}\\
		&+\max_{\mu^i\in \Delta(\mathcal{A}^i)}\max_{\mu^{-i}\in\Delta(\mathcal{A}^{-i})}\left|(\mu^i)^\top  (\mathcal{T}^i(v^i)(s)-\mathcal{T}^i(v_*^i)(s))\mu^{-i}\right|.
	\end{align*}
	Using Lemma \ref{le:use_for_v} (2), we have from the previous inequality that
	\begin{align}
		0\leq \;&v^i_{*,\pi^{-i}}(s)-v^i_*(s)\nonumber\\
		\leq \;& \gamma \|v^i_{*,\pi^{-i}}-v^i_{*}\|_\infty+2\gamma \|v^i-v^i_{*}\|_\infty\nonumber\\
		&+\max_{\mu^i\in \Delta(\mathcal{A}^i)}(\mu^i)^\top  \mathcal{T}^i(v^i)(s)\pi^{-i}(s)-\max_{\mu^i\in \Delta(\mathcal{A}^i)}\min_{\mu^{-i}\in\Delta(\mathcal{A}^{-i})}(\mu^i)^\top  \mathcal{T}^i(v^i)(s)\mu^{-i}.\label{eq:vpi_to_v1}
	\end{align}
	It remains to bound the last term on the RHS of Eq. (\ref{eq:vpi_to_v1}). Observe that
	\begin{align}
		&\max_{\mu^i\in \Delta(\mathcal{A}^i)}(\mu^i)^\top  \mathcal{T}^i(v^i)(s)\pi^{-i}(s)-\max_{\mu^i\in \Delta(\mathcal{A}^i)}\min_{\mu^{-i}\in\Delta(\mathcal{A}^{-i})}(\mu^i)^\top  \mathcal{T}^i(v^i)(s)\mu^{-i}\nonumber\\
		\leq \;&\max_{\mu^i\in \Delta(\mathcal{A}^i)}(\mu^i)^\top  \mathcal{T}^i(v^i)(s)\pi^{-i}(s)-\min_{\mu^{-i}\in\Delta(\mathcal{A}^{-i})}(\pi^i(s))^\top  \mathcal{T}^i(v^i)(s)\mu^{-i}\nonumber\\
		= \;&\max_{\mu^i\in \Delta(\mathcal{A}^i)}(\mu^i)^\top  \mathcal{T}^i(v^i)(s)\pi^{-i}(s)+\max_{\mu^{-i}\in\Delta(\mathcal{A}^{-i})}(\mu^{-i})^\top  \mathcal{T}^{-i}(v^{-i})(s)\pi^i(s)\nonumber\\
		&-\max_{\mu^{-i}\in\Delta(\mathcal{A}^{-i})}(\mu^{-i})^\top \mathcal{T}^{-i}(v^{-i})(s)\pi^i(s)-\min_{\mu^{-i}\in\Delta(\mathcal{A}^{-i})}(\pi^i(s))^\top  \mathcal{T}^i(v^i)(s)\mu^{-i}\nonumber\\
		= \;&\sum_{i=1,2}\max_{\mu^i\in \Delta(\mathcal{A}^i)}(\mu^i)^\top  \mathcal{T}^i(v^i)(s)\pi^{-i}(s)+\min_{\mu^{-i}\in\Delta(\mathcal{A}^{-i})}(\pi^i(s))^\top \mathcal{T}^i(-v^{-i})(s)\mu^{-i}\nonumber\\
  &-\min_{\mu^{-i}\in\Delta(\mathcal{A}^{-i})}(\pi^i(s))^\top  \mathcal{T}^i(v^i)(s)\mu^{-i}\nonumber\\
		\leq  \;&\sum_{i=1,2}\max_{\mu^i\in \Delta(\mathcal{A}^i)}(\mu^i)^\top  \mathcal{T}^i(v^i)(s)\pi^{-i}(s)+\max_{\mu^{-i}\in\Delta(\mathcal{A}^{-i})}\left|(\pi^i(s))^\top (\mathcal{T}^i(-v^{-i})(s)-\mathcal{T}^i(v^i)(s))\mu^{-i}\right|\nonumber\\
		\leq \;&\sum_{i=1,2}\max_{\mu^i\in \Delta(\mathcal{A}^i)}(\mu^i)^\top  \mathcal{T}^i(v^i)(s)\pi^{-i}(s)+\gamma \|v^1+v^2\|_\infty\tag{Lemma \ref{le:use_for_v} (2)}
  \nonumber\\
		\leq \;&\frac{9 A_{\max}^2}{\tau^2(1-\gamma)^2}V_{v,s}(\tilde{q}^1(s),\tilde{q}^2(s))+2 \|v^1+v^2\|_\infty+6\tau \log(A_{\max})+\mu\label{eq3:le:vpi_to_vt}
	\end{align}
	where the last line follows from Lemma \ref{le:use_for_v} (3). Using the previous inequality in Eq. (\ref{eq:vpi_to_v1}), we have 
	\begin{align*}
		0\leq v^i_{*,\pi^{-i}}(s)-v^i_*(s)\leq\;& \gamma \|v^i_{*,\pi^{-i}}-v^i_{*}\|_\infty+2 \|v^i-v^i_{*}\|_\infty+2\|v^1+v^2\|_\infty\\
		&+\frac{9 A_{\max}^2}{\tau^2(1-\gamma)^2}\max_{s\in\mathcal{S}}V_{v,s}(\tilde{q}^1(s),\tilde{q}^2(s))+6\tau \log(A_{\max})+\mu.
	\end{align*}
	Since the RHS of the previous inequality does not depend on $s$, we have 
	\begin{align*}
		\|v^i_{*,\pi^{-i}}-v^i_*\|_\infty\leq\;&\gamma \|v^i_{*,\pi^{-i}}-v^i_{*}\|_\infty+2 \|v^i-v^i_{*}\|_\infty+2\|v^1+v^2\|_\infty\\
		&+\frac{9 A_{\max}^2}{\tau^2(1-\gamma)^2}\max_{s\in\mathcal{S}}V_{v,s}(\tilde{q}^1(s),\tilde{q}^2(s))+6\tau \log(A_{\max})+\mu,
	\end{align*}
	which by rearranging terms implies
	\begin{align*}
		\|v^i_{*,\pi^{-i}}-v^i_*\|_\infty\leq\;& \frac{1}{1-\gamma}\left(2\|v^i-v^i_{*}\|_\infty+2\|v^1+v^2\|_\infty\right.\\
		&+\left.\frac{9 A_{\max}^2}{\tau^2(1-\gamma)^2}\max_{s\in\mathcal{S}}V_{v,s}(\tilde{q}^1(s),\tilde{q}^2(s))+6\tau \log(A_{\max})+\mu\right).
	\end{align*}
 Summing up the previous inequality for $i\in \{1,2\}$, we obtain
 \begin{align*}
		\sum_{i=1,2}\|v^i_{*,\pi^{-i}}-v^i_*\|_\infty\leq\;& \frac{2}{1-\gamma}\left(\sum_{i=1,2}\|v^i-v^i_{*}\|_\infty+2\|v^1+v^2\|_\infty\right.\\
		&+\left.\frac{9 A_{\max}^2}{\tau^2(1-\gamma)^2}\max_{s\in\mathcal{S}}V_{v,s}(\tilde{q}^1(s),\tilde{q}^2(s))+6\tau \log(A_{\max})+\mu\right).
	\end{align*}
\end{proof}

Combing Lemmas \ref{prop:NashGap_to_vt} and \ref{le:vpi_to_vt} and using our definition of the Lyapunov functions $\mathcal{L}_v(\cdot)$, $\mathcal{L}_{\text{sum}}(\cdot)$, and $\mathcal{L}_\theta(\cdot,\cdot)$ to simplify the notation, we have the following result.

\begin{lemma}\label{le:Bound_Nash_Gap}
It holds that
    \begin{align*}
        \Gap{\pi_{T,K}^1,\pi_{T,K}^2}\leq \frac{4}{1-\gamma}\left(\mathcal{L}_v(T)+2\mathcal{L}_{\text{sum}}(T)+\frac{9 A_{\max}^2}{\tau^2(1-\gamma)^2}\mathcal{L}_\theta(T,K)+6\tau \log(A_{\max})+\mu\right).
    \end{align*}
\end{lemma}

\subsection{Analysis of the Outer Loop: the $v$-Function Update}\label{ap:v_analysis}

\begin{lemma}\label{le:vt-v*}
	It holds for all $t\geq 0$ that
	\begin{align*}
     \mathcal{L}_v(t+1)\leq 
    \gamma \mathcal{L}_v(t)+\frac{17 A_{\max}^2}{\tau^2(1-\gamma)^2}\mathcal{L}_\theta(t,K)+2\mathcal{L}_{\text{sum}}(t)+ 2\mathcal{L}_w(t,K)^{1/2}+12\tau \log(A_{\max})+\mu.
 \end{align*}
\end{lemma}
\begin{proof}[Proof of Lemma \ref{le:vt-v*}]
	For any $t\geq 0$ and $s\in\mathcal{S}$, we have
	\begin{align}
		\left|v_{t+1}^i(s)-v_*^i(s)\right|
		=\;&\left|(\pi_{t,K}^i(s))^\top\ceil{ q_{t,K}^i(s)}-v_*^i(s)\right|\nonumber \tag{Algorithm \ref{algorithm:reformulation_ap} Line $8$}\\
		\leq \;&\left|\mathcal{B}^i(v_t^i)(s)-\mathcal{B}^i(v_*^i)(s)\right|+\left|(\pi_{t,K}^i(s))^\top \ceil{q_{t,K}^i(s)}-\mathcal{B}^i(v_t^i)(s)\right|\tag{$v_*^i$ is the unique fixed point of $\mathcal{B}^i(\cdot)$}\\
		\leq \;& \gamma \|v_t^i-v_*^i\|_\infty+\left|(\pi_{t,K}^i(s))^\top \ceil{q_{t,K}^i(s)}-\mathcal{B}^i(v_t^i)(s)\right|,\label{eq1:le:vt-v*}
	\end{align}
	where the last line follows from the contraction property of the minimax Bellman operator $\mathcal{B}^i(\cdot)$.
	Next, we bound the second term on the RHS of the previous inequality. Observe that
	\begin{align}
		&\left|(\pi_{t,K}^i(s))^\top \ceil{q_{t,K}^i(s)}-\mathcal{B}^i(v_t^i)(s)\right|\nonumber\\
		\leq \;&\left|(\pi_{t,K}^i(s))^\top \ceil{q_{t,K}^i(s)}-(\pi_{t,K}^i(s))^\top \mathcal{T}^i(v_t^i)(s)\pi_{t,K}^{-i}(s)\right|\nonumber\\
		&+\left|(\pi_{t,K}^i(s))^\top \mathcal{T}^i(v_t^i)(s)\pi_{t,K}^{-i}(s)-(\pi_{t,K}^i(s))^\top \Tilde{q}_{t,K}^i(s)\right|\nonumber\\
		&+\left|(\pi_{t,K}^i(s))^\top \Tilde{q}_{t,K}^i(s)-\max_{\mu^i\in\Delta(\mathcal{A}^i)}(\mu^i)^\top \Tilde{q}_{t,K}^i(s)\right|\nonumber\\  
		&+\left|\max_{\mu^i\in\Delta(\mathcal{A}^i)}(\mu^i)^\top \Tilde{q}_{t,K}^i(s)-\max_{\mu^i\in\Delta(\mathcal{A}^i)}(\mu^i)^\top \mathcal{T}^i(v_t^i)(s)\pi_{t,K}^{-i}(s)\right|\nonumber\\
		&+\left|\max_{\mu^i\in\Delta(\mathcal{A}^i)}(\mu^i)^\top \mathcal{T}^i(v_t^i)(s)\pi_{t,K}^{-i}(s)-\max_{\mu^i\in\Delta(\mathcal{A}^i)}\min_{\mu^{-i}\in\Delta(\mathcal{A}^{-i})}(\mu^i)^\top \mathcal{T}^i(v_t^i)(s)\mu^{-i}\right|\nonumber\\
		\leq \;&\left\|q_{t,K}^i(s)-\mathcal{T}^i(v_t^i)(s)\pi_{t,K}^{-i}(s)\right\|_\infty+2\left\|\mathcal{T}^i(v_t^i)(s)\pi_{t,K}^{-i}(s)-\Tilde{q}_{t,K}^i(s)\right\|_\infty\nonumber\\
		&+\max_{\mu^i\in\Delta(\mathcal{A}^i)}(\mu^i-\pi_{t,K}^i(s))^\top \Tilde{q}_{t,K}^i(s)\nonumber\\
		&+\max_{\mu^i\in\Delta(\mathcal{A}^i)}(\mu^i)^\top \mathcal{T}^i(v_t^i)(s)\pi_{t,K}^{-i}(s)-\max_{\mu^i\in\Delta(\mathcal{A}^i)}\min_{\mu^{-i}\in\Delta(\mathcal{A}^{-i})}(\mu^i)^\top \mathcal{T}^i(v_t^i)(s)\mu^{-i}.\label{eq2:le:vt-v*}	
	\end{align}
	We next bound each term on the RHS of the previous inequality. For the first term, we have
	\begin{align}
		\left\|q_{t,K}^i(s)-\mathcal{T}^i(v_t^i)(s)\pi_{t,K}^{-i}(s)\right\|_\infty^2
		=\;&\left\|q_{t,K}^i(s)-\Bar{q}_{t,K}^i(s)\right\|_\infty^2\nonumber\\ 
		\leq \;&\left\|q_{t,K}^i-\Bar{q}_{t,K}^i\right\|_\infty^2\nonumber\\
        = \;&\left\|\Phi^i(w_{t,K}^i-\Bar{w}_{t,K}^i)\right\|_\infty^2\nonumber\\
        \leq \;&\|\Phi^i\|_{2,\infty}\|w_{t,K}^i-\Bar{w}_{t,K}^i\|_2^2\nonumber\\
        \leq \;&\|w_{t,K}^i-\Bar{w}_{t,K}^i\|_2^2,\label{eqeq:12}
	\end{align}
 where the last line follows from $\|\Phi^i\|_{2,\infty}=\max_{s,a^i}\|\phi^i(s,a^i)\|_2\leq 1$. For the second term on the RHS of Eq. (\ref{eq2:le:vt-v*}), we have
	\begin{align}
		2\left\|\mathcal{T}^i(v_t^i)(s)\pi_{t,K}^{-i}(s)-\Tilde{q}_{t,K}^i(s)\right\|_\infty\leq \;&2\left\|\mathcal{T}^i(v_t^i)(s)\pi_{t,K}^{-i}(s)-\Tilde{q}_{t,K}^i(s)\right\|_2\nonumber\\
		\leq \;&\frac{4\sqrt{2}A_{\max}}{\sqrt{\tau}(1-\gamma)}V_{v_t,s}^{1/2}(\tilde{q}_{t,K}^1(s),\tilde{q}_{t,K}^2(s))\tag{Lemma \ref{le:useful} and $\|\mathcal{T}^i(v_t^i)(s)\|_2\leq \frac{A_{\max}}{1-\gamma}$}\\
		\leq \;&\tau+\frac{8A_{\max}^2}{\tau^2(1-\gamma)^2}V_{v_t,s}(\tilde{q}_{t,K}^1(s),\tilde{q}_{t,K}^2(s))
		\label{eq4:le:vt-v*},
	\end{align}
	where the last inequality follows from $a^2+b^2\geq 2ab$.
	For the third term on the RHS of Eq. (\ref{eq2:le:vt-v*}), since 
 \begin{align*}
		\pi_{t,K}^i(s)=\sigma_\tau(\Tilde{q}_{t,K}^i(s))={\arg\max}_{\mu^i\in\Delta(\mathcal{A}^i)}\{(\mu^i)^\top \Tilde{q}_{t,K}^i(s)+\tau \nu(\mu^i)\},\tag{Algorithm \ref{algorithm:reformulation_ap} Line $4$}
	\end{align*}
 we have
	\begin{align}
		&\max_{\mu^i\in\Delta(\mathcal{A}^i)}(\mu^i-\pi_{t,K}^i(s))^\top \Tilde{q}_{t,K}^i(s)\nonumber\\
		\leq \;&\max_{\mu^i\in\Delta(\mathcal{A}^i)}\left\{(\mu^i-\pi_{t,K}^i(s))^\top \Tilde{q}_{t,K}^i(s)+\tau \nu(\mu^i)-\tau \nu(\pi_{t,K}^i(s))\right\}+2\tau \log(A_{\max})\nonumber\\
		= \;&2\tau \log(A_{\max}).\label{eq5:le:vt-v*}
	\end{align}
	For the last term on the RHS of Eq. (\ref{eq2:le:vt-v*}), we have by Eq. (\ref{eq3:le:vpi_to_vt}) that
	\begin{align}
		&\max_{\mu^i\in\Delta(\mathcal{A}^i)}(\mu^i)^\top \mathcal{T}^i(v_t^i)(s)\pi_{t,K}^{-i}(s)-\max_{\mu^i\in\Delta(\mathcal{A}^i)}\min_{\mu^{-i}\in\Delta(\mathcal{A}^{-i})}(\mu^i)^\top \mathcal{T}^i(v_t^i)(s)\mu^{-i}\nonumber\\
		\leq \;&\frac{9 A_{\max}^2}{\tau^2(1-\gamma)^2}V_{v_t,s}(\tilde{q}_{t,K}^1(s),\tilde{q}_{t,K}^2(s))+2\|v_t^1+v_t^2\|_\infty+6\tau \log(A_{\max})+\mu.\label{eq6:le:vt-v*}
	\end{align}
	Using the bounds we obtained in Eqs. (\ref{eqeq:12}), (\ref{eq4:le:vt-v*}), (\ref{eq5:le:vt-v*}), and (\ref{eq6:le:vt-v*}) together in Eq. (\ref{eq2:le:vt-v*}), we have
	\begin{align*}
		\left|(\pi_{t,K}^i(s))^\top \ceil{\tilde{q}_{t,K}^i(s)}-\mathcal{B}^i(v_t^i)(s)\right|
		\leq \;&\frac{17 A_{\max}^2}{\tau^2(1-\gamma)^2}V_{v_t,s}(\tilde{q}_{t,K}^1(s),\tilde{q}_{t,K}^2(s))+2\|v_t^1+v_t^2\|_\infty\\
  &+ \|w_{t,K}^i-\bar{w}_{t,K}^i\|_2+12\tau \log(A_{\max})+\mu.
	\end{align*}
 Using the previous inequality in Eq. (\ref{eq1:le:vt-v*}), we have
 \begin{align*}
     \|v_{t+1}^i-v_*^i\|_\infty\leq \;&\gamma \|v_t^i-v_*^i\|_\infty+\frac{17 A_{\max}^2}{\tau^2(1-\gamma)^2}\mathcal{L}_\theta(t,K)+2\mathcal{L}_{\text{sum}}(t)+ \|w_{t,K}^i-\bar{w}_{t,K}^i\|_2\\
     &+12\tau \log(A_{\max})+\mu.
 \end{align*}
 Summing up the resulting inequality for $i\in \{1,2\}$, we obtain
 \begin{align*}
     \mathcal{L}_v(t+1)\leq \;&\gamma \mathcal{L}_v(t)+\frac{17 A_{\max}^2}{\tau^2(1-\gamma)^2}\mathcal{L}_\theta(t,K)+2\mathcal{L}_{\text{sum}}(t)+ \sum_{i=1,2}\|w_{t,K}^i-\bar{w}_{t,K}^i\|_2+12\tau \log(A_{\max})+\mu\\
     \leq \;&\gamma \mathcal{L}_v(t)+\frac{17 A_{\max}^2}{\tau^2(1-\gamma)^2}\mathcal{L}_\theta(t,K)+2\mathcal{L}_{\text{sum}}(t)+ 2\mathcal{L}_w(t,K)^{1/2}+12\tau \log(A_{\max})+\mu,
 \end{align*}
 where the last line follows from
 \begin{align}\label{eq:w_Jensen}
     \sum_{i=1,2}\|w_{t,K}^i-\bar{w}_{t,K}^i\|_2=\sum_{i=1,2}(\|w_{t,K}^i-\bar{w}_{t,K}^i\|_2^2)^{1/2}\leq 2\bigg(\sum_{i=1,2}\|w_{t,K}^i-\bar{w}_{t,K}^i\|_2^2\bigg)^{1/2}.
 \end{align}
\end{proof}

\begin{lemma}\label{le:v_sum}
	It holds for all $t\geq 0$ that
	\begin{align*}
		\mathcal{L}_{\text{sum}}(t+1)\leq\gamma \mathcal{L}_{\text{sum}}(t)+2\mathcal{L}_{w}(t,K)^{1/2}.
	\end{align*}
\end{lemma}
\begin{proof}[Proof of Lemma \ref{le:v_sum}]
	For any $t\geq 0$ and $s\in\mathcal{S}$, we have 
	\begin{align*}
		\left|\sum_{i=1,2}v_{t+1}^i(s)\right|=\;&\left|\sum_{i=1,2} (\pi_{t,K}^i(s))^\top \ceil{q_{t,K}^i(s)}\right|\tag{Algorithm \ref{algorithm:reformulation_ap} Line $8$}\\
		=\;&\left|\sum_{i=1,2}(\pi_{t,K}^i(s))^\top \ceil{q_{t,K}^i(s)}-(\pi_{t,K}^i(s))^\top \mathcal{T}^i(v_t^i)(s)\pi_{t,K}^{-i}(s)\right|\\
		&+\left|\sum_{i=1,2}(\pi_{t,K}^i(s))^\top \mathcal{T}^i(v_t^i)(s)\pi_{t,K}^{-i}(s)\right|\\
		\leq \;&\gamma \|v_t^1+v_t^2\|_\infty+\sum_{i=1,2}\left\|q_{t,K}^i(s)-\mathcal{T}^i(v_t^i)(s)\pi_{t,K}^{-i}(s)\right\|_\infty\tag{Lemma \ref{le:use_for_v}}\\
  \leq \;&\gamma \|v_t^1+v_t^2\|_\infty+ \sum_{i=1,2}\|w_{t,K}^i-\bar{w}_{t,K}^i\|_2\tag{Eq. (\ref{eqeq:12})}\\
  \leq \;&\gamma \|v_t^1+v_t^2\|_\infty+ 2\mathcal{L}_{w}(t,K)^{1/2}.
	\end{align*}
	where the last line follows from Eq. (\ref{eq:w_Jensen}).
	Since the RHS of the previous inequality does not depend on $s$, we have the desired result.
\end{proof}

\subsection{Analysis of the Inner Loop}\label{ap:inner_loop}
In this section, we will focus on analyzing the update equations of the two-timescale iterates $\{w_{t,k}^i\}$ and $\{\theta_{t,k}^i\}$ updated in the inner loop of Algorithm \ref{algorithm:reformulation_ap}. We begin by writing only the inner loop of Algorithm \ref{algorithm:reformulation_ap} in Algorithm \ref{algorithm:inner_loop}, where we omit the subscript $t$. The results derived for Algorithm \ref{algorithm:inner_loop} can be easily combined with the analysis of the outer loop in previous sections using the tower property of conditional expectations and the Markov property.

\begin{algorithm}[ht]
	\caption{The Inner Loop of Algorithm \ref{algorithm:reformulation_ap}}
	\label{algorithm:inner_loop} 
	\begin{algorithmic}[1]
		\STATE \textbf{Input:} Integer $K$, initializations $w_0^i,\theta_0^i$ satisfying $\|w_0^i\|_2\leq M$, $\|\theta_0^i\|_2\leq M$, and a value function $v^i$ (satisfying $\|v^i\|_\infty\leq 1/(1-\gamma)$) from the outer loop. 
		\FOR{$k=0,1,\cdots,K-1$}
		\STATE $\Tilde{q}^i_{k+1}=\Tilde{q}_k^i+\beta_k (q_k^i-\Tilde{q}_k^i)$ and compute $\pi_{k+1}^i(S_k)=\sigma_\tau(\Tilde{q}_{k+1}^i(S_k))$
		\STATE Play
		$A_k^i\sim \pi_{k+1}^i(\cdot \mid S_k)$ (against $A_k^{-i}$), and observe $S_{k+1}\sim p(\cdot\mid S_k,A_k^i,A_k^{-i})$
		\STATE $w_{k+1}^i=\pj_M^i[w_k^i+\alpha_k\phi^i(S_k,A_k^i)(R_i(S_k,A_k^i,A_k^{-i})+\gamma v^i(S_{k+1})-\phi^i(S_k,A_k^i)^\top w_k^i)]$ 
		\ENDFOR
	\end{algorithmic}
\end{algorithm}

\subsubsection{Analysis of the Slow-Timescale Iterates}\label{ap:pi_analysis}
Note that Algorithm \ref{algorithm:inner_loop} Line $3$ is equivalent to
\begin{align}
	\tilde{q}^i_{k+1}(s)=\tilde{q}_k^i(s)+\beta_k (\mathcal{T}^i(v^i)(s)\sigma_\tau(\tilde{q}_k^{-i}(s))-\tilde{q}_k^i(s))+\beta_k (q_k^i(s)-\mathcal{T}^i(v^i)(s)\sigma_\tau(\tilde{q}_k^{-i}(s))),\;\; \forall\,s\in\mathcal{S}.\label{eq:policy_equivalent}
\end{align}
The last term on the RHS of the previous inequality represents the estimation error of the fast-timescale iterates $\{w_k^i\}$, which will be analyzed in Appendix \ref{ap:q_analysis}. Here we focus on the evolution of the $q$-function $\tilde{q}_k^i$ associated with the slow-timescale iterate $\theta_k^i$, which we use to compute the policies. In view of Eq. (\ref{eq:policy_equivalent}), it is enough to study the following joint update equation:
\begin{align}
	x_{k+1}=\;&x_k+\beta_k(B\sigma_\tau(y_k)-x_k+\mathcal{E}_{k,x}),\label{eq:policy_simple1}\\
	y_{k+1}=\;&y_k+\beta_k(A\sigma_\tau(x_k)-y_k+\mathcal{E}_{k,y}),\label{eq:policy_simple2}
\end{align}
where $x_k,\mathcal{E}_{k,x}\in\mathbb{R}^{n}$, $y_k,\mathcal{E}_{k,y}\in\mathbb{R}^m$, $A\in\mathbb{R}^{m\times n}$, and $B\in\mathbb{R}^{n \times m}$. Note that we do not assume $A+B^\top =0$. We next use a Lyapunov approach to study the behavior of $(x_k,y_k)$. Let
\begin{align*}
	V_1(x,y)=\;&\max_{u'\in  \Delta^m}\min_{u\in\Delta^m}\left\{(u'-u)^\top y+\tau \nu(u')-\tau \nu(u)+\frac{1}{2\mu}\|x-Bu\|_2^2\right\},\\
	V_2(x,y)=\;&\max_{v'\in  \Delta^n}\min_{v\in\Delta^n}\left\{(v'-v)^\top x+\tau \nu(v')-\tau \nu(v)+\frac{1}{2\mu}\|y-Av\|_2^2\right\},
\end{align*}
where $\mu>0$ is a tunable parameter. 
Let $V:\mathbb{R}^n\times \mathbb{R}^m\mapsto \mathbb{R}$ be the Lyapunov function defined as $V(x,y)=V_1(x,y)+V_2(x,y)$.
For simplicity of notation, we denote
\begin{align*}
	p(\mu,x,y)=\;&{\arg\min}_{u\in\Delta^m}\left\{-u^\top y-\tau \nu(u)+\frac{1}{2\mu}\|x-Bu\|_2^2\right\},\\
	q(\mu,x,y)=\;&{\arg\min}_{v\in\Delta^n}\left\{-v^\top x-\tau \nu(v)+\frac{1}{2\mu}\|y-Av\|_2^2\right\},
\end{align*}
both of which are well-defined due to strong convexity. The following proposition shows that the joint update equations (\ref{eq:policy_simple1}) and (\ref{eq:policy_simple2}) produce a negative drift with respect to the Lyapunov function $V(\cdot,\cdot)$.

\begin{proposition}\label{prop:drift}
	Suppose that there exists $L_b\geq 0$ such that $\|B\sigma_\tau(y_k)-x_k+\mathcal{E}_{k,x}\|_2\leq L_b$ and $\|A\sigma_\tau(x_k)-y_k+\mathcal{E}_{k,y}\|_2\leq L_b$ for all $k\geq 0$. Then, by choosing $\mu = \tau/64$, we have for all $k\geq 0$ that
	\begin{align*}
		V(x_{k+1},y_{k+1})\leq\;&\left(1-\frac{\beta_k}{2}\right)V(x_k,y_k)+ \frac{520 \beta_k}{\tau}(\|\mathcal{E}_{k,x}\|_2^2+\|\mathcal{E}_{k,y}\|_2^2)\\
  &+4\beta_k\max_{i,j}|A(i,j)+B (j,i)|+\frac{138L_b\beta_k^2}{\tau}.
	\end{align*}
\end{proposition}
\begin{remark}
    By letting $A=X_1$, $B=X_2$, $\mathcal{E}_{k,x}=0$, and $\mathcal{E}_{k,y}=0$, Proposition \ref{prop:drift} implies Proposition \ref{prop:drift1}.
\end{remark}
\begin{proof}[Proof of Proposition \ref{prop:drift}]
A sequence of properties regarding the Lyapunov function $V(\cdot)$ is provided in Appendix \ref{ap:lemma_policy}. For any $k\geq 0$, we have by the smoothness property of the Lyapunov function (cf. Lemma \ref{le:smoothness}) that
	\begin{align}
		V(x_{k+1},y_{k+1})-V(x_k,y_k)
		\leq  \;&\underbrace{\langle \nabla_x V(x_k,y_k),x_{k+1}-x_k\rangle}_{E_1}+\underbrace{\frac{L_V}{2}\|x_{k+1}-x_k\|_2^2}_{E_2}\nonumber\\
		&+\underbrace{\langle \nabla_y V(x_k,y_k),y_{k+1}-y_k\rangle}_{E_3}+\underbrace{\frac{L_V}{2}\|y_{k+1}-y_k\|_2^2}_{E_4}.\label{eq:E1toE4}
	\end{align}
	We next bound all the terms on the RHS of the previous inequality.
	For the term $E_2$, we have
	\begin{align*}
		E_2=\frac{L_V}{2}\|x_{k+1}-x_k\|_2^2
		=\frac{L_V \beta_k^2}{2}\|B\sigma_\tau (y_k)-x_k+\mathcal{E}_{k,x}\|_2^2
        \leq \frac{L_V L_b^2\beta_k^2 }{2}.
	\end{align*}
	Similarly, we also have $E_4\leq \frac{L_V L_b^2\beta_k^2 }{2}$.
	Next, we consider the term $E_1$. Using Eq. (\ref{eq:policy_simple1}) and Lemma \ref{le:Danskin}, we have
	\begin{align}
		E_1=\;&\langle \nabla_x V(x_k,y_k), x_{k+1}-x_k\rangle\nonumber\\
		=\;&\beta_k\underbrace{\frac{1}{\mu}\langle x_k-Bp(\mu,x_k,y_k), B \sigma_\tau(y_k)-x_k\rangle}_{E_{1,1}}+\beta_k\underbrace{\frac{1}{\mu}\langle x_k-Bp(\mu,x_k,y_k), \mathcal{E}_{k,x}\rangle}_{E_{1,2}}\nonumber\\
		&+\beta_k\underbrace{\langle \sigma_\tau(x_k)-q(\mu,x_k,y_k), B \sigma_\tau(y_k)-x_k\rangle}_{E_{1,3}}+\beta_k\underbrace{\langle \sigma_\tau(x_k)-q(\mu,x_k,y_k), \mathcal{E}_{k,x}\rangle}_{E_{1,4}}\label{eq:E1_decomposition}
	\end{align}
	We next bound the terms $\{E_{1,i}\}_{1\leq i\leq 4}$. For the term $E_{1,1}$, we have
	\begin{align*}
		E_{1,1}=\;&\frac{1}{\mu}\langle x_k-Bp(\mu,x_k,y_k),  B\sigma_\tau(y_k)- B p(\mu,x_k,y_k)+B p(\mu,x_k,y_k)-x_k\rangle\\
		=\;&\frac{1}{\mu}\langle B^\top (x_k-Bp(\mu,x_k,y_k)),  \sigma_\tau(y_k)- p(\mu,x_k,y_k)\rangle-\frac{1}{\mu}\|x_k-Bp(\mu,x_k,y_k)\|_2^2\\
		=\;&\langle -y_k-\tau \nabla \nu(p(\mu,x_k,y_k)),  \sigma_\tau(y_k)- p(\mu,x_k,y_k)\rangle-\frac{1}{\mu}\|x_k-Bp(\mu,x_k,y_k)\|_2^2\tag{Lemma \ref{le:identities}}\\
		=\;&-y_k^\top (\sigma_\tau(y_k)- p(\mu,x_k,y_k))+\tau\langle  \nabla \nu(p(\mu,x_k,y_k)),  p(\mu,x_k,y_k)- \sigma_\tau(y_k)\rangle\\
		&-\frac{1}{\mu}\|x_k-Bp(\mu,x_k,y_k)\|_2^2\\
		\leq \;&-y_k^\top (\sigma_\tau(y_k)- p(\mu,x_k,y_k))-\tau(\nu(\sigma_\tau(y_k))-\nu(p(\mu,x_k,y_k)))-\frac{1}{\mu}\|x_k-Bp(\mu,x_k,y_k)\|_2^2\tag{The concavity of $\nu(\cdot)$}\\
		=\;&-\left(y_k^\top \sigma_\tau(y_k)+\tau\nu(\sigma_\tau(y_k))-y_k^\top  p(\mu,x_k,y_k)-\tau\nu(p(\mu,x_k,y_k))+\frac{1}{\mu}\|x_k-Bp(\mu,x_k,y_k)\|_2^2\right)\\
		\leq \;&- V_1(x_k,y_k),
	\end{align*}
	where the last line follows from the definition of $V_1(\cdot,\cdot)$. For the term $E_{1,2}$ on the RHS of Eq. (\ref{eq:E1_decomposition}), we have
	\begin{align*}
		E_{1,2}=\;&\frac{1}{\mu}\langle x_k-Bp(\mu,x_k,y_k), \mathcal{E}_{k,x}\rangle\\
		\leq \;&\frac{1}{\mu}\|x_k-Bp(\mu,x_k,y_k)\|_2\|\mathcal{E}_{k,x}\|_2\\
		\leq \;&\frac{\sqrt{2}}{\sqrt{\mu}}V_1(x_k,y_k)^{1/2}\|\mathcal{E}_{k,x}\|_2\tag{Lemma \ref{le:useful}}\\
		\leq \;&\frac{V_1(x_k,y_k)}{16}+\frac{8}{\mu}\|\mathcal{E}_{k,x}\|_2^2,
	\end{align*}
	where the last line follows from $a^2+b^2\geq 2ab$. For the term $E_{1,3}$ on the RHS of Eq. (\ref{eq:E1_decomposition}), we have
	\begin{align*}
		E_{1,3}=\;&\langle \sigma_\tau(x_k)-q(\mu,x_k,y_k), B\sigma_\tau(y_k)-B p(\mu,x_k,y_k)+B p(\mu,x_k,y_k)-x_k\rangle\\
  =\;&\langle \sigma_\tau(x_k)-q(\mu,x_k,y_k), B p(\mu,x_k,y_k)-x_k\rangle\\
		&+ (\sigma_\tau(x_k)-q(\mu,x_k,y_k))^\top B (\sigma_\tau(y_k)- p(\mu,x_k,y_k))\\
		\leq \;&\| \sigma_\tau(x_k)-q(\mu,x_k,y_k)\|_2 \|B p(\mu,x_k,y_k)-x_k\|_2\\
		&+ (\sigma_\tau(x_k)-q(\mu,x_k,y_k))^\top B (\sigma_\tau(y_k)- p(\mu,x_k,y_k))\\
		\leq \;& \frac{2\sqrt{\mu}}{\sqrt{\tau}} V_2(x_k,y_k)^{1/2}V_1(x_k,y_k)^{1/2}+ (\sigma_\tau(x_k)-q(\mu,x_k,y_k))^\top B (\sigma_\tau(y_k)- p(\mu,x_k,y_k))\tag{Lemma \ref{le:useful}}\\
		\leq \;&\frac{\sqrt{\mu}}{\sqrt{\tau}} V(x_k,y_k)+ (\sigma_\tau(x_k)-q(\mu,x_k,y_k))^\top B (\sigma_\tau(y_k)- p(\mu,x_k,y_k)).
	\end{align*}
	where the last line follows from $a+b\geq 2\sqrt{ab}$ for all $a,b\geq 0$. For the term $E_{1,4}$ on the RHS of Eq. (\ref{eq:E1_decomposition}), we have
	\begin{align*}
		E_{1,4}=\;&\langle \sigma_\tau(x_k)-q(\mu,x_k,y_k), \mathcal{E}_{k,x}\rangle\\
		\leq \;&\|\sigma_\tau(x_k)-q(\mu,x_k,y_k)\|_2 \|\mathcal{E}_{k,x}\|_2\\
		\leq \;&\frac{\sqrt{2}}{\sqrt{\tau}}V_2(x_k,y_k)^{1/2} \|\mathcal{E}_{k,x}\|_2\tag{Lemma \ref{le:useful}}\\
		\leq \;& \frac{V_2(x_k,y_k)}{16}+\frac{8}{\tau }\|\mathcal{E}_{k,x}\|_2^2,
	\end{align*}
	where the last line follows from $a+b\geq 2\sqrt{ab}$ for all $a,b\geq 0$. Using the upper bounds we obtained for $\{E_{1,i}\}_{1\leq i\leq 4}$ all together in Eq. (\ref{eq:E1_decomposition}), we have
	\begin{align*}
		E_1=\;&\beta_k (E_{1,1}+E_{1,2}+E_{1,3}+E_{1,4})\\
  \leq \;&-\beta_k V_1(x_k,y_k)+\frac{\beta_k}{16}V(x_k,y_k)+8\beta_k\left(\frac{1}{\mu}+\frac{1}{\tau}\right)\|\mathcal{E}_{k,x}\|_2^2+\frac{\sqrt{\mu}}{\sqrt{\tau}}\beta_k V(x_k,y_k)\\
  &+ \beta_k (\sigma_\tau(x_k)-q(\mu,x_k,y_k))^\top B (\sigma_\tau(y_k)- p(\mu,x_k,y_k)).
	\end{align*}
	Using an identical argument, we also have
	\begin{align*}
		E_3
		\leq\;&-\beta_k V_2(x_k,y_k)+\frac{\beta_k}{16}V(x_k,y_k)+8\beta_k\left(\frac{1}{\mu}+\frac{1}{\tau}\right)\|\mathcal{E}_{k,y}\|_2^2+\frac{\sqrt{\mu}}{\sqrt{\tau}}\beta_k V(x_k,y_k)\\
  &+ \beta_k (\sigma_\tau(y_k)-p(\mu,x_k,y_k))^\top A (\sigma_\tau(x_k)- q(\mu,x_k,y_k))
	\end{align*}
	Using the upper bounds we obtained for the terms $E_1$ to $E_4$ all together in Eq. (\ref{eq:E1toE4}), we have
	\begin{align*}
		V(x_{k+1},y_{k+1})\leq\;&\left(1-\frac{7}{8}\beta_k+\frac{2\sqrt{\mu}}{\sqrt{\tau}}\beta_k\right)V(x_k,y_k)+8\beta_k\left(\frac{1}{\mu}+\frac{1}{\tau}\right)(\|\mathcal{E}_{k,x}\|_2^2+\|\mathcal{E}_{k,y}\|_2^2)\\
		&+\beta_k(\sigma_\tau(y_k)-p(\mu,x_k,y_k))^\top (A+B^\top ) (\sigma_\tau(x_k)- q(\mu,x_k,y_k))+L_vL_b\beta_k^2.
	\end{align*}
 Since $L_V=2/\mu+2/\tau+1/\sqrt{\mu\tau}$ and $\mu = \tau/64$, we have from the previous inequality that
\begin{align*}
		V(x_{k+1},y_{k+1})\leq\;&\left(1-\frac{\beta_k}{2}\right)V(x_k,y_k)+ \frac{520 \beta_k}{\tau}(\|\mathcal{E}_{k,x}\|_2^2+\|\mathcal{E}_{k,y}\|_2^2)\\
  &+\beta_k(\sigma_\tau(y_k)-p(\mu,x_k,y_k))^\top (A+B^\top ) (\sigma_\tau(x_k)- q(\mu,x_k,y_k))+\frac{138L_b\beta_k^2}{\tau}\\
  \leq \;&\left(1-\frac{\beta_k}{2}\right)V(x_k,y_k)+ \frac{520 \beta_k}{\tau}(\|\mathcal{E}_{k,x}\|_2^2+\|\mathcal{E}_{k,y}\|_2^2)\\
  &+4\beta_k\max_{i,j}|A(i,j)+B (j,i)|+\frac{138L_b\beta_k^2}{\tau}
	\end{align*}
\end{proof}

We next apply Proposition \ref{prop:drift} to establish a negative drift inequality of the slow-timescale iterate $\theta_k^i$.

\begin{lemma}\label{le:policy}
It holds for all $k\geq 0$ that
\begin{align}
	\mathcal{L}_\theta(k+1)
 \leq \left(1-\frac{\beta_k}{2}\right)\mathcal{L}_\theta(k)+ \frac{520 A_{\max}\beta_k}{\tau}\mathcal{L}_w(k)+ 4\beta_k\|v^1+v^2\|_\infty+\frac{276A_{\max}^{1/2}\beta_k^2}{\tau\lambda^{1/2}(1-\gamma)}.\label{drift:policy}
\end{align}
\end{lemma}

\begin{proof}[Proof of Lemma \ref{le:policy}]
    To apply Proposition \ref{prop:drift} to the policy update equation in Algorithm \ref{algorithm:inner_loop} Line $3$ (or equivalently Eq. (\ref{eq:policy_equivalent})), we first identify the constant $L_b$. Observe that
\begin{align*}
    \|q_k^i(s)-\Tilde{q}_k^i(s)\|_2=\;&\left[\sum_{a^i}(\phi^i(s,a^i)^\top (w_k^i-\theta_k^i))^2\right]^{1/2}\\
    \leq \;&\left[\sum_{a^i}\|\phi^i(s,a^i\|_2^2 \|w_k^i-\theta_k^i\|_2^2\right]^{1/2}\\
    \leq \;&\left[\sum_{a^i} (\|w_k^i\|_2+\|\theta_k^i\|_2)^2\right]^{1/2}\tag{$\|\phi^i(s,a^i)\|_2\leq 1$ for all $(s,a^i)$}\\
    \leq \;&2A_{\max}^{1/2}M\\
    =\;&\frac{2A_{\max}^{1/2}}{\lambda^{1/2}(1-\gamma)}.
\end{align*}
Therefore, we have for all $k\geq 0$ and $s\in\mathcal{S}$ that
\begin{align*}
	V_{v,s}(\tilde{q}_{k+1}^1(s),\tilde{q}_{k+1}^2(s))
	\leq\;&\left(1-\frac{\beta_k}{2}\right)V_{v,s}(\tilde{q}_k^1(s),\tilde{q}_k^2(s))+ \frac{520 \beta_k}{\tau}\sum_{i=1,2}\|q_k^i(s)-\mathcal{T}^i(v^i)(s)\pi_k^{-i}(s)\|_2^2\\
	&+ 4\beta_k\max_{a^1,a^2}\left|\mathcal{T}^1(v^1)(s,a^1,a^2)+[\mathcal{T}^2(v^2)(s,a^2,a^1)]^\top  \right|+\frac{276A_{\max}^{1/2}\beta_k^2}{\tau\lambda^{1/2}(1-\gamma)}\\
 \leq \;&\left(1-\frac{\beta_k}{2}\right)V_{v,s}(\tilde{q}_k^1(s),\tilde{q}_k^2(s))+ \frac{520 A_{\max}\beta_k}{\tau}\sum_{i=1,2}\|w_k^i-\Bar{w}_k^i\|_2^2\\
	&+ 4\beta_k\|v^1+v^2\|_\infty+\frac{276A_{\max}^{1/2}\beta_k^2}{\tau\lambda^{1/2}(1-\gamma)},
\end{align*}
where the last line follows from
\begin{align*}
    \sum_{i=1,2}\|q_k^i(s)-\mathcal{T}^i(v^i)(s)\pi_k^{-i}(s)\|_2^2=\;&\sum_{i=1,2}\|q_k^i(s)-\Bar{q}_k^i(s)\|_2^2\\
    =\;&\sum_{i=1,2}\sum_{a^i}(\phi^i(s,a^i)^\top (w_k^i-\Bar{w}_k^i))^2\\
    \leq \;&\sum_{i=1,2}\sum_{a^i}\|\phi^i(s,a^i)\|_2^2\|w_k^i-\Bar{w}_k^i\|_2^2\\
    \leq \;&A_{\max}\sum_{i=1,2}\|w_k^i-\Bar{w}_k^i\|_2^2.\tag{$\|\phi^i(s,a^i)\|_2\leq 1$ for all $(s,a^i)$}
\end{align*}
As a result, we have
\begin{align*}
	\mathcal{L}_\theta(k+1)
 \leq \;&\left(1-\frac{\beta_k}{2}\right)\mathcal{L}_\theta(k)+ \frac{520 A_{\max}\beta_k}{\tau}\mathcal{L}_w(k)+ 4\beta_k\|v^1+v^2\|_\infty+\frac{276A_{\max}^{1/2}\beta_k^2}{\tau\lambda^{1/2}(1-\gamma)}.
\end{align*}
\end{proof}

\subsubsection{Properties of the Lyapunov function $V(x,y)$}\label{ap:lemma_policy}
\begin{lemma}\label{le:Danskin}
	For any $x\in\mathbb{R}^n$ and $y\in\mathbb{R}^m$, we have
	\begin{align*}
		\nabla_xV_1(x,y)=\;&\frac{1}{\mu}(x-Bp(\mu,x,y)),&\quad 
		\nabla_yV_1(x,y)=\;&\sigma_\tau(y)-p(\mu,x,y),\\
		\nabla_xV_2(x,y)=\;&\sigma_\tau(x)-q(\mu,x,y),&\quad 
		\nabla_yV_2(x,y)=\;&\frac{1}{\mu}(y-Aq(\mu,x,y)).
	\end{align*}
	
\end{lemma}
\begin{proof}[Proof of Lemma \ref{le:Danskin}]
	We will only show the computation of $ \nabla_xV_1(x,y)$. The computation of all other gradients follows from an identical approach.
	We begin by writing $V_1(\cdot,\cdot)$ in the following:
	\begin{align}
		V_1(x,y)=\max_{u'\in  \Delta^n}\left\{u'^\top y+\tau \nu(u')\right\}+\min_{u\in\Delta^n}\left\{-u^\top y-\tau \nu(u)+\frac{1}{2\mu}\|x-Bu\|_2^2\right\}.\label{eq1:le:Danskin}
	\end{align}
	We next compute the gradient of each term on the RHS of the previous inequality.
	The first term on the RHS of Eq. (\ref{eq1:le:Danskin}) is not a function of $x$. For the second term, observe that
	\begin{align*}
		\min_{u\in\Delta^n}\left\{-u^\top y-\tau \nu(u)+\frac{1}{2\mu}\|x-Bu\|_2^2\right\}
  =\;&-\max_{u\in\Delta^n}\left\{u^\top y+\tau \nu(u)-\frac{1}{2\mu}\|x-Bu\|_2^2\right\}\\
		=\;&-\max_{u\in\Delta^n}\underbrace{\left\{u^\top y+\tau \nu(u)-\frac{1}{2\mu}\|Bu\|_2^2+\frac{1}{\mu}x^\top Bu\right\}}_{:=G(\mu,x,y,u)}\\
  &+\frac{1}{2\mu}\|x\|_2^2.
	\end{align*}
	Since $G(\mu,x,y,u)$ as a function of $x$ is linear, hence convex, we have by Danskin's theorem that
	\begin{align*}
		\nabla_x\max_{u\in\Delta^n}G(\mu,x,y,u)=\frac{1}{\mu}Bp(\mu,x,y).
	\end{align*}
 where we recall that $p(\mu,x,y)={\arg\max}_{u\in\Delta^n}G(\mu,x,y,u)$.
	It follows that
 \begin{align*}
     \nabla_xV_1(x,y)=-\nabla_x\max_{u\in\Delta^n}G(\mu,x,y,u)+\frac{1}{\mu}x
     =\frac{1}{\mu}(x-Bp(\mu,x,y)).
 \end{align*}
\end{proof}
\begin{lemma}\label{le:smoothness}
	The Lyapunov function $V(x,y)$ is $L_V$ -- smooth, where $L_V=2/\mu+2/\tau+1/\sqrt{\mu\tau}$.
\end{lemma}
\begin{proof}[Proof of Lemma \ref{le:smoothness}]
	For any $x_1,x_2\in\mathbb{R}^n$ and $y_1,y_2\in\mathbb{R}^m$, we have by Lemma  \ref{le:Danskin} that
	\begin{align*}
		\nabla_x V(x_1,y_1)-\nabla_x V(x_2,y_2)=\;&\frac{1}{\mu}(x_1-Bp(\mu,x_1,y_1))+\sigma_\tau(x_1)-q(\mu,x_1,y_1)\\
		&-\frac{1}{\mu}(x_2-Bp(\mu,x_2,y_2))-\sigma_\tau(x_2)+q(\mu,x_2,y_2)\\
		= \;&\frac{1}{\mu}(x_1-x_2)+\frac{1}{\mu}B(p(\mu,x_2,y_2)-p(\mu,x_1,y_1))\\
		&+\sigma_\tau(x_1)-\sigma_\tau(x_2)+q(\mu,x_2,y_2)-q(\mu,x_1,y_1)\\
		= \;&\frac{1}{\mu}(x_1-x_2)+\frac{1}{\mu}B(p(\mu,x_2,y_2)-p(\mu,x_1,y_2))\\
		&+\frac{1}{\mu}B(p(\mu,x_1,y_2)-p(\mu,x_1,y_1))\\
		&+\sigma_\tau(x_1)-\sigma_\tau(x_2)+q(\mu,x_2,y_2)-q(\mu,x_1,y_2)\\
		&+q(\mu,x_1,y_2)-q(\mu,x_1,y_1).
	\end{align*}
	Taking $\|\cdot\|_2$ on both sides of the previous inequality and then using triangle inequality, we have
	\begin{align*}
		\|\nabla_x V(x_1,y_1)-\nabla_x V(x_2,y_2)\|_2\leq \;&\frac{1}{\mu}\|x_1-x_2\|_2+\frac{1}{\mu}\|B(p(\mu,x_2,y_2)-p(\mu,x_1,y_2))\|_2\\
		&+\frac{1}{\mu}\|B(p(\mu,x_1,y_2)-p(\mu,x_1,y_1))\|_2\\
		&+\|\sigma_\tau(x_1)-\sigma_\tau(x_2)\|_2+\|q(\mu,x_2,y_2)-q(\mu,x_1,y_2)\|_2\\
		&+\|q(\mu,x_1,y_2)-q(\mu,x_1,y_1)\|_2\\
  \leq \;&\frac{1}{\mu}\|x_1-x_2\|_2+\frac{1}{\mu}\|x_1-x_2\|_2+\frac{1}{2\sqrt{\mu\tau}}\|y_1-y_2\|_2\\
		&+\frac{1}{\tau}\|x_1-x_2\|_2+\frac{1}{\tau}\|x_1-x_2\|_2\\
		&+\frac{1}{2\sqrt{\mu\tau}}\|y_1-y_2\|_2\tag{Lemmas \ref{le:sensitivity1} and \ref{le:sensitivity2}}\\
		= \;&\left(\frac{2}{\mu}+\frac{2}{\tau}\right)\|x_1-x_2\|_2+\frac{1}{\sqrt{\tau \mu}}\|y_1-y_2\|_2.
	\end{align*}
	Similarly, we also have
	\begin{align*}
		\|\nabla_y V(x_1,y_1)-\nabla_y V(x_2,y_2)\|_2
		\leq \;&\left(\frac{2}{\mu}+\frac{2}{\tau}\right)\|y_1-y_2\|_2+\frac{1}{\sqrt{\tau \mu}}\|x_1-x_2\|_2.
	\end{align*}
	It follows that
	\begin{align*}
		&\left\|\nabla V(x_1,y_1)-\nabla V(x_2,y_2)\right\|_2^2\\
		=\;&\|\nabla_x V(x_1,y_1)-\nabla_x V(x_2,y_2)\|_2^2+\|\nabla_y V(x_1,y_1)-\nabla_y V(x_2,y_2)\|_2^2\\
		\leq \;&\left[\left(\frac{2}{\mu}+\frac{2}{\tau}\right)\|y_1-y_2\|_2+\frac{1}{\sqrt{\tau \mu}}\|x_1-x_2\|_2\right]^2+\left[\left(\frac{2}{\mu}+\frac{2}{\tau}\right)\|x_1-x_2\|_2+\frac{1}{\sqrt{\tau \mu}}\|y_1-y_2\|_2\right]^2\\
		\leq \;&\left(\frac{2}{\mu}+\frac{2}{\tau}+\frac{1}{\sqrt{\tau \mu}}\right)^2(\|x_1-x_2\|_2^2+\|y_1-y_2\|_2^2).
	\end{align*}
	which implies
	\begin{align*}
		\left\|\nabla V(x_1,y_1)-\nabla V(x_2,y_2)\right\|_2\leq \left(\frac{2}{\mu}+\frac{2}{\tau}+\frac{1}{\sqrt{\tau \mu}}\right)(\|x_1-x_2\|_2^2+\|y_1-y_2\|_2^2)^{1/2}.
	\end{align*}
	As a result, the Lyapunov function $V(\cdot,\cdot)$ is $L_V$ -- smooth, where $L_V=2/\mu+2/\tau+1/\sqrt{\mu\tau}$.
\end{proof}

\begin{lemma}\label{le:identities}
	For any $x\in\mathbb{R}^n$ and $y\in\mathbb{R}^m$, we have
	\begin{align*}
		y+\tau \nabla \nu(\sigma_\tau(y))=0,\quad \text{and}\quad 
		y+\tau \nabla \nu(p(\mu,x,y))=\frac{1}{\mu}B^\top(Bp(\mu,x,y)-x).
	\end{align*}
\end{lemma}
\begin{proof}[Proof of Lemma \ref{le:identities}]
	Since $\sigma_\tau(y)={\arg\max}_{u\in\Delta^m}\{u^\top y+\tau \nu(u)\}$, the first claimed inequality follows from the first order optimality condition. The second inequality follows from an identical approach.
\end{proof}

\begin{lemma}\label{le:sensitivity1}
	The following two inequalities hold for all $x_1,x_2\in\mathbb{R}^n$, and $y\in\mathbb{R}^m$:
	\begin{align*}
		\|B (p(\mu,x_2,y)-p(\mu,x_1,y))\|_2\leq\;& \|x_2-x_1\|_2\\
		\|p(\mu,x_2,y)-p(\mu,x_1,y)\|_2\leq\;& \frac{1}{2\sqrt{\mu \tau}}\|x_2-x_1\|_2.
	\end{align*}
\end{lemma}
\begin{proof}[Proof of Lemma \ref{le:sensitivity1}]
	For any $x_1,x_2\in\mathbb{R}^n$, and $y\in\mathbb{R}^m$, we have by Lemma \ref{le:identities} that
	\begin{align*}
		-y-\tau \nabla \nu(p(\mu,x_1,y))+\frac{1}{\mu}B^\top (Bp(\mu,x_1,y)-x_1)=\;&0,\\
		-y-\tau \nabla \nu(p(\mu,x_2,y))+\frac{1}{\mu}B^\top (Bp(\mu,x_2,y)-x_2)=\;&0.
	\end{align*}
	It follows that
	\begin{align*}
		B^\top (x_2-x_1)=B^\top B (p(\mu,x_2,y)-p(\mu,x_1,y))+\mu\tau (\nabla \nu(p(\mu,x_1,y))-\nabla \nu(p(\mu,x_2,y))),
	\end{align*}
	which implies
	\begin{align*}
		&\|B(p(\mu,x_2,y)-p(\mu,x_1,y))\|_2 \|x_2-x_1\|_2\\
		\geq \;&(p(\mu,x_2,y)-p(\mu,x_1,y))^\top B^\top (x_2-x_1)\tag{Cauchy–Schwarz inequality}\\
		=\;&\|B (p(\mu,x_2,y)-p(\mu,x_1,y)\|_2^2+\mu\tau (p(\mu,x_2,y)-p(\mu,x_1,y))^\top(\nabla \nu(p(\mu,x_1,y))-\nabla \nu(p(\mu,x_2,y)))\tag{The previous inequality}\\
		\geq \;&\|B (p(\mu,x_2,y)-p(\mu,x_1,y))\|_2^2+\mu\tau \|p(\mu,x_2,y)-p(\mu,x_1,y)\|_2^2,
	\end{align*}
	where the last line follows from $-\nu(\cdot)$ being a $1$ -- strongly convex function with respect to $\|\cdot\|_2$ \cite[Example 5.27]{beck2017first}.
	As a result, we have from the previous inequality that
	\begin{align*}
		\|B (p(\mu,x_2,y)-p(\mu,x_1,y))\|_2\leq \|x_2-x_1\|_2
	\end{align*}
	and
	\begin{align*}
		\|p(\mu,x_2,y)-p(\mu,x_1,y)\|_2\leq \frac{1}{2\sqrt{\mu \tau}}\|x_2-x_1\|_2.
	\end{align*}
\end{proof}

\begin{lemma}\label{le:sensitivity2}
	The following two inequalities hold for all $x\in\mathbb{R}^n$ and $y_1,y_2\in\mathbb{R}^m$:
	\begin{align*}
		\|p(\mu,x,y_2)-p(\mu,x,y_1)\|_2\leq\;& \frac{1}{\tau}\|y_2-y_1\|_2,\\
		\| B(p(\mu,x,y_2)-p(\mu,x,y_1))\|_2\leq\;& \frac{\sqrt{\mu }}{2\sqrt{\tau}}\|y_2-y_1\|_2.
	\end{align*}
\end{lemma}
\begin{proof}[Proof of Lemma \ref{le:sensitivity2}]
	For any $x\in\mathbb{R}^n$ and $y_1,y_2\in\mathbb{R}^m$, we have by Lemma \ref{le:identities} that
	\begin{align*}
		-y_1-\tau \nabla \nu(p(\mu,x,y_1))+\frac{1}{\mu}B^\top (Bp(\mu,x,y_1)-x)=\;&0\\
		-y_2-\tau \nabla \nu(p(\mu,x,y_2))+\frac{1}{\mu}B^\top (Bp(\mu,x,y_2)-x)=\;&0.
	\end{align*}
	It follows that
	\begin{align*}
		y_2-y_1=\tau \nabla \nu(p(\mu,x,y_1))-\tau \nabla \nu(p(\mu,x,y_2))+\frac{1}{\mu}B^\top B(p(\mu,x,y_2)-p(\mu,x,y_1)),
	\end{align*}
	which implies
	\begin{align*}
		&\|p(\mu,x,y_2)-p(\mu,x,y_1)\|_2\|y_2-y_1\|_2\\
		\geq \;&(p(\mu,x,y_2)-p(\mu,x,y_1))^\top (y_2-y_1)\tag{Cauchy–Schwarz inequality}\\
		=\;&\tau (p(\mu,x,y_2)-p(\mu,x,y_1))^\top (\nabla \nu(p(\mu,x,y_1))-\nabla \nu(p(\mu,x,y_2)))\tag{The previous inequality}\\
		&+\frac{1}{\mu}\| B(p(\mu,x,y_2)-p(\mu,x,y_1))\|_2^2\\
		\geq \;& \tau \|p(\mu,x,y_2)-p(\mu,x,y_1)\|_2^2+\frac{1}{\mu}\| B(p(\mu,x,y_2)-p(\mu,x,y_1))\|_2^2,
	\end{align*}
	where the last line follows from the $1$ -- strong convexity of $-\nu(\cdot)$. The previous inequality implies
	\begin{align*}
		\|p(\mu,x,y_2)-p(\mu,x,y_1)\|_2\leq \frac{1}{\tau}\|y_2-y_1\|_2,
	\end{align*}
	and 
	\begin{align*}
		\| B(p(\mu,x,y_2)-p(\mu,x,y_1))\|_2\leq \frac{\sqrt{\mu }}{2\sqrt{\tau}}\|y_2-y_1\|_2.
	\end{align*}
\end{proof}

\begin{lemma}\label{le:useful}
	The following $2$ inequalities hold for any $x\in\mathbb{R}^n$ and $y\in\mathbb{R}^m$:
	\begin{align*}
		\|\sigma_\tau(y)-p(\mu,x,y)\|_2\leq\;& \frac{\sqrt{2}}{\sqrt{\tau}}V_1(x,y)^{1/2},\\
		\|B p(\mu,x,y)-x\|_2\leq\;& \sqrt{2\mu}V_1(x,y)^{1/2}.
	\end{align*}
 In addition, suppose that $\mu\leq \|B\|_2^2/\tau$. Then we have
 \begin{align*}
     \|B \sigma_\tau(y)-x\|_2\leq  \frac{2\sqrt{2}\|B\|_2}{\sqrt{\tau}}V_1(x,y)^{1/2}.
 \end{align*}
\end{lemma}
\begin{proof}[Proof of Lemma \ref{le:useful}]
For any $x\in\mathbb{R}^n$ and $y\in\mathbb{R}^m$, since $-u^\top y-\tau \nu(u)$ as a function $u$ is $\tau$ -- strongly convex with respect to $\|\cdot\|_2$, we have by the quadratic growth property of strongly convex functions that
		\begin{align*}
			&\|\sigma_\tau(y)-p(\mu,x,y)\|_2^2\\
   \leq\;& \frac{2}{\tau}\left[\sigma_\tau(y)^\top y+\tau \nu(\sigma_\tau(y))-p(\mu,x,y)^\top y-\tau\nu(p(\mu,x,y))\right]\\
   \leq \;&\frac{2}{\tau}\left[\sigma_\tau(y)^\top y+\tau \nu(\sigma_\tau(y))-p(\mu,x,y)^\top y-\tau\nu(p(\mu,x,y))+\frac{1}{2\mu}\|x-Bp(\mu,x,y)\|_2^2\right]\\
   =\;&\frac{2}{\tau}\left[\max_{u'\in  \Delta^m}\left\{u'^\top y+\tau \nu(u')\right\}+\min_{u\in\Delta^m}\left\{-u^\top y-\tau \nu(u)+\frac{1}{2\mu}\|x-Bu\|_2^2\right\}\right]\\
			=\;& \frac{2}{\tau} V_1(x,y).
		\end{align*}
		Taking square root on both sides of the previous inequality, we obtain
		\begin{align*}
			\|\sigma_\tau(y)-p(\mu,x,y)\|_2\leq \frac{\sqrt{2}}{\sqrt{\tau}}V_1(x,y)^{1/2}.
		\end{align*}
    Next, we prove the second claimed inequality. For any $x\in\mathbb{R}^n$ and $y\in\mathbb{R}^m$, we have
		\begin{align*}
			\|B p(\mu,x,y)-x\|_2^2= 2\mu \frac{1}{2\mu}\|B p(\mu,x,y)-x\|_2^2\leq 2\mu V_1(x,y),
		\end{align*}
		which implies
		\begin{align*}
			\|B p(\mu,x,y)-x\|_2\leq \sqrt{2\mu}V_1(x,y)^{1/2}.
		\end{align*}
  Finally, we prove the last claimed inequality. For any $x\in\mathbb{R}^n$ and $y\in\mathbb{R}^m$, we have
	\begin{align*}
		\|B\sigma_\tau(y)-x\|_2
		\leq \;&\|B\|_2\|\sigma_\tau(y)-p(\mu,x,y)\|_2+\|Bp(\mu,x,y)-x\|_2\tag{Triangle inequality}\\
		\leq \;&\frac{\sqrt{2}\|B\|_2}{\sqrt{\tau}}V_1(x,y)^{1/2}+\sqrt{2\mu}V_1(x,y)^{1/2}\tag{The previous two inequalities in this lemma}\\
		\leq \;&\frac{2\sqrt{2}\|B\|_2}{\sqrt{\tau}}V_1(x,y)^{1/2},
	\end{align*}
	where the last line follows from $\mu\leq \|B\|_2^2/\tau$.
\end{proof}

\subsubsection{Analysis of the Fast-Timescale Iterates}\label{ap:q_analysis}
Fixing $v^i\in\mathbb{R}^{|\mathcal{S}|}$, let $F^i:\mathbb{R}^{d_i}\times \mathcal{S}\times \mathcal{A}^i\times \mathcal{A}^{-i}\times \mathcal{S}\mapsto \mathbb{R}^{d_i}$ be an operator defined as
\begin{align*}
	F^i(w^i,s,a^i,a^{-i},s')=\phi^i(s,a^i)(R_i(s,a^i,a^{-i})+\gamma v^i(s')-\phi^i(s,a^i)^\top w^i)
\end{align*}
for all $(s,a^i,a^{-i},s')$ and $w\in\mathbb{R}^{d_i}$. Then, Algorithm \ref{algorithm:inner_loop} Line $5$ can be equivalently written as
\begin{align}\label{eq:q_sa}
    w_{k+1}^i=\pj_M^i\left[w_k^i+\alpha_k F^i(w_k^i,S_k,A_k^i,A_k^{-i},S_{k+1})\right].
\end{align}
For any $k\geq 0$, define $\Bar{F}_k^i:\mathbb{R}^{d_i}\mapsto \mathbb{R}^{d_i}$ as
\begin{align*}
	\Bar{F}_k^i(w^i)=\mathbb{E}_{S\sim \mu_k(\cdot),A^i\sim \pi_k^i(\cdot\mid S),A^{-i}\sim \pi_k^{-i}(\cdot\mid S),S'\sim p(\cdot\mid S,A^i,A^{-i})}[F^i(w^i,S,A^i,A^{-i},S')]
\end{align*}
for all $w^i\in\mathbb{R}^{d_i}$, where $\mu_k\in\Delta(\mathcal{S})$ is the stationary distribution of the Markov chain $\{S_k\}$ induced by the joint policy $\pi_k=(\pi_k^1,\pi_k^2)$. Note that $\mu_k$ is guaranteed to exist and is unique under Assumption \ref{as:ergodicity}. Eq. (\ref{eq:q_sa}) can then be viewed as a stochastic approximation algorithm for solving the time-varying equation $\Bar{F}_k^i(w^i)=0$. The properties of the operators $F^i(\cdot)$ and $\bar{F}_k^i(\cdot)$ are established in Lemma \ref{le:propertiesF}.

Using the equivalent formulation of the update equation (\ref{eq:q_sa}), we have for all $k\geq 0$ that
\begin{align}
	&\mathbb{E}[\|w_{k+1}^i-\Bar{w}_{k+1}^i\|_2^2]\nonumber\\
 =\;&\mathbb{E}[\|\pj_M^i[w_k^i+\alpha_kF^i(w_k^i,S_k,A_k^i,A_k^{-i},S_{k+1})]-\pj_M^i(\Bar{w}_{k+1}^i)\|_2^2]\nonumber\\
 \leq \;&\mathbb{E}[\|w_k^i+\alpha_kF^i(w_k^i,S_k,A_k^i,A_k^{-i},S_{k+1})-\Bar{w}_{k+1}^i\|_2^2]\tag{$\pj_M^i(\cdot)$ is non-expansive w.r.t. $\|\cdot\|_2$}\nonumber\\
	= \;&\mathbb{E}[\|\alpha_kF^i(w_k^i,S_k,A_k^i,A_k^{-i},S_{k+1})+w_k^i-\Bar{w}_k^i+\Bar{w}_k^i-\Bar{w}_{k+1}^i\|_2^2]\nonumber\\
	= \;&\mathbb{E}[\|w_k^i-\Bar{w}_k^i\|_2^2]+\alpha_k^2\mathbb{E}[\| F^i(w_k^i,S_k,A_k^i,A_k^{-i},S_{k+1})\|_2^2]+\mathbb{E}[\|\Bar{w}_k^i-\Bar{w}_{k+1}^i\|_2^2]\nonumber\\
	&+2\alpha_k \mathbb{E}[(w_k^i-\Bar{w}_k^i)^\top (\Bar{F}_k^i(w_k^i))]\nonumber\\ 
 &+2\alpha_k \mathbb{E}[(w_k^i-\Bar{w}_k^i)^\top (F^i(w_k^i,S_k,A_k^i,A_k^{-i},S_{k+1})-\Bar{F}_k^i(w_k^i))]\nonumber\\ 
 &+2\alpha_k \mathbb{E}[(\Bar{w}_k^i-\Bar{w}_{k+1}^i)^\top  F^i(w_k^i,S_k,A_k^i,A_k^{-i},S_{k+1})]\nonumber\\
	&+2\mathbb{E}[(w_k^i-\Bar{w}_k^i)^\top (\Bar{w}_k^i-\Bar{w}_{k+1}^i)]\nonumber\\
 \leq \;&(1-2\lambda\alpha_k)\mathbb{E}[\|w_k^i-\Bar{w}_k^i\|_2^2]+\alpha_k^2\mathbb{E}[\| F^i(w_k^i,S_k,A_k^i,A_k^{-i},S_{k+1})\|_2^2]+\mathbb{E}[\|\Bar{w}_k^i-\Bar{w}_{k+1}^i\|_2^2]\nonumber\\
 &+2\alpha_k \mathbb{E}[(w_k^i-\Bar{w}_k^i)^\top (F^i(w_k^i,S_k,A_k^i,A_k^{-i},S_{k+1})-\Bar{F}_k^i(w_k^i))]\nonumber\\ 
 &+\mathbb{E}[\|\Bar{w}_k^i-\Bar{w}_{k+1}^i\|_2^2]+\alpha_k^2\mathbb{E}[\|F^i(w_k^i,S_k,A_k^i,A_k^{-i},S_{k+1}\|_2^2]\nonumber\\
	&+\frac{\lambda \alpha_k}{2}\mathbb{E}[\|w_k^i-\Bar{w}_k^i\|_2^2]+\frac{2}{\lambda \alpha_k}\mathbb{E}[\|\Bar{w}_k^i-\Bar{w}_{k+1}^i\|_2^2]\nonumber\\
 \leq \;&\left(1-\frac{3\lambda\alpha_k}{2}\right)\mathbb{E}[\|w_k^i-\Bar{w}_k^i\|_2^2]+2\alpha_k^2\underbrace{\mathbb{E}[\| F^i(w_k^i,S_k,A_k^i,A_k^{-i},S_{k+1})\|_2^2]}_{E_{w,1}}+\frac{4}{\lambda \alpha_k}\underbrace{\mathbb{E}[\|\Bar{w}_k^i-\Bar{w}_{k+1}^i\|_2^2]}_{E_{w,2}}\nonumber\\
 &+2\alpha_k \underbrace{\mathbb{E}[(w_k^i-\Bar{w}_k^i)^\top (F^i(w_k^i,S_k,A_k^i,A_k^{-i},S_{k+1})-\Bar{F}_k^i(w_k^i))]}_{E_{w,3}},\label{eq:q_drift_decomposition}
\end{align}
where the second last inequality follows from Lemma \ref{le:propertiesF} (2), Cauchy–Schwarz inequality, and $a^2+b^2\geq 2\sqrt{ab}$, and the last inequality follows from $\alpha_0\leq 1/\lambda$. Next, we bound the terms $E_{w,1}$ to $E_{w,3}$ on the RHS of the previous inequality. 

Consider the term $E_{w,1}$. Using Lemma \ref{le:propertiesF} (1) and (2), we have
\begin{align*}
E_{w,1}=\;&\mathbb{E}[\| F^i(w_k^i,S_k,A_k^i,A_k^{-i},S_{k+1})\|_2^2]\\
	=\;&\mathbb{E}[\| F^i(w_k^i,S_k,A_k^i,A_k^{-i},S_{k+1})-F^i(0,S_k,A_k^i,A_k^{-i},S_{k+1})+F^i(0,S_k,A_k^i,A_k^{-i},S_{k+1})\|_2^2]\\
	\leq \;&\mathbb{E}[(\|w_k^i\|_2+1/(1-\gamma))^2]\\
	\leq \;&\left(\frac{1}{\lambda^{1/2}(1-\gamma)}+\frac{1}{1-\gamma}\right)^2\tag{$\|w_k^i\|_2\leq M=\frac{1}{\lambda^{1/2}(1-\gamma)}$}\\
	\leq \;&\frac{4}{\lambda(1-\gamma)^2},
\end{align*}
where the last line follows from
\begin{align*}
    \lambda^2
    \leq\;& \max_{i\in \{1,2\}}\max_{w^i\in\mathbb{R}^{d_i},\|w^i\|_2=1}\|\Phi^iw^i\|_\infty\\
    =\;&\max_{i\in \{1,2\}}\max_{w^i\in\mathbb{R}^{d_i},\|w^i\|_2=1}\max_{s,a^i}|\phi^i(s,a^i)^\top w^i|\\
    \leq\;&\max_{i\in \{1,2\}} \max_{s,a^i}\|\phi^i(s,a^i)\|_2\\
    \leq \;&1.
\end{align*}
Next, we consider the term $E_{w,2}$. Using Lemma \ref{le:propertiesF} (4), we have 
\begin{align*}
    E_{w,2}\leq \frac{4A_{\max}^2\beta_k^2}{\tau^2 \lambda^2(1-\gamma)^4}.
\end{align*}
Finally, we bound the term $E_{w,3}$ in the following lemma.

\begin{lemma}[Proof in Appendix \ref{pf:le:Markov_noise}]\label{le:Markov_noise}
    Under Condition \ref{con:stepsize}, we have for all $k\geq z_k$ that
    \begin{align*}
        E_{w,3}\leq \frac{17z_k\alpha_{k-z_k,k-1}}{\tau\lambda^{3/2} (1-\gamma)^3}.
    \end{align*}
\end{lemma}

Using the upper bounds we obtained for the terms $E_{w,1}$ to $E_{w,3}$ in Eq. (\ref{eq:q_drift_decomposition}), we have for all $k\geq z_k$ that
\begin{align*}
	\mathbb{E}[\|w_{k+1}^i-\Bar{w}_{k+1}^i\|_2^2]
	\leq \;&\left(1-\frac{3\lambda\alpha_k}{2}\right)\mathbb{E}[\|w_k^i-\Bar{w}_k^i\|_2^2]+\frac{8\alpha_k^2}{\lambda(1-\gamma)^2}+\frac{16A_{\max}^2\beta_k^2}{\tau^2 \lambda^3(1-\gamma)^4\alpha_k}\nonumber\\
 &+ \frac{34z_k\alpha_k\alpha_{k-z_k,k-1}}{\tau\lambda^{3/2} (1-\gamma)^3}\\
 \leq \;&\left(1-\frac{3\lambda\alpha_k}{2}\right)\mathbb{E}[\|w_k^i-\Bar{w}_k^i\|_2^2]+\frac{42z_k\alpha_k\alpha_{k-z_k,k-1}}{\tau\lambda^{3/2} (1-\gamma)^3}+\frac{16A_{\max}^2\beta_k^2}{\tau^2 \lambda^3(1-\gamma)^4\alpha_k},
\end{align*}
where the last inequality follows from $\lambda\leq 1$ and $\tau\leq 1/(1-\gamma)$.
Summing up the previous inequality for $i\in \{1,2\}$, we have the following lemma.
\begin{lemma}\label{le:drift_inequality_q}
    It holds for all $k\geq z_k$ that
    \begin{align*}
        \mathbb{E}[\mathcal{L}_w(k+1)]
 \leq \;&\left(1-\frac{3\lambda\alpha_k}{2}\right) \mathbb{E}[\mathcal{L}_w(k)]+\frac{84z_k\alpha_k\alpha_{k-z_k,k-1}}{\tau\lambda^{3/2} (1-\gamma)^3}+\frac{32A_{\max}^2\beta_k^2}{\tau^2 \lambda^3(1-\gamma)^4\alpha_k}.
    \end{align*}
\end{lemma}

\subsubsection{Supporting Lemmas}
\begin{lemma}\label{le:propertiesF}
	The operators $F^i(\cdot)$ and $\bar{F}_k^i(\cdot)$ have the following properties.
	\begin{enumerate}[(1)]
		\item It holds for any $w_1^i,w_2^i\in\mathbb{R}^{d_i}$ and $(s,a^i,a^{-i},s')$ that
		\begin{align*}
			\|F^i(w_1^i,s,a^i,a^{-i},s')-F^i(w_2^i,s,a^i,a^{-i},s')\|_2
			\leq\;& \|w_1^i-w_2^i\|_2,\\
			\|F^i(0,s,a^i,a^{-i},s')\|_2\leq\;& \frac{1}{1-\gamma}.
		\end{align*}
		\item It holds for any $k\geq 0$ that $\langle \Bar{F}_k(w_1^i)-\Bar{F}_k(w_2^i),w_1^i-w_2^i \rangle\leq  -\lambda\|w_1^i-w_2^i\|_2^2$.
		\item It holds for any $k\geq 0$ that $\|\bar{w}_k^i\|_2\leq \lambda^{-1/2} (1-\gamma)^{-1}$.
		\item It holds for all $k\geq  0$ that $\|\bar{w}_{k+1}^i-\bar{w}_{k}^i\|_2\leq \frac{2A_{\max}\beta_k}{\tau \lambda(1-\gamma)^2}$.
	\end{enumerate}
\end{lemma}
\begin{proof}[Proof of Lemma \ref{le:propertiesF}]
	\begin{enumerate}[(1)]
		\item For any $w_1^i,w_2^i\in\mathbb{R}^{d_i}$ and $(s,a^i,a^{-i},s')$, by definition of $F^i(\cdot)$, we have
		\begin{align*}
			\|F^i(w_1^i,s,a^i,a^{-i},s')-F^i(w_2^i,s,a^i,a^{-i},s')\|_2
			=\;&\|\phi^i(s,a^i)\phi^i(s,a^i)^\top (w_1^i-w_2^i)\|_2\\
			\leq \;&\|\phi^i(s,a^i)\|_2^2\|w_1^i-w_2^i\|_2\\
			\leq \;&\|w_1^i-w_2^i\|_2,
		\end{align*}
		where the last line follows from $\max_{s,a^i}\|\phi^i(s,a^i)\|_2\leq 1$. Similarly, for any $(s,a^i,a^{-i},s')$, we have
		\begin{align*}
			\|F^i(0,s,a^i,a^{-i},s')\|_2
			=\;&|R_i(s,a^i,a^{-i})+\gamma v^i(s')|\|\phi^i(s,a^i)\|_2\\
			\leq \;&1+\frac{\gamma}{1-\gamma}\\
			\leq \;&\frac{1}{1-\gamma}.
		\end{align*}
		\item For simplicity of notation, for $i\in \{1,2\}$, let $\mathcal{H}^i:\mathbb{R}^{|\mathcal{S}|}\times \mathbb{R}^{|\mathcal{S}||\mathcal{A}^{-i}|}\mapsto \mathbb{R}^{|\mathcal{S}||\mathcal{A}^i|}$ be an operator defined as 
\begin{align}\label{def:H}
    [\mathcal{H}^i(v,\pi^{-i})](s,a^i)=\mathbb{E}\left[R_i(s,a^i,A_0^{-i})+\gamma v(S_1)\;\middle|\;S_0=s,A_0^i=a^i,A_0^{-i}\sim \pi^{-i}(\cdot\mid S_0)\right]
\end{align}
for all $v$, $\pi^{-i}$, and $(s,a^i)$. Observe that the operator $\Bar{F}_k(\cdot)$ is explicitly given as
		\begin{align*}
			\Bar{F}_k(w^i)=(\Phi^i)^\top D_k^i(\mathcal{H}^i(v^i,\pi_k^{-i})-\Phi^iw^i),
		\end{align*}
		where $D_k^i=\text{diag}(\{\mu_k(s)\pi_k^i(a^i|s)\}_{(s,a^i)\in\mathcal{S}\times \mathcal{A}^i})$. Therefore, the equation $\Bar{F}_k(w^i)=0$ has a unique solution 
		\begin{align*}
			\Bar{w}_k^i=[(\Phi^i)^\top D_k^i\Phi^i]^{-1}(\Phi^i)^\top D_k^i\mathcal{H}^i(v^i,\pi_k^{-i}).
		\end{align*}
		
		In addition, for any $w_1^i,w_2^i\in\mathbb{R}^{d_i}$, we have
		\begin{align*}
			\langle \Bar{F}_k(w_1^i)-\Bar{F}_k(w_2^i),w_1^i-w_2^i \rangle=\;&-(w_1^i-w_2^i)^\top ((\Phi^i)^\top D_k^i\Phi^i)(w_1^i-w_2^i)\\
			\leq\;& -\lambda\|w_1^i-w_2^i\|_2^2,
		\end{align*}
		where we recall from Assumption \ref{as:ergodicity} that $\lambda=\min_{i\in \{1,2\}}\inf_{\pi\in \Pi}\lambda_{\min}((\Phi^i)^\top D_\pi^i\Phi^i)>0$. 
		\item Observe that 
  \begin{align}\label{eq:barw1}
      \|\Phi^i\bar{w}_k^i\|_{D_k^i}=\left[(\bar{w}_k^i)^\top (\Phi^i)^\top D_k^i \Phi^i\bar{w}_k^i\right]^{1/2}\geq \lambda^{1/2} \|\bar{w}_k^i\|_2.
  \end{align}
		Moreover,
		using the explicit expression of $\bar{w}_k^i$ from Part (2) of this lemma, we have
		\begin{align*}
			\|\Phi^i\bar{w}_k^i\|_{D_k^i}=\;&\|\Phi^i[(\Phi^i)^\top D_k^i\Phi^i]^{-1}(\Phi^i)^\top D_k^i\mathcal{H}^i(v^i,\pi_k^{-i})\|_{D_k^i}\\
			= \;&\|\text{Proj}_k^i\mathcal{H}^i(v^i,\pi_k^{-i})\|_{D_k^i}\\
			\leq  \;&\|\mathcal{H}^i(v^i,\pi_k^{-i})\|_{D_k^i}\tag{$\text{Proj}_k^i(\cdot)$ being non-expansive with respect to $\|\cdot\|_{D_k^i}$}\\
   \leq\;& \|\mathcal{H}^i(v^i,\pi_k^{-i})\|_\infty\\
   \leq\;& 1+\gamma \|v^i\|_\infty\\
   \leq\;& \frac{1}{1-\gamma},
		\end{align*}
  where the second last line follows from the definition of $\mathcal{H}^i(\cdot,\cdot)$ in Eq. (\ref{def:H}).
		Combining the previous inequality and Eq. (\ref{eq:barw1}), we have $\|\bar{w}_k^i\|_2\leq \lambda^{-1/2} (1-\gamma)^{-1}$.
		\item For any $k\geq 0$, we have by Assumption \ref{as:Bellman_Completeness} that
        \begin{align*}
			\Phi^i\bar{w}_{k}^i=\Phi^i[(\Phi^i)^\top D_k^i\Phi^i]^{-1}(\Phi^i)^\top D_k^i\mathcal{H}^i(v^i,\pi_k^{-i})=\text{Proj}_k^i\mathcal{H}^i(v^i,\pi_k^{-i})=\mathcal{H}^i(v^i,\pi_k^{-i}).
		\end{align*}
  Let $D^i\in\mathbb{R}^{|\mathcal{S}||\mathcal{A}^i|\times |\mathcal{S}||\mathcal{A}^i|}$ be 
  a diagonal matrix such that each of the diagonal components is $1/(|\mathcal{S}||\mathcal{A}^i|)$.
  Then, on the one hand, we have
  \begin{align*}
      \|\Phi^i\bar{w}_{k+1}^i-\Phi^i\bar{w}_{k}^i\|_D^2\geq \;&\lambda_{\min}((\Phi^i)^\top D\Phi^i)\|\bar{w}_{k+1}^i-\bar{w}_{k}^i\|_2^2\geq \lambda\|\bar{w}_{k+1}^i-\bar{w}_{k}^i\|_2^2.\tag{Assumption \ref{as:ergodicity}}
  \end{align*}
  On the other hand, we have
  \begin{align*}
      &\|\mathcal{H}^i(v^i,\pi_{k+1}^{-i})-\mathcal{H}^i(v^i,\pi_k^{-i})\|_D\\
      \leq \;&\|\mathcal{H}^i(v^i,\pi_{k+1}^{-i})-\mathcal{H}^i(v^i,\pi_k^{-i})\|_\infty\\			=\;&\max_{s,a^i}\left|\sum_{a^{-i}}\mathcal{T}^i(v^i)(s,a^i,a^{-i})(\pi_k^{-i}(a^{-i}\mid s)-\pi_{k+1}^{-i}(a^{-i}\mid s))\right|\\
    \leq \;&\max_{s,a^i}\left\{\max_{a^{-i}}|\mathcal{T}^i(v^i)(s,a^i,a^{-i})|\right\}\sum_{a^{-i}}|\pi_k^{-i}(a^{-i}\mid s)-\pi_{k+1}^{-i}(a^{-i}\mid s)|\\
    \leq \;&\frac{1}{1-\gamma}\max_{s}\|\pi_{k+1}^{-i}(s)-\pi_k^{-i}(s)\|_1\\
    \leq \;&\frac{2A_{\max}\beta_k}{\tau \lambda^{1/2}(1-\gamma)^2}.\tag{Lemma \ref{le:Lipschitzpi}}
    \end{align*}
    It follows from the previous $3$ inequalities that
    \begin{align*}
        \|\bar{w}_{k+1}^i-\bar{w}_{k}^i\|_2\leq \frac{2A_{\max}\beta_k}{\tau \lambda(1-\gamma)^2}.
    \end{align*}
	\end{enumerate}
\end{proof}

\begin{lemma}\label{le:Lipschitzw}
	Given non-negative integers $k_1\leq k_2$, we have for any $k\in \{k_1,\cdots,k_2\}$ that
	\begin{align*}
		\|w_{k_2}^i-w_{k_1}^i\|_2\leq  \frac{2\alpha_{k_1,k_2-1}}{\lambda^{1/2}(1-\gamma)}.
	\end{align*}
\end{lemma}
\begin{proof}[Proof of Lemma \ref{le:Lipschitzw}]
    For any $k\in \{k_1,\cdots,k_2-1\}$, we have
    \begin{align*}
        \|w_{k+1}^i-w_k^i\|_2\leq\;& \alpha_k\|F^i(w_k^i,S_k,A_k^i,A_k^{-i},S_{k+1})\|_2\\
        \leq \;&\alpha_k(\|F^i(w_k^i,S_k,A_k^i,A_k^{-i},S_{k+1})-F^i(0,S_k,A_k^i,A_k^{-i},S_{k+1})\|_2)\\
        &+\alpha_k\|F^i(0,S_k,A_k^i,A_k^{-i},S_{k+1})\|_2\\
        \leq \;&\alpha_k\left(\|w_k^i\|_2+\frac{1}{1-\gamma}\right)\tag{Lemma \ref{le:propertiesF} (1) and (2)}\\
        \leq \;&\alpha_k\left(\frac{1}{\lambda^{1/2}(1-\gamma)}+\frac{1}{1-\gamma}\right)\tag{$\|w_k^i\|_2\leq M=\frac{1}{\lambda^{1/2}(1-\gamma)}$}\\
        \leq \;&\frac{2\alpha_k}{\lambda^{1/2}(1-\gamma)}.\tag{$\lambda\leq 1$}
    \end{align*}
    Therefore, we have by telescoping that
    \begin{align*}
        \|w_{k_2}^i-w_{k_1}^i\|_2\leq\frac{2\alpha_{k_1,k_2-1}}{\lambda^{1/2}(1-\gamma)}.
    \end{align*}
\end{proof}
\begin{lemma}\label{le:Lipschitzpi}
	Given non-negative integers $k_1\leq k_2$ and $i\in \{1,2\}$, we have 
	\begin{align*}
		\|\pi_{k_2}^i(s)-\pi_{k_1}^i(s)\|_1\leq  \frac{2A_{\max}\beta_{k_1,k_2-1}}{\tau \lambda^{1/2}(1-\gamma)}.
	\end{align*}
\end{lemma}
\begin{proof}[Proof of Lemma \ref{le:MC_Lipschitz}]
    For any $k\in \{k_1,\cdots,k_2-1\}$, we have
    \begin{align*}
        \|\pi_{k+1}^i(s)-\pi_k^i(s)\|_1=\;&\|\sigma_\tau(\Tilde{q}_{k+1}^i(s))-\sigma_\tau(\Tilde{q}_k^i(s))\|_1\\
        \leq \;&\frac{1}{\tau}\|\Tilde{q}_{k+1}^i(s)-\Tilde{q}_k^i(s)\|_1\\
        \leq \;&\frac{\beta_k}{\tau}\|q_k^i(s)-\Tilde{q}_k^i(s)\|_1\\
        \leq \;&\frac{\beta_k}{\tau}\sum_{a^i}\left|\phi^i(s,a^i)^\top (w_k^i-\theta_k^i)\right|\\
        \leq \;&\frac{\beta_k}{\tau}\sum_{a^i}\|\phi^i(s,a^i)\|_2 (\|w_k^i\|_2+\|\theta_k^i\|_2)\\
        \leq \;&\frac{2A_{\max}\beta_k}{\tau \lambda^{1/2}(1-\gamma)}.
    \end{align*}
    The result follows from the previous inequality and telescoping.
\end{proof}
\begin{lemma}\label{le:MC_Lipschitz}
    Under Assumption \ref{as:ergodicity}, there exists $L_p\geq 1$ such that
    \begin{align*}
    \|\mu_{\pi_1}-\mu_{\pi_2}\|_1\leq L_p\max_{s\in\mathcal{S}}\sum_{i=1,2}\|\pi_1^i(s)-\pi_2^i(s)\|_1.
    \end{align*}
\end{lemma}
\begin{proof}[Proof of Lemma \ref{le:MC_Lipschitz}]
    For any joint policies $\pi,\Bar{\pi}\in\Pi$, under Assumption \ref{as:ergodicity}, the Markov chain induced by either $\pi$ or $\Bar{\pi}$ is irreducible and aperiodic, hence is uniformly ergodic. Therefore, for any $\pi,\Bar{\pi}\in \Pi$, the existing Markov chain sensitivity analysis \cite[Corollary 3.1]{mitrophanov2005sensitivity} implies that
\begin{align*}
\|\mu_{\pi}-\mu_{\Bar{\pi}}\|_1\leq \;&\left(\hat{n}+\frac{C\rho^{\hat{n}}}{1-\rho}\right)\|P_{\pi}-P_{\Bar{\pi}}\|_\infty\tag{$\hat{n}=\min\{k\in \mathbb{Z}_{+}\,:\, k\geq \log_\rho C^{-1}\}$}\\
\leq \;&\left(\hat{n}+\frac{C\rho^{\hat{n}}}{1-\rho}\right)\max_{s\in\mathcal{S}}\|\pi(a^1,a^2\mid s)-\Bar{\pi}(a^1,a^2\mid s)\|_1\\
\leq \;&\left(\hat{n}+\frac{C\rho^{\hat{n}}}{1-\rho}\right)\max_{s\in\mathcal{S}}\sum_{a^1,a^2}|\pi^1(a^1\mid s)\pi^2(a^2\mid s)-\Bar{\pi}^1(a^1\mid s)\Bar{\pi}^2(a^2\mid s)|\\
\leq \;&\left(\hat{n}+\frac{C\rho^{\hat{n}}}{1-\rho}\right)\max_{s\in\mathcal{S}}\left(\sum_{a^2}|\pi^2(a^2\mid s)-\Bar{\pi}^2(a^2\mid s)|+\sum_{a^1}|\pi^1(a^1\mid s)-\Bar{\pi}^1(a^1\mid s)|\right)\\
=\;&\left(\hat{n}+\frac{C\rho^{\hat{n}}}{1-\rho}\right)\max_{s\in\mathcal{S}}\sum_{i=1,2}\|\pi_1^i(s)-\pi_2^i(s)\|_1.
\end{align*}
Since
\begin{align*}
    \hat{n}+\frac{C\rho^{\hat{n}}}{1-\rho}\leq \frac{\log(\rho/C)}{\log(\rho)}+\frac{1}{1-\rho},
\end{align*}
The result follows by letting $L_p=\frac{\log(\rho/C)}{\log(\rho)}+\frac{1}{1-\rho}$.

\subsubsection{Proof of Lemma \ref{le:Markov_noise}}\label{pf:le:Markov_noise}
Inspired by \cite{bhandari2018finite,srikant2019finite,zou2019finite,chen2023finite}, we use a conditioning argument to control $E_{w,3}$. For any $k\geq z_k$, we have
\begin{align}
	E_{w,3}=\;&\mathbb{E}[\langle F^i(w_k^i,S_k,A_k^i,A_k^{-i},S_{k+1})-\bar{F}_k^i(w_k^i),w_k^i-\bar{w}_k^i\rangle]\nonumber\\
	= \;&\underbrace{\mathbb{E}[\langle F^i(w_{k-z_k}^i,S_k,A_k^i,A_k^{-i},S_{k+1})-\bar{F}_{k-z_k}^i(w_{k-z_k}^i),w_{k-z_k}^i-\bar{w}_{k-z_k}^i\rangle]}_{E_{w,3,1}}\nonumber\\
	&+\underbrace{\mathbb{E}[\langle F^i(w_{k-z_k}^i,S_k,A_k^i,A_k^{-i},S_{k+1})-\bar{F}_{k-z_k}^i(w_{k-z_k}^i),w_k^i-w_{k-z_k}^i\rangle]}_{E_{w,3,2}}\nonumber\\
	&+\underbrace{\mathbb{E}[\langle F^i(w_{k-z_k}^i,S_k,A_k^i,A_k^{-i},S_{k+1})-\bar{F}_{k-z_k}^i(w_{k-z_k}^i),\bar{w}_{k-z_k}^i-\bar{w}_k^i\rangle]}_{E_{w,3,3}}\nonumber\\
	&+\underbrace{\mathbb{E}[\langle F^i(w_k^i,S_k,A_k^i,A_k^{-i},S_{k+1})-F^i(w_{k-z_k}^i,S_k,A_k^i,A_k^{-i},S_{k+1}),w_k^i-\bar{w}_k^i\rangle]}_{E_{w,3,4}}\nonumber\\
	&+\underbrace{\mathbb{E}[\langle \bar{F}_{k-z_k}^i(w_{k-z_k}^i)-\bar{F}_k^i(w_k^i),w_k^i-\bar{w}_k^i\rangle]}_{E_{w,3,5}}.\label{eq:Ew3}
\end{align}
We next bound each term on the RHS of the previous inequality.

\paragraph{The Term $E_{w,3,1}$.} Using the tower property of conditional expectations, we have
\begin{align}
	E_{w,3,1}=\;&\mathbb{E}[\langle F^i(w_{k-z_k}^i,S_k,A_k^i,A_k^{-i},S_{k+1})-\bar{F}_{k-z_k}^i(w_{k-z_k}^i),w_{k-z_k}^i-\bar{w}_{k-z_k}^i\rangle]\nonumber\\
	=\;&\mathbb{E}[\langle \mathbb{E}[F^i(w_{k-z_k}^i,S_k,A_k^i,A_k^{-i},S_{k+1})\mid \mathcal{F}_{k-z_k}]-\bar{F}_{k-z_k}^i(w_{k-z_k}^i),w_{k-z_k}^i-\bar{w}_{k-z_k}^i\rangle]\nonumber\\
	\leq \;&\mathbb{E}[\| \mathbb{E}[F^i(w_{k-z_k}^i,S_k,A_k^i,A_k^{-i},S_{k+1})\mid \mathcal{F}_{k-z_k}]-\bar{F}_{k-z_k}^i(w_{k-z_k}^i)\|_2\|w_{k-z_k}^i-\bar{w}_{k-z_k}^i\|_2]\nonumber\\
	\leq \;&\frac{2}{\lambda^{1/2}(1-\gamma)}\mathbb{E}[\| \mathbb{E}[F^i(w_{k-z_k}^i,S_k,A_k^i,A_k^{-i},S_{k+1})\mid \mathcal{F}_{k-z_k}]-\bar{F}_{k-z_k}^i(w_{k-z_k}^i)\|_2],\label{eq:Ew311}
\end{align}
where the last line follows from $\|w_k^i\|_2\leq M=\lambda^{-1/2}(1-\gamma)^{-1}$ and Lemma \ref{le:propertiesF} (3). To proceed, observe that
\begin{align}
	&\| \mathbb{E}[F^i(w_{k-z_k}^i,S_k,A_k^i,A_k^{-i},S_{k+1})\mid \mathcal{F}_{k-z_k}]-\bar{F}_{k-z_k}^i(w_{k-z_k}^i)\|_2\nonumber\\
	=\;&\| \bar{F}_k^i(w_{k-z_k}^i)-\bar{F}_{k-z_k}^i(w_{k-z_k}^i)\|_2+\| \mathbb{E}[F^i(w_{k-z_k}^i,S_k,A_k^i,A_k^{-i},S_{k+1})\mid \mathcal{F}_{k-z_k}]-\bar{F}_k^i(w_{k-z_k}^i)\|_2\label{eq:Ew31}.
\end{align}
We next bound each term on the RHS of Eq. (\ref{eq:Ew31}). For the term $\| \bar{F}_k^i(w_{k-z_k}^i)-\bar{F}_{k-z_k}^i(w_{k-z_k}^i)\|_2$, we have 
\begin{align*}
	&\| \bar{F}_k^i(w_{k-z_k}^i)-\bar{F}_{k-z_k}^i(w_{k-z_k}^i)\|_2\\
	=\;&\|(\Phi^i)^\top D_k^i(\mathcal{H}^i(v^i,\pi_k^{-i})-\Phi^iw_{k-z_k}^i)-(\Phi^i)^\top D_{k-z_k}^i(\mathcal{H}^i(v^i,\pi_{k-z_k}^{-i})-\Phi^iw_{k-z_k}^i)\|_2\tag{Lemma \ref{le:propertiesF} (2) }\\
	=\;&\|(\Phi^i)^\top D_k^i\Phi^i(\bar{w}_k^i-w_{k-z_k}^i)-(\Phi^i)^\top D_{k-z_k}^i\Phi^i(\bar{w}_{k-z_k}^i-w_{k-z_k}^i)\|_2\tag{Lemma \ref{le:propertiesF} (2) }\\
	\leq \;&\|(\Phi^i)^\top D_k^i\Phi^i(\bar{w}_k^i-\bar{w}_{k-z_k}^i)\|_2+\|(\Phi^i)^\top (D_{k-z_k}^i- D_k^i)\Phi^i(w_{k-z_k}^i-\bar{w}_{k-z_k}^i)\|_2\tag{Triangle inequality}\\
	\leq \;&\frac{2A_{\max}\beta_{k-z_k,k-1}}{\tau \lambda(1-\gamma)^2}+\frac{16L_pA_{\max}\beta_{k-z_k,k-1}}{\tau \lambda(1-\gamma)^2}\\
 \leq \;&\frac{18L_pA_{\max}\beta_{k-z_k,k-1}}{\tau \lambda(1-\gamma)^2},\tag{$L_p\geq 1$}
\end{align*}
where the second last inequality follows from
\begin{align*}
    \|(\Phi^i)^\top D_k^i\Phi^i(\bar{w}_k^i-\bar{w}_{k-z_k}^i)\|_2\leq\;& \|(\Phi^i)^\top D_k^i\Phi^i\|_2\|\bar{w}_k^i-\bar{w}_{k-z_k}^i\|_2\\
    \leq\;& \|\bar{w}_k^i-\bar{w}_{k-z_k}^i\|_2\\
    \leq \;&\frac{2A_{\max}\beta_{k-z_k,k-1}}{\tau \lambda(1-\gamma)^2}\tag{Lemma \ref{le:propertiesF} (4) and telescoping}
\end{align*}
and
\begin{align*}
    &\|(\Phi^i)^\top (D_{k-z_k}^i- D_k^i)\Phi^i(w_{k-z_k}^i-\bar{w}_{k-z_k}^i)\|_2\\
    \leq \;&\|(\Phi^i)^\top\|_{1,2} \|D_{k-z_k}^i- D_k^i\|_{\infty,1}\|\Phi^i\|_{2,\infty}\|w_{k-z_k}^i-\bar{w}_{k-z_k}^i\|_2\\
    = \;&\left(\max_{s,a^i}\|\phi^i(s,a^i)\|_2\right)^2\sum_{s,a^i}\left|\mu_{k-z_k}(s)\pi_{k-z_k}^i(a^i\mid s)-\mu_{k}(s)\pi_{k}^i(a^i\mid s)\right|\|w_{k-z_k}^i-\bar{w}_{k-z_k}^i\|_2\\
    \leq \;&\sum_{s,a^i}\left|\mu_{k-z_k}(s)\pi_{k-z_k}^i(a^i\mid s)-\mu_{k}(s)\pi_{k}^i(a^i\mid s)\right|\|w_{k-z_k}^i-\bar{w}_{k-z_k}^i\|_2\\
    \leq \;&(\|\mu_{k-z_k}-\mu_k\|_1+\max_{s\in\mathcal{S}}\|\pi_{k-z_k}^i(s)-\pi_{k}^i(s)\|_1)\|w_{k-z_k}^i-\bar{w}_{k-z_k}^i\|_2\\
    \leq \;&2L_p\sum_{j=1,2}\max_{s\in\mathcal{S}}\|\pi_{k-z_k}^j(s)-\pi_{k}^j(s)\|_1\|w_{k-z_k}^i-\bar{w}_{k-z_k}^i\|_2\tag{Lemma \ref{le:MC_Lipschitz} and $L_p\geq 1$}\\
    \leq \;&\frac{16L_pA_{\max}\beta_{k-z_k,k-1}}{\tau \lambda(1-\gamma)^2}\tag{Lemma \ref{le:Lipschitzpi}, Lemma \ref{le:propertiesF} (4), and $\|w_k^i\|_2\leq M$ for all $k$}.
\end{align*}
For the term $\| \mathbb{E}[F^i(w_{k-z_k}^i,S_k,A_k^i,A_k^{-i},S_{k+1})\mid \mathcal{F}_{k-z_k}]-\bar{F}_k^i(w_{k-z_k}^i)\|_2$ on the RHS of Eq. (\ref{eq:Ew31}), we have
\begin{align*}
	&\| \mathbb{E}[F^i(w_{k-z_k}^i,S_k,A_k^i,A_k^{-i},S_{k+1})\mid \mathcal{F}_{k-z_k}]-\bar{F}_k^i(w_{k-z_k}^i)\|_2\\
	=\;&\left\|\sum_{s}\left[\left(\prod_{j=k+1}^{k+z_k}P_{\pi_{j-z_k}}\right)(s_{k-z_k},s)-\mu_k(s)\right]\sum_{a^i}\pi_k^i(a^i\mid s)\sum_{a^{-i}}\pi_k^{-i}(a^{-i}\mid s)\right.\\
	&\left.\times \sum_{s'}p(s'\mid s,a^i,a^{-i})F^i(w_{k-z_k}^i,s,a^i,a^{-i},s') \right\|_2\\
	\leq \;&\frac{2}{\lambda^{1/2}(1-\gamma)}\sum_{s}\left|\left(\prod_{j=k+1}^{k+z_k}P_{\pi_{j-z_k}}\right)(s_{k-z_k},s)-\mu_k(s)\right|\tag{Lemma \ref{le:propertiesF} (1)}\\
	\leq \;&\frac{2}{\lambda^{1/2}(1-\gamma)}\sum_{s}\left|\left(\prod_{j=k+1}^{k+z_k}P_{\pi_{j-z_k}}\right)(s_{k-z_k},s)-P_{\pi_k}^{z_k}(S_{k-z_k},s)\right|\\
	&+\frac{2}{\lambda^{1/2}(1-\gamma)}\sum_{s}\left|P_{\pi_k}^{z_k}(S_{k-z_k},s)-\mu_k(s)\right|\\
	\leq \;&\frac{2}{\lambda^{1/2}(1-\gamma)}\left\|\prod_{j=k+1}^{k+z_k}P_{\pi_{j-z_k}}-P_{\pi_k}^{z_k}\right\|_\infty+\frac{4C\rho^{z_k}}{\lambda^{1/2}(1-\gamma)}\tag{Assumption \ref{as:ergodicity}}\\
	\leq \;&\frac{2}{\lambda^{1/2}(1-\gamma)}\left\|\prod_{j=k+1}^{k+z_k}P_{\pi_{j-z_k}}-P_{\pi_k}^{z_k}\right\|_\infty+\frac{4\beta_k}{\lambda^{1/2}(1-\gamma)},
\end{align*}
where the last line follows from the definition of $z_k$. To proceed, observe that
\begin{align*}
	\left\|\prod_{j=k+1}^{k+z_k}P_{\pi_{j-z_k}}-P_{\pi_k}^{z_k}\right\|_\infty
	=\;&\left\|\sum_{\ell=1}^{z_k}\left(\prod_{j=k+1}^{k-\ell+1+z_k}P_{\pi_{j-z_k}}P_{\pi_k}^{\ell-1}-\prod_{j=k+1}^{k-\ell+z_k}P_{\pi_{j-z_k}}P_{\pi_k}^{\ell}\right)\right\|_\infty\\
	=\;&\left\|\sum_{\ell=1}^{z_k}\left(\prod_{j=k+1}^{k-\ell+z_k}P_{\pi_{j-z_k}}(P_{\pi_{k-\ell+1}}-P_{\pi_k})P_{\pi_k}^{\ell-1}\right)\right\|_\infty\\
	\leq \;&\sum_{\ell=1}^{z_k}\left\|\prod_{j=k+1}^{k-\ell+z_k}P_{\pi_{j-z_k}}\right\|_\infty\|P_{\pi_{k-\ell+1}}-P_{\pi_k}\|_\infty\|P_{\pi_k}^{\ell-1}\|_\infty\\
	\leq \;&\sum_{\ell=1}^{z_k}\|P_{\pi_{k-\ell+1}}-P_{\pi_k}\|_\infty.
\end{align*}
It remains to bound $\|P_{\pi_{k-\ell+1}}-P_{\pi_k}\|_\infty$. Using the explicit expression of $P_{\pi_k}$, we have for any $\ell\in [1,z_k]$ that
\begin{align*}
	&\|P_{\pi_{k-\ell+1}}-P_{\pi_k}\|_\infty\\
	=\;&\max_{s}\sum_{s'}\left|\sum_{a^i,a^{-i}}(\pi_{k-\ell+1}^i(a^i\mid s)\pi_{k-\ell+1}^{-i}(a^{-i}\mid s)-\pi_{k}^i(a^i\mid s)\pi_{k}^{-i}(a^{-i}\mid s))p(s'\mid s,a^i,a^{-i})\right|\\
	\leq \;&\max_{s}\sum_{a^i,a^{-i}}\left|\pi_{k-\ell+1}^i(a^i\mid s)\pi_{k-\ell+1}^{-i}(a^{-i}\mid s)-\pi_{k}^i(a^i\mid s)\pi_{k}^{-i}(a^{-i}\mid s)\right|\\
	\leq \;&\max_{s}\sum_{a^{-i}}\left|\pi_{k-\ell+1}^{-i}(a^{-i}\mid s)-\pi_k^{-i}(a^{-i}\mid s)\right|+\max_{s}\sum_{a^i}\left|\pi_{k-\ell+1}^i(a^i\mid s)-\pi_{k}^i(a^i\mid s)\right|\\
	=\;&\sum_{i=1,2}\max_{s}\|\pi_k^i(s)-\pi_{k-\ell+1}^i(s)\|_1\\
	\leq \;&\frac{4A_{\max}\beta_{k-z_k,k-1}}{\tau \lambda^{1/2}(1-\gamma)}\tag{Lemma \ref{le:Lipschitzpi}}
\end{align*}
It follows that
\begin{align*}
	\left\|\prod_{j=k+1}^{k+z_k}P_{\pi_{j-z_k}}-P_{\pi_k}^{z_k}\right\|_\infty
	\leq \frac{4A_{\max}z_k\beta_{k-z_k,k-1}}{\tau \lambda^{1/2}(1-\gamma)}
\end{align*}
As a result, we have
\begin{align*}
	&\| \mathbb{E}[F^i(w_{k-z_k}^i,S_k,A_k^i,A_k^{-i},S_{k+1})\mid \mathcal{F}_{k-z_k}]-\bar{F}_k^i(w_{k-z_k}^i)\|_2\\
	\leq \;&\frac{8A_{\max}z_k\beta_{k-z_k,k-1}}{\tau \lambda(1-\gamma)^2}+\frac{4\beta_k}{\lambda^{1/2}(1-\gamma)}\\
	\leq \;&\frac{10A_{\max}z_k\beta_{k-z_k,k-1}}{\tau \lambda(1-\gamma)^2},
\end{align*}
where the last line follows from $\lambda\leq 1$, $A_{\max}\geq 2$, and $\tau\leq 1/(1-\gamma)$.
Using the upper bounds we established for both terms on the RHS of Eq. (\ref{eq:Ew31}), we have
\begin{align*}
	&\| \mathbb{E}[F^i(w_{k-z_k}^i,S_k,A_k^i,A_k^{-i},S_{k+1})\mid \mathcal{F}_{k-z_k}]-\bar{F}_{k-z_k}^i(w_{k-z_k}^i)\|_2\\
	\leq \;&\frac{10A_{\max}z_k\beta_{k-z_k,k-1}}{\tau \lambda(1-\gamma)^2}+\frac{18L_pA_{\max}\beta_{k-z_k,k-1}}{\tau \lambda(1-\gamma)^2}\\
 \leq \;&\frac{28L_pA_{\max}z_k\beta_{k-z_k,k-1}}{\tau \lambda(1-\gamma)^2}\tag{$L_p\geq 1$}
\end{align*}
Finally, using the previous inequality in Eq. (\ref{eq:Ew311}), we have
\begin{align*}
	E_{w,3,1}
	\leq \;&\frac{2}{\lambda^{1/2}(1-\gamma)}\mathbb{E}[\| \mathbb{E}[F^i(w_{k-z_k}^i,S_k,A_k^i,A_k^{-i},S_{k+1})\mid \mathcal{F}_{k-z_k}]-\bar{F}_{k-z_k}^i(w_{k-z_k}^i)\|_2]\\
	\leq \;&\frac{56L_pA_{\max}z_k\beta_{k-z_k,k-1}}{\tau \lambda^{3/2}(1-\gamma)^3}.
\end{align*}

\paragraph{The Term $E_{w,3,2}$.} For the term $E_{w,3,2}$, we have
\begin{align}
	E_{w,3,2}=\;&\mathbb{E}[\langle F^i(w_{k-z_k}^i,S_k,A_k^i,A_k^{-i},S_{k+1})-\bar{F}_{k-z_k}^i(w_{k-z_k}^i),w_k^i-w_{k-z_k}^i\rangle]\nonumber\\
	\leq \;&\mathbb{E}[\| F^i(w_{k-z_k}^i,S_k,A_k^i,A_k^{-i},S_{k+1})-\bar{F}_{k-z_k}^i(w_{k-z_k}^i)\|_2\|w_k^i-w_{k-z_k}^i\|_2]\label{eq:Ew32}.
\end{align}
For the term $\| F^i(w_{k-z_k}^i,S_k,A_k^i,A_k^{-i},S_{k+1})-\bar{F}_{k-z_k}^i(w_{k-z_k}^i)\|_2$, we have
\begin{align}
	&\| F^i(w_{k-z_k}^i,S_k,A_k^i,A_k^{-i},S_{k+1})-\bar{F}_{k-z_k}^i(w_{k-z_k}^i)\|_2\nonumber\\
	\leq \;&\| F^i(w_{k-z_k}^i,S_k,A_k^i,A_k^{-i},S_{k+1})-F^i(0,S_k,A_k^i,A_k^{-i},S_{k+1})\|_2\nonumber\\
	&+\|F^i(0,S_k,A_k^i,A_k^{-i},S_{k+1})\|_2+\| \bar{F}_{k-z_k}^i(w_{k-z_k}^i)-\bar{F}_{k-z_k}^i(0)\|_2+\|\bar{F}_{k-z_k}^i(0)\|_2\nonumber\\
	\leq \;&2\|w_{k-z_k}^i\|_2+\frac{2}{1-\gamma}\tag{Lemma \ref{le:propertiesF} and Jensen's inequality}\nonumber\\
	\leq \;&\frac{4}{\lambda^{1/2}(1-\gamma)},\label{eq:Ew321}
\end{align}
where the last line follows from $\|w_k^i\|_2\leq M=\lambda^{-1/2}(1-\gamma)^{-1}$ for all $k$. For the term $\|w_k^i-w_{k-z_k}^i\|_2$ on the RHS of Eq. (\ref{eq:Ew32}), we have by Lemma \ref{le:Lipschitzw} that
\begin{align*}
	\|w_k^i-w_{k-z_k}^i\|_2\leq \frac{2\alpha_{k-z_k,k-1}}{\lambda^{1/2}(1-\gamma)}.
\end{align*}
Using the previous two inequalities in Eq. (\ref{eq:Ew32}), we have
\begin{align*}
	E_{w,3,2}\leq \frac{8\alpha_{k-z_k,k-1}}{\lambda(1-\gamma)^2}.
\end{align*}

\paragraph{The Term $E_{w,3,3}$.} For the term $E_{w,3,3}$, we have
\begin{align}
	E_{w,3,3}=\;&\mathbb{E}[\langle F^i(w_{k-z_k}^i,S_k,A_k^i,A_k^{-i},S_{k+1})-\bar{F}_{k-z_k}^i(w_{k-z_k}^i),\bar{w}_{k-z_k}^i-\bar{w}_k^i\rangle]\nonumber\\
	\leq \;&\mathbb{E}[\| F^i(w_{k-z_k}^i,S_k,A_k^i,A_k^{-i},S_{k+1})-\bar{F}_{k-z_k}^i(w_{k-z_k}^i)\|_2\|\bar{w}_{k-z_k}^i-\bar{w}_k^i\|_2]\nonumber\\
	\leq \;&\frac{4}{\lambda^{1/2}(1-\gamma)}\frac{2A_{\max}\beta_{k-z_k,k-1}}{\tau \lambda(1-\gamma)^2}\nonumber\\
 =\;&\frac{8A_{\max}\beta_{k-z_k,k-1}}{\tau \lambda^{3/2}(1-\gamma)^3},
	\label{eq:Ew33}
\end{align}
where the last line follows from Eq. (\ref{eq:Ew321}) and Lemma \ref{le:propertiesF} (4).

\paragraph{The Term $E_{w,3,4}$.} For the term $E_{w,3,4}$, we have
\begin{align*}
	E_{w,3,4}=\;&\mathbb{E}[\langle F^i(w_k^i,S_k,A_k^i,A_k^{-i},S_{k+1})-F^i(w_{k-z_k}^i,S_k,A_k^i,A_k^{-i},S_{k+1}),w_k^i-\bar{w}_k^i\rangle]\\
	=\;&\mathbb{E}[\| F^i(w_k^i,S_k,A_k^i,A_k^{-i},S_{k+1})-F^i(w_{k-z_k}^i,S_k,A_k^i,A_k^{-i},S_{k+1})\|_2\|w_k^i-\bar{w}_k^i\|_2]\\
	\leq \;&\mathbb{E}[\|w_k^i-w_{k-z_k}^i\|_2\|w_k^i-\bar{w}_k^i\|_2]\tag{Lemma \ref{le:propertiesF} (1)}\\
 \leq \;&\mathbb{E}[\|w_k^i-w_{k-z_k}^i\|_2(\|w_k^i\|_2+\|\bar{w}_k^i\|_2)]\\
	\leq \;&\frac{4\alpha_{k-z_k,k-1}}{\lambda(1-\gamma)^2},
\end{align*}
where the last line follows from Lemma \ref{le:Lipschitzw}, Lemma \ref{le:propertiesF} (3), and $\|w_k^i\|_2\leq M=\lambda^{-1/2}(1-\gamma)^{-1}$.

\paragraph{The Term $E_{w,3,5}$.} For the term $E_{w,3,5}$, we have
\begin{align*}
	E_{w,3,4}=\;&\mathbb{E}[\langle \bar{F}_{k-z_k}^i(w_{k-z_k}^i)-\bar{F}_k^i(w_k^i),w_k^i-\bar{w}_k^i\rangle]\\
	\leq \;&\mathbb{E}[\| \bar{F}_{k-z_k}^i(w_{k-z_k}^i)-\bar{F}_k^i(w_k^i)\|_2\|w_k^i-\bar{w}_k^i\|_2]\\
	\leq \;&\mathbb{E}[\| w_{k-z_k}^i-w_k^i\|_2\|w_k^i-\bar{w}_k^i\|_2]\tag{Lemma \ref{le:propertiesF} (1) and Jensen's inequality}\\
	\leq \;&\frac{4\alpha_{k-z_k,k-1}}{\lambda(1-\gamma)^2},
\end{align*}
where the last line follows from Lemma \ref{le:Lipschitzw}, Lemma \ref{le:propertiesF} (3), and $\|w_k^i\|_2\leq M=\lambda^{-1/2}(1-\gamma)^{-1}$.

Using the upper bounds we obtained for the terms $E_{w,3,1}$ to $E_{w,3,5}$ in Eq. (\ref{eq:Ew3}), we have
\begin{align*}
	E_{w,3}=\;&\frac{56L_pA_{\max}z_k\beta_{k-z_k,k-1}}{\tau \lambda^{3/2}(1-\gamma)^3}+\frac{8A_{\max}\beta_{k-z_k,k-1}}{\tau \lambda^{3/2}(1-\gamma)^3}+\frac{16\alpha_{k-z_k,k-1}}{\lambda(1-\gamma)^2}\\
 \leq \;&\frac{64L_pA_{\max}z_k\beta_{k-z_k,k-1}}{\tau \lambda^{3/2}(1-\gamma)^3}+\frac{16\alpha_{k-z_k,k-1}}{\lambda(1-\gamma)^2}\tag{$L_p\geq 1$}\\
 \leq \;&\frac{17z_k\alpha_{k-z_k,k-1}}{\tau\lambda^{3/2} (1-\gamma)^3},
\end{align*}
where the last line follows from $\frac{\beta_k}{\alpha_k}=c_{\alpha,\beta}\leq \frac{1}{64L_p A_{\max}}$, $\lambda\leq 1$, and $\tau\leq 1/(1-\gamma)$.
\end{proof}

\subsection{Solving the Recursion}\label{ap:recursion}

We begin by restating the Lyapunov inequalities from the previous sections:
\begin{itemize}
    \item Lemma \ref{le:vt-v*}:
    \begin{align}\label{recursion_v}
     \mathcal{L}_v(t+1)\leq 
    \gamma \mathcal{L}_v(t)+\frac{17 A_{\max}^2}{\tau^2(1-\gamma)^2}\mathcal{L}_\theta(t,K)+2\mathcal{L}_{\text{sum}}(t)+ 2\mathcal{L}_w(t,K)^{1/2}+12\tau \log(A_{\max})+\mu.
 \end{align}
 \item Lemma \ref{le:v_sum}:
 \begin{align}\label{recursion_vsum}
		\mathcal{L}_{\text{sum}}(t+1)\leq\gamma \mathcal{L}_{\text{sum}}(t)+2\mathcal{L}_{w}(t,K)^{1/2}.
	\end{align}
 \item Lemma \ref{le:policy}:
 \begin{align}\label{recursion_pi}
	\mathcal{L}_\theta(t,k+1)
 \leq \left(1-\frac{\beta_k}{2}\right)\mathcal{L}_\theta(t,k)+ \frac{520 A_{\max}\beta_k}{\tau}\mathcal{L}_w(t,k)+ 4\beta_k\mathcal{L}_{\text{sum}}(t)+\frac{276A_{\max}^{1/2}\beta_k^2}{\tau\lambda^{1/2}(1-\gamma)}.
\end{align}
\item Lemma \ref{le:drift_inequality_q}:
\begin{align}\label{recursion_q}
        \mathbb{E}[\mathcal{L}_w(t,k+1)]
 \leq \;&\left(1-\frac{3\lambda\alpha_k}{2}\right) \mathbb{E}[\mathcal{L}_w(t,k)]+\frac{84z_k\alpha_k\alpha_{k-z_k,k-1}}{\tau\lambda^{3/2} (1-\gamma)^3}+\frac{32A_{\max}^2\beta_k^2}{\tau^2 \lambda^3(1-\gamma)^4\alpha_k}.
    \end{align}
\end{itemize}
The next step is to solve the coupled Lyapunov inequalities above to obtain finite-sample bounds. Here, we consider using constant stepsizes: $\alpha_k\equiv \alpha$ and $\beta_k\equiv \beta$. The finite-sample bounds for using diminishing stepsizes are straightforward extensions as we derived all the Lyapunov drift inequalities for general stepsizes.

Repeatedly using Eq. (\ref{recursion_q}), we have for all $k\geq z_\beta$ that
\begin{align}
    \mathbb{E}[\mathcal{L}_w(t,k)]
 \leq \;&\left(1-\lambda\alpha\right)^{k-z_\beta} \mathbb{E}[\mathcal{L}_w(t,z_\beta)]+\frac{84z_\beta^2\alpha}{\tau\lambda^{5/2} (1-\gamma)^3}+\frac{32A_{\max}^2\beta^2}{\tau^2 \lambda^4(1-\gamma)^4\alpha^2}\nonumber\\
 \leq \;&\frac{8\left(1-\lambda\alpha\right)^{k-z_\beta}}{\lambda(1-\gamma)^2} +\frac{84z_\beta^2\alpha}{\tau\lambda^{5/2} (1-\gamma)^3}+\frac{32A_{\max}^2\beta^2}{\tau^2 \lambda^4(1-\gamma)^4\alpha^2},\label{eq:solving_q1}
\end{align}
where the last line follows from
\begin{align*}
    \mathcal{L}_w(t,k)=\sum_{i=1,2}\mathbb{E}[\|w_{t,k}^i-\Bar{w}_{t,k}^i\|_2^2]\leq \frac{8}{\lambda(1-\gamma)^2}.
\end{align*}
Using Eq. (\ref{eq:solving_q1}) in Eq. (\ref{recursion_pi}), we have
\begin{align*}
    \mathbb{E}[\mathcal{L}_\theta(t,k+1)]
 \leq\;& \left(1-\frac{\beta}{2}\right)\mathbb{E}[\mathcal{L}_\theta(t,k)]+ 4\beta\mathbb{E}[\mathcal{L}_{\text{sum}}(t)]+\frac{276A_{\max}^{1/2}\beta^2}{\tau\lambda^{1/2}(1-\gamma)}\\
 &+\frac{520 A_{\max}\beta}{\tau}\left[\frac{8\left(1-\lambda\alpha\right)^{k-z_\beta}}{\lambda(1-\gamma)^2} +\frac{84z_\beta^2\alpha}{\tau\lambda^{5/2} (1-\gamma)^3}+\frac{32A_{\max}^2\beta^2}{\tau^2 \lambda^4(1-\gamma)^4\alpha^2}\right]\\
 \leq\;& \left(1-\frac{\beta}{2}\right)\mathbb{E}[\mathcal{L}_\theta(t,k)]+ 4\beta\mathbb{E}[\mathcal{L}_{\text{sum}}(t)]+\frac{276A_{\max}^{1/2}\beta^2}{\tau\lambda^{1/2}(1-\gamma)}\\
 &+\frac{520 A_{\max}\beta}{\tau}\left[\frac{8\left(1-\beta\right)^{k-z_\beta}}{\lambda(1-\gamma)^2} +\frac{84z_\beta^2\alpha}{\tau\lambda^{5/2} (1-\gamma)^3}+\frac{32A_{\max}^2\beta^2}{\tau^2 \lambda^4(1-\gamma)^4\alpha^2}\right],
\end{align*}
where the last line follows from $\beta\leq \lambda\alpha$.
Repeatedly using the previous inequality, we have for all $k\geq z_\beta$ that
\begin{align}
    \mathbb{E}[\mathcal{L}_\theta(t,k)]
 \leq\;& \left(1-\frac{\beta}{2}\right)^{k-z_\beta}\mathcal{L}_\theta(t,z_\beta)+ 8\mathbb{E}[\mathcal{L}_{\text{sum}}(t)]+\frac{552A_{\max}^{1/2}\beta}{\tau\lambda^{1/2}(1-\gamma)}\nonumber\\
 &+\frac{520 A_{\max}}{\tau}\left[\frac{8\beta(k-z_\beta)\left(1-\beta/2\right)^{k-z_\beta-1}}{\lambda(1-\gamma)^2} +\frac{168z_\beta^2\alpha}{\tau\lambda^{5/2} (1-\gamma)^3}+\frac{64A_{\max}^2\beta^2}{\tau^2 \lambda^4(1-\gamma)^4\alpha^2}\right]\nonumber\\
 \leq \;&\frac{256A_{\max}}{\tau\lambda(1-\gamma)^2}\left(1-\frac{\beta}{2}\right)^{k-z_\beta}+ 8\mathbb{E}[\mathcal{L}_{\text{sum}}(t)]+\frac{552A_{\max}^{1/2}\beta}{\tau\lambda^{1/2}(1-\gamma)}\nonumber\\
 &+\frac{520 A_{\max}}{\tau}\left[\frac{8\beta(k-z_\beta)\left(1-\beta/2\right)^{k-z_\beta-1}}{\lambda(1-\gamma)^2} +\frac{168z_\beta^2\alpha}{\tau\lambda^{5/2} (1-\gamma)^3}+\frac{64A_{\max}^2\beta^2}{\tau^2 \lambda^4(1-\gamma)^4\alpha^2}\right],\label{eq:solving_pi1}
\end{align}
where the last line follows from
\begin{align*}
    \mathcal{L}_\theta(t,k)=\;&\max_{s\in\mathcal{S}}\sum_{i=1,2}\max_{u^i\in  \Delta(\mathcal{A}^i)}\min_{\bar{u}^i\in\Delta(\mathcal{A}^i)}\left\{(u^i-\Bar{u}^i)^\top \Tilde{q}_{t,k}^i(s)+\tau \nu(u^i)-\tau \nu(\bar{u}^i)\right.\\
    &\left.+\;\frac{1}{2\mu}\|\Tilde{q}_{t,k}^{-i}(s)-\mathcal{T}^{-i}(v_t^{-i})(s)\bar{u}^i\|_2^2\right\}\\
    \leq \;&\max_{s\in\mathcal{S}}\sum_{i=1,2}\frac{1}{2\mu}\|\Tilde{q}_{t,k}^{-i}(s)-\mathcal{T}^{-i}(v_t^{-i})(s)\sigma_\tau(\Tilde{q}_{t,k}^i(s))\|_2^2\\
    \leq \;&\max_{s\in\mathcal{S}}\sum_{i=1,2}\frac{1}{2\mu}(\|\Tilde{q}_{t,k}^{-i}(s)\|_2+\|\mathcal{T}^{-i}(v_t^{-i})(s)\sigma_\tau(\Tilde{q}_{t,k}^i(s))\|_2)^2\\
    \leq \;&\max_{s\in\mathcal{S}}\sum_{i=1,2}\frac{1}{\mu}(\|\Tilde{q}_{t,k}^{-i}(s)\|_2^2+\|\mathcal{T}^{-i}(v_t^{-i})(s)\sigma_\tau(\Tilde{q}_{t,k}^i(s))\|_2^2)\tag{$(a+b)^2\leq 2a^2+2b^2$}\\
    \leq \;&\max_{s\in\mathcal{S}}\sum_{i=1,2}\frac{1}{\mu}\left(\sum_{a^i}(\phi^i(s,a^i)^\top \theta_k^i)^2+\frac{A_{\max}}{(1-\gamma)^2}\right)\\
    \leq \;&\sum_{i=1,2}\frac{1}{\mu}\left(A_{\max} \|\theta_k^i\|_2^2+\frac{A_{\max}}{(1-\gamma)^2}\right)\tag{Cauchy–Schwarz inequality and $\|\phi^i(s,a^i)\|_2\leq 1$}\\
    \leq \;&\sum_{i=1,2}\frac{1}{\mu}\left(\frac{A_{\max}}{\lambda(1-\gamma)^2}+\frac{A_{\max}}{(1-\gamma)^2}\right)\\
    \leq \;&\frac{256A_{\max}}{\tau\lambda(1-\gamma)^2},\quad \forall\;t,k,
\end{align*}
where the second last inequality follows from $\|w_{t,k}^i\|_2\leq M=\lambda^{-1/2}(1-\gamma)^{-1}$ and $\theta_{t,k}^i$ being a convex combination of $\{w_{t,\ell}^i\}_{0\leq \ell\leq k-1}$, and the last inequality follows from $\lambda\leq 1$ and $\mu=\tau/64$.

Similarly, using Eq. (\ref{eq:solving_q1}) in Eq. (\ref{recursion_vsum}), we have
\begin{align*}
    \mathbb{E}[\mathcal{L}_{\text{sum}}(t+1)]\leq\gamma \mathbb{E}[\mathcal{L}_{\text{sum}}(t)]+\frac{6\left(1-\lambda\alpha\right)^{\frac{K-z_\beta}{2}}}{\lambda^{1/2}(1-\gamma)} +\frac{20z_\beta\alpha^{1/2}}{\tau^{1/2}\lambda^{5/4} (1-\gamma)^{3/2}}+\frac{12A_{\max}\beta}{\tau \lambda^2(1-\gamma)^2\alpha}.
\end{align*}
Repeatedly using the previous inequality, since $\mathcal{L}_{\text{sum}}(t)\leq 2/(1-\gamma)$, we have
\begin{align}\label{eq:solving_vsum}
    \mathbb{E}[\mathcal{L}_{\text{sum}}(t)]\leq \frac{2\gamma^t}{1-\gamma}+\frac{6\left(1-\lambda\alpha\right)^{\frac{K-z_\beta}{2}}}{\lambda^{1/2}(1-\gamma)^2} +\frac{20z_\beta\alpha^{1/2}}{\tau^{1/2}\lambda^{5/4} (1-\gamma)^{5/2}}+\frac{12A_{\max}\beta}{\tau \lambda^2(1-\gamma)^3\alpha}
\end{align}
Using Eqs. (\ref{eq:solving_q1}), (\ref{eq:solving_pi1}), and (\ref{eq:solving_vsum}) altogether in Eq. (\ref{recursion_v}), we have
\begin{align*}
    \mathbb{E}[\mathcal{L}_v(t+1)]\leq\;& 
    \gamma \mathbb{E}[\mathcal{L}_v(t)]+\frac{17 A_{\max}^2}{\tau^2(1-\gamma)^2}\mathbb{E}[\mathcal{L}_\theta(t,K)]+2\mathbb{E}[\mathcal{L}_{\text{sum}}(t)]+ 2\mathbb{E}[\mathcal{L}_w(t,K)^{1/2}]\\
    &+12\tau \log(A_{\max})+\mu\\
    \leq \;&\gamma \mathbb{E}[\mathcal{L}_v(t)]+\frac{276 A_{\max}^2\gamma^t}{\tau^2(1-\gamma)^3}+\frac{828A_{\max}^2\left(1-\lambda\alpha\right)^{\frac{K-z_\beta}{2}}}{\tau^2\lambda^{1/2}(1-\gamma)^4} +\frac{2760A_{\max}^2z_\beta\alpha^{1/2}}{\tau^{5/2}\lambda^{5/4} (1-\gamma)^{9/2}}\\
    &+\frac{1656A_{\max}^3\beta}{\tau^3 \lambda^2(1-\gamma)^5\alpha}+\frac{4352A_{\max}^3}{\tau^3\lambda(1-\gamma)^4}\left(1-\frac{\beta}{2}\right)^{K-z_\beta}+\frac{9384A_{\max}^{5/2}\beta}{\tau^3\lambda^{1/2}(1-\gamma)^3}\nonumber\\
 &+\frac{70720A_{\max}^3(K-z_\beta)\left(1-\beta/2\right)^{K-z_\beta-1}}{\tau^3\lambda(1-\gamma)^4} +\frac{1485120A_{\max}^3z_\beta^2\alpha}{\tau^4\lambda^{5/2} (1-\gamma)^5}+\frac{565760A_{\max}^5\beta^2}{\tau^5 \lambda^4(1-\gamma)^6\alpha^2}\\
    &+ \frac{6\left(1-\lambda\alpha\right)^{\frac{K-z_\beta}{2}}}{\lambda^{1/2}(1-\gamma)} +\frac{20z_\beta\alpha^{1/2}}{\tau^{1/2}\lambda^{5/4} (1-\gamma)^{3/2}}+\frac{12A_{\max}\beta}{\tau \lambda^2(1-\gamma)^2\alpha}\\
    &+12\tau \log(A_{\max})+\frac{\tau}{64}\\
    \leq \;&\gamma \mathbb{E}[\mathcal{L}_v(t)]+\frac{75906A_{\max}^3K\left(1-\beta/2\right)^{\frac{K-z_\beta-1}{2}}}{\tau^3\lambda(1-\gamma)^5}+\frac{276 A_{\max}^2\gamma^t}{\tau^2(1-\gamma)^3} +\frac{2780A_{\max}^2z_\beta\alpha^{1/2}}{\tau^{5/2}\lambda^{5/4} (1-\gamma)^{9/2}}\\
    &+\frac{1668A_{\max}^3\beta}{\tau^3 \lambda^2(1-\gamma)^5\alpha} +\frac{1494504A_{\max}^3z_\beta^2\alpha}{\tau^4\lambda^{5/2} (1-\gamma)^5}+\frac{565760A_{\max}^5\beta^2}{\tau^5 \lambda^4(1-\gamma)^6\alpha^2}+13\tau \log(A_{\max}).
\end{align*}
Repeatedly using the previous inequality, since $\mathcal{L}_v(t)\leq \frac{4}{1-\gamma}$, we have
\begin{align}
    \mathbb{E}[\mathcal{L}_v(T)]\lesssim\;& \frac{A_{\max}^3K\left(1-\beta/2\right)^{\frac{K-z_\beta-1}{2}}}{\tau^3\lambda(1-\gamma)^6}+\frac{A_{\max}^2T\gamma^{T-1}}{\tau^2(1-\gamma)^3} +\frac{A_{\max}^2z_\beta\alpha^{1/2}}{\tau^{5/2}\lambda^{5/4} (1-\gamma)^{11/2}}\nonumber\\
    &+\frac{A_{\max}^3\beta}{\tau^3 \lambda^2(1-\gamma)^6\alpha} +\frac{A_{\max}^3z_\beta^2\alpha}{\tau^4\lambda^{5/2} (1-\gamma)^6}+\frac{A_{\max}^5\beta^2}{\tau^5 \lambda^4(1-\gamma)^7\alpha^2}+\frac{\tau \log(A_{\max})}{1-\gamma},\label{eq:solving_v}
\end{align}
where we recall that the notation $a\lesssim b$ means that there exists an absolute constant $c>0$ such that $a\leq bc$.

Finally, using Eqs. (\ref{eq:solving_v}), (\ref{eq:solving_q1}), (\ref{eq:solving_pi1}), and (\ref{eq:solving_vsum}) together in Lemma \ref{le:Bound_Nash_Gap}, we have
\begin{align*}
    \mathbb{E}[\text{NG}(\pi_{T,K}^1,\pi_{T,K}^2)]\lesssim \;&\frac{A_{\max}^3K\left(1-\beta/2\right)^{\frac{K-z_\beta-1}{2}}}{\tau^3\lambda(1-\gamma)^7}+\frac{A_{\max}^2T\gamma^{T-1}}{\tau^2(1-\gamma)^4} +\frac{A_{\max}^2z_\beta\alpha^{1/2}}{\tau^{5/2}\lambda^{5/4} (1-\gamma)^{13/2}}\nonumber\\
    &+\frac{A_{\max}^3\beta}{\tau^3 \lambda^2(1-\gamma)^7\alpha} +\frac{A_{\max}^3z_\beta^2\alpha}{\tau^4\lambda^{5/2} (1-\gamma)^7}+\frac{A_{\max}^5\beta^2}{\tau^5 \lambda^4(1-\gamma)^8\alpha^2}+\frac{\tau \log(A_{\max})}{(1-\gamma)^2}.
\end{align*}
As for the sample complexity, to achieve $\mathbb{E}[\text{NG}(\pi_{T,K}^1,\pi_{T,K}^2)]\leq \epsilon$, in view of the previous inequality, we have
\begin{align*}
    T=\Tilde{\mathcal{O}}(1),\;\tau=\mathcal{O}(\epsilon),\; \alpha=\Tilde{\mathcal{O}}(\epsilon^7), \beta=\Tilde{\mathcal{O}}(\epsilon^{11}),\; K=\Tilde{\mathcal{O}}(\epsilon^{-11}).
\end{align*}
 Therefore, the overall sample complexity (i.e., $KT$) is $\tilde{\mathcal{O}}(\epsilon^{-11})$. It is easy to check that it is also polynomial in $1/(1-\gamma)$ and $\lambda$.

\subsection{Supporting Lemmas}\label{ap:supporting_lemma}
\begin{lemma}\label{le:use_for_v}
The following results hold.
	\begin{enumerate}[(1)]
		\item For any $v_1,v_2\in\mathbb{R}^{|\mathcal{S}|}$ and $i\in \{1,2\}$, we have
		\begin{align*}
			\|\mathcal{T}^i(v_1)-\mathcal{T}^i(v_2)\|_\infty\leq \gamma\|v_1-v_2\|_\infty,\quad \text{and}\quad 
			\|\mathcal{T}^i(v_1)+\mathcal{T}^{-i}(v_2)\|_\infty\leq \gamma\|v_1+v_2\|_\infty.
		\end{align*}
		\item For any $v_1,v_2\in\mathbb{R}^{|\mathcal{S}|}$, $s\in\mathcal{S}$, and $i\in \{1,2\}$, we have
		\begin{align*}
			\max_{\mu^i\in \Delta(\mathcal{\mathcal{A}}^i)}\max_{\mu^{-i}\in \Delta(\mathcal{\mathcal{A}}^{-i})}\left|(\mu^i)^\top  (\mathcal{T}^i(v_1)(s)-\mathcal{T}^i(v_2)(s))\mu^{-i}\right|\leq \gamma \|v_1-v_2\|_\infty.
		\end{align*}
		\item Given any $v^1,v^2\in\mathbb{R}^{|\mathcal{S}|}$, $\tilde{q}^1\in\mathbb{R}^{|\mathcal{S}||\mathcal{A}^1|}$ and $\tilde{q}^2\in\mathbb{R}^{|\mathcal{S}||\mathcal{A}^2|}$, let $\pi^i(s)=\sigma_\tau(\tilde{q}^i(s))$ for all $s\in\mathcal{S}$ and $i\in \{1,2\}$. Then we have for all $s\in\mathcal{S}$ that
		\begin{align*}
			\left|\sum_{i=1,2}\max_{\mu^i\in \Delta(\mathcal{A}^i)}(\mu^i)^\top  \mathcal{T}^i(v^i)(s)\pi^{-i}(s)\right|
			\leq\;& \frac{9 A_{\max}^2}{\tau^2(1-\gamma)^2}V_{v,s}(\tilde{q}^1(s),\tilde{q}^2(s))+\gamma \|v^1+v^2\|_\infty\\
            &+6\tau \log(A_{\max})+\mu.
		\end{align*}
	\end{enumerate}
\end{lemma}
\begin{proof}[Proof of Lemma \ref{le:use_for_v}]
	\begin{enumerate}[(1)]
		\item For any $v_1,v_2\in\mathbb{R}^{|\mathcal{S}|}$ and $i\in \{1,2\}$, we have for any $(s,a^i,a^{-i})$ that
		\begin{align*}
			|\mathcal{T}^i(v_1)(s,a^i,a^{-i})-\mathcal{T}^i(v_2)(s,a^i,a^{-i})|
			=\;&\gamma|\mathbb{E}[v_1(S_1)-v_2(S_1)\mid S_0=s,A_0^i=a^i,A_0^{-i}=a^{-i}]|\\
			\leq\;& \gamma\|v_1-v_2\|_\infty.
		\end{align*}
		It follows that $\|\mathcal{T}^i(v_1)-\mathcal{T}^i(v_2)\|_\infty\leq \gamma\|v_1-v_2\|_\infty$.
		Since $\mathcal{T}^i(-v)=-\mathcal{T}^{-i}(v)$ for any $v\in\mathbb{R}^{|\mathcal{S}|}$, we also have 
  \begin{align*}
      \|\mathcal{T}^i(v_1)+\mathcal{T}^{-i}(v_2)\|_\infty=\|\mathcal{T}^i(v_1)-\mathcal{T}^i(-v_2)\|_\infty
      \leq \gamma\|v_1+v_2\|_\infty.
  \end{align*}
		\item For any $v_1,v_2\in\mathbb{R}^{|\mathcal{S}|}$, we have 
		\begin{align*}
			\max_{\mu^i\in \Delta(\mathcal{\mathcal{A}}^i)}\max_{\mu^{-i}\in \Delta(\mathcal{\mathcal{A}}^{-i})}\left|(\mu^i)^\top  (\mathcal{T}^i(v_1)(s)-\mathcal{T}^i(v_2)(s))\mu^{-i}\right|\leq\;& \left\|\mathcal{T}^i(v_1)-\mathcal{T}^i(v_2)\right\|_\infty\\
			\leq\;& \gamma \|v_1-v_2\|_\infty,
		\end{align*}
  where the last inequality follows from Part (1) of this lemma. 
		\item We start by decomposing the term we want to bound as
		\begin{align}
			\left|\sum_{i=1,2}\max_{\mu^i\in \Delta(\mathcal{A}^i)}(\mu^i)^\top  \mathcal{T}^i(v^i)(s)\pi^{-i}(s)\right|
			= \;&\left|\sum_{i=1,2}\max_{\mu^i\in \Delta(\mathcal{A}^i)}(\mu^i)^\top  (\mathcal{T}^i(v^i)(s)\pi^{-i}(s)-\tilde{q}^i(s))\right|\nonumber\\
   &+\left|\sum_{i=1,2}\max_{\mu^i\in \Delta(\mathcal{A}^i)}(\mu^i)^\top  \tilde{q}^i(s)\right|.\label{eq2:le:use_for_v}
		\end{align}
		We next bound the two terms on the RHS of the previous inequality. For the first term on the RHS of Eq. (\ref{eq2:le:use_for_v}), we have
		\begin{align}
			\left|\sum_{i=1,2}\max_{\mu^i\in \Delta(\mathcal{A}^i)}(\mu^i)^\top  (\mathcal{T}^i(v^i)(s)\pi^{-i}(s)-\tilde{q}^i(s))\right|
			\leq \;&\sum_{i=1,2}\|\tilde{q}^i(s)-\mathcal{T}^i(v^i)(s)\pi^{-i}(s)\|_2\tag{Cauchy–Schwarz inequality}\nonumber\\
   \leq \;&\frac{2\sqrt{2}}{\sqrt{\tau}}V_{v,s}^{1/2}(\tilde{q}^1(s),\tilde{q}^2(s))\sum_{i=1,2}\|\mathcal{T}^i(v^i)(s)\|_2\tag{Lemma \ref{le:useful}}\nonumber\\
			\leq  \;&\frac{4\sqrt{2}A_{\max}}{\sqrt{\tau}(1-\gamma)}V^{1/2}_{v,s}(\tilde{q}^1(s),\tilde{q}^2(s))\nonumber\\
			\leq \;&\tau+\frac{8 A_{\max}^2}{\tau^2(1-\gamma)^2}V_{v,s}(\tilde{q}^1(s),\tilde{q}^2(s)),\label{eqfirst:le:use_for_v}
		\end{align}
		where the last line follows from $a^2+b^2\geq 2ab$. For the second term on the RHS of Eq. (\ref{eq2:le:use_for_v}), we have for any $\hat{\mu}^1\in\Delta(\mathcal{A}^1)$, $\hat{\mu}^2\in\Delta(\mathcal{A}^2)$, and $s\in\mathcal{S}$ that
		\begin{align}
			\left|\sum_{i=1,2}\max_{\mu^i\in \Delta(\mathcal{A}^i)}(\mu^i)^\top \tilde{q}^i(s)\right|
			\leq \;&\sum_{i=1,2}\max_{\mu^i\in \Delta(\mathcal{A}^i)}(\mu^i-\hat{\mu}^i)^\top \tilde{q}^i(s)+\left|\sum_{i=1,2}(\hat{\mu}^i)^\top \tilde{q}^i(s)\right|\nonumber\\
			\leq \;&\sum_{i=1,2}\max_{\mu^i\in \Delta(\mathcal{A}^i)}\left\{(\mu^i-\hat{\mu}^i)^\top \tilde{q}^i(s)+\tau \nu(\mu^i)-\tau \nu(\hat{\mu}^i)\right\}+2\tau \log(A_{\max})\nonumber\\
			&+\left|\sum_{i=1,2}(\hat{\mu}^i)^\top \tilde{q}^i(s)\right|.\label{eq:le1v:1}
		\end{align}
		We next consider the last term on the RHS of the previous inequality. Observe that
		\begin{align*}
			\left|\sum_{i=1,2}(\hat{\mu}^i)^\top \tilde{q}^i(s)\right|=\;&\left|\sum_{i=1,2}(\hat{\mu}^i)^\top \left(\tilde{q}^i(s)-\mathcal{T}^i(v^i)(s)\hat{\mu}^{-i}\right)\right|+\left|\sum_{i=1,2}(\hat{\mu}^i)^\top \mathcal{T}^i(v^i)(s)\hat{\mu}^{-i}\right|\\
			\leq \;&\sum_{i=1,2}\|\hat{\mu}^i\|_2\|\tilde{q}^i(s)-\mathcal{T}^i(v^i)(s)\hat{\mu}^{-i}\|_2+\gamma \|v^1+v^2\|_\infty \tag{Part (2) of this lemma}\\
			\leq \;&\sum_{i=1,2}\left(\frac{\mu\|\hat{\mu}^i\|_2^2}{2}+\frac{\|\tilde{q}^i(s)-\mathcal{T}^i(v^i)(s)\hat{\mu}^{-i}\|_2^2}{2\mu}\right)+\gamma \|v^1+v^2\|_\infty\tag{$a^2+b^2\geq 2ab$}\\
			\leq \;&\mu+\sum_{i=1,2}\frac{1}{2\mu}\|\tilde{q}^i(s)-\mathcal{T}^i(v^i)(s)\hat{\mu}^{-i}\|_2^2+\gamma \|v^1+v^2\|_\infty.
		\end{align*}
		Substituting the previous inequality into Eq. (\ref{eq:le1v:1}), we have
		\begin{align*}
			&\left|\sum_{i=1,2}\max_{\mu^i\in \Delta(\mathcal{A}^i)}(\mu^i)^\top \tilde{q}^i(s)\right|\\
			\leq \;&\sum_{i=1,2}\max_{\mu^i\in \Delta(\mathcal{A}^i)}\left\{(\mu^i-\hat{\mu}^i)^\top \tilde{q}^i(s)+\tau \nu(\mu^i)-\tau \nu(\hat{\mu}^i)+\frac{1}{2\mu}\|\tilde{q}^i(s)-\mathcal{T}^i(v^i)(s)\hat{\mu}^{-i}\|_2^2\right\}\\
			&+\gamma \|v^1+v^2\|_\infty+2\tau \log(A_{\max})+\mu.
		\end{align*}
		Since the previous inequality holds for all $\hat{\mu}^1\in\Delta(\mathcal{A}^1)$ and  $\hat{\mu}^2\in\Delta(\mathcal{A}^2)$, we have
		\begin{align*}
			&\left|\sum_{i=1,2}\max_{\mu^i\in \Delta(\mathcal{A}^i)}(\mu^i)^\top \tilde{q}^i(s)\right|\\
			\leq \;&\sum_{i=1,2}\max_{\mu^i\in \Delta(\mathcal{A}^i)}\min_{\hat{\mu}^i\in \Delta(\mathcal{A}^i)}\left\{(\mu^i-\hat{\mu}^i)^\top \tilde{q}^i(s)+\tau \nu(\mu^i)-\tau \nu(\hat{\mu}^i)+\frac{1}{2\mu}\|\tilde{q}^i(s)-\mathcal{T}^i(v^i)(s)\hat{\mu}^{-i}\|_2^2\right\}\\
			&+\gamma \|v^1+v^2\|_\infty+2\tau \log(A_{\max})+\mu\\
			=\;&V_{v,s}(\tilde{q}^1(s),\tilde{q}^2(s))+\gamma \|v^1+v^2\|_\infty+2\tau \log(A_{\max})+\mu.
		\end{align*}
		Using the previous inequality and Eq. (\ref{eqfirst:le:use_for_v}) together in Eq. (\ref{eq2:le:use_for_v}), since $A_{\max}\geq 2$ and $\tau\leq 1/(1-\gamma)$, we have
		\begin{align*}
			\left|\sum_{i=1,2}\max_{\mu^i\in \Delta(\mathcal{A}^i)}(\mu^i)^\top  \mathcal{T}^i(v^i)(s)\pi^{-i}(s)\right|
			\leq\;&\frac{9 A_{\max}^2}{\tau^2(1-\gamma)^2}V_{v,s}(\tilde{q}^1(s),\tilde{q}^2(s))+\gamma \|v^1+v^2\|_\infty\\
   &+6\tau \log(A_{\max})+\mu.
		\end{align*}
	\end{enumerate}
\end{proof}
\end{document}